\pdfoutput=1
\documentclass[examplefnt,biber]{nowfnt} 

\usepackage[utf8]{inputenc}
\usepackage{microtype}

\usepackage{amsmath}
\usepackage{amsfonts}
\usepackage{amssymb}
\usepackage{thmtools}
\usepackage{mathtools}
\usepackage{bm}
\usepackage{physics}
\usepackage{xfrac}

\numberwithin{theorem}{chapter} 
\newtheorem{assumption}[theorem]{Assumption} 

\DeclarePairedDelimiterX{\kldiv}[2]{(}{)}{%
  #1\;\delimsize\|\;#2%
}
\newcommand{\KL}{\textup{KL}\kldiv} 

\newcommand{\gpy}{{GPy}}
\newcommand{\numpy}{{NumPy}}
\newcommand{\scipy}{{SciPy}}
\newcommand{\gail}{{GAIL}}
\newcommand{\emukit}{{EmuKit}}
\newcommand{\probnum}{{ProbNum}}
\newcommand{\matlab}{{MATLAB}}

\newcommand{\domain}{D}
\newcommand{\measure}{P}
\newcommand{\f}{f}
\newcommand{\points}{X}
\newcommand{\pointsb}{X_\textup{batch}}
\newcommand{\pointsi}{X_\textup{init}}
\newcommand{\dataset}{{\cal D}}

\newcommand{\GP}{\textup{GP}}
\newcommand{\MI}{\textup{MI}}  
\newcommand{\IVR}{\textup{IVR}}  
\newcommand{\NIV}{\textup{NIV}}  
\newcommand{\US}{\textup{US}}  
\newcommand{\PVC}{\textup{PVC}}  

\newcommand{\hparam}{\theta} 
\newcommand{\hparams}{{\vectorbm \hparam}} 
\newcommand{\kernel}{k} 
\newcommand{\mean}{m} 
\newcommand{\transformation}{\varphi} 

\newcommand{\gGP}{\randvar{g}}
\newcommand{\fGP}{\randvar{f}}

\newcommand{\Normal}{\textup{N}}

\newcommand{\Id}{\mI} 

\newcommand{\rkhs}{\mathcal{H}}

\newcommand{\algres}{res}

\DeclarePairedDelimiterX\Set[2]{\lbrace}{\rbrace}%
{ #1 \,:\, #2 }                                         

\newcommand{\N}{\mathbb{N}}

\newcommand{\R}{\mathbb{R}}

\newcommand{\Rd}{\mathbb{R}^d}

\newcommand{\Trans}{^{\intercal}}

\DeclareSymbolFont{stmry}{U}{stmry}{m}{n}
\DeclareMathSymbol\obar\mathrel{stmry}{"3A}
\DeclareMathSymbol\otimes\mathrel{stmry}{"0F}
\DeclareMathSymbol\ominus\mathrel{stmry}{"17}
\makeatletter
\newcommand{\superimpose}[2]{
  {\ooalign{$#1\@firstoftwo#2$\cr\hfil$#1\@secondoftwo#2$\hfil\cr}}}
\makeatother

\makeatother

\renewcommand{\Pr}{\mathbb{P}}
\newcommand{\Exp}{\mathbb{E}}
\newcommand{\Var}{\mathrm{Var}}
\newcommand{\Cov}{\mathrm{Cov}}

\DeclareMathOperator*{\argmax}{arg\,max}
\DeclareMathOperator*{\argmin}{arg\,min}

\newcommand{\randvar}[1]{\mathsf{#1}} 
\newcommand{\vectorbm}[1]{\bm{#1}} 

\def\rvf{{\vectorbm{\randvar{f}}}}

\def\va{{\vectorbm{a}}}

\def\vf{{\vectorbm{f}}}

\def\vk{{\vectorbm{k}}}

\def\vm{{\vectorbm{m}}}
\def\vn{{\vectorbm{n}}}

\def\vp{{\vectorbm{p}}}

\def\vu{{\vectorbm{u}}}
\def\vv{{\vectorbm{v}}}
\def\vw{{\vectorbm{w}}}
\def\vx{{\vectorbm{x}}}
\def\vy{{\vectorbm{y}}}
\def\vz{{\vectorbm{z}}}

\def\vbeta{{\vectorbm{\beta}}}

\def\vgamma{{\vectorbm{\gamma}}}

\def\vxi{{\vectorbm{\xi}}}

\def\mA{{\vectorbm{A}}}

\def\mI{{\vectorbm{I}}}

\def\mK{{\vectorbm{K}}}

\def\mV{{\vectorbm{V}}}

\def\mPhi{{\vectorbm{\Phi}}}

\def\mSigma{{\vectorbm{\Sigma}}}

\DeclareMathAlphabet{\randvarit}{\encodingdefault}{\sfdefault}{m}{sl}
\SetMathAlphabet{\randvarit}{bold}{\encodingdefault}{\sfdefault}{bx}{n}

\def\gO{{\mathcal{O}}}

\renewcommand{\d}{\:d}

\usepackage{commath}
\usepackage{cleveref}
\usepackage{wrapfig}
\usepackage{algorithm}
\usepackage{algpseudocode}
\usepackage{afterpage}
\usepackage[dvipsnames, table]{xcolor}

\usepackage[edges]{forest}
\usepackage{makecell}

\definecolor{bluegraph}{rgb}{0.0, 0.33, 0.71}

\usepackage{placeins}

\usepackage{pifont}
\newcommand{\cmark}{\ding{51}}%
\newcommand{\xmark}{\ding{55}}%

\newcommand{\cmarkg}{\textcolor{green!50!black!50}{\cmark}} 
\newcommand{\xmarkr}{\textcolor{red}{\xmark}} 

\newtheorem{assumption-own}[theorem]{Assumption}

\title{Bayesian Quadrature}

\subtitle{Gaussian Processes for Integration}

\maintitleauthorlist{
Maren Mahsereci \\
Yahoo Research \\
maren.mahsereci@yahooinc.com
\and
Toni Karvonen \\
School of Engineering Sciences \\ Lappeenranta--Lahti University of Technology LUT \\
toni.karvonen@lut.fi
}

\issuesetup
{%
 copyrightowner={A.~Gupta and M.~Casey},
 volume        = 19,
 issue         = 2--3,
 pubyear       = 2026,
 isbn          = xxx-x-xxxxx-xxx-x,
 eisbn         = xxx-x-xxxxx-xxx-x,
 doi           = 10.1108/FTMAL-06-2026-0132,
 firstpage     = 121,
 lastpage      = 240
 }

\usepackage{mwe}

\author[1]{Mahsereci, Maren}
\author[2]{Karvonen, Toni}

\affil[1]{Yahoo Research; maren.mahsereci@yahooinc.com}
\affil[2]{School of Engineering Sciences, Lappeenranta--Lahti University of Technology LUT; toni.karvonen@lut.fi}

\articledatabox{\nowfntstandardcitation}

\begin{document}

\makeabstracttitle

\begin{abstract}
  Bayesian quadrature is a probabilistic, model-based approach to numerical integration, the estimation of intractable integrals, or expectations.
  Although Bayesian quadrature was popularised already in the 1980s, no systematic and comprehensive treatment has been published.
  The purpose of this survey is to fill this gap.
  We review the mathematical foundations of Bayesian quadrature from different points of view; present a systematic taxonomy for classifying different Bayesian quadrature methods along the three axes of modelling, inference, and sampling; collect general theoretical guarantees; and provide a controlled numerical study that explores and illustrates the effect of different choices along the axes of the taxonomy.
  We also provide a realistic assessment of practical challenges and limitations to application of Bayesian quadrature methods and include an up-to-date and nearly exhaustive bibliography that covers not only machine learning and statistics literature but all areas of mathematics and engineering in which Bayesian quadrature or equivalent methods have seen use.
\end{abstract}

\chapter{Introduction}
\label{sec:introduction}

Numerical computation of integrals, often representing expectations or normalization constants, is a pervasive practical problem throughout machine learning, statistics, scientific computing, and engineering.
In the form studied in this survey, numerical integration consists in approximating the definite integral
\begin{equation*}
  I_\measure(\f) = \int_\domain \f(\vx) \dif \measure(\vx)
\end{equation*}
of a real-valued integrand function $\f$ with respect to a probability measure $\measure$ on a set $\domain \subseteq \R^d$.
The approximation typically takes the form of a \emph{quadrature rule} $\sum_{i=1}^N w_i \f(\vx_i)$, which is a weighted sum of integrand evaluations at some nodes $\vx_i \in \domain$.
Numerical integration can be done on arbitrary non-Euclidean domains and can also incorporate information other than point evaluations, generalisations that we discuss in \Cref{sec:generalisations}.
Over the past two centuries a vast number of numerical integration techniques have been developed from different starting points and for different types of integration problems.
The most popular of these are classical \emph{Gaussian quadratures} that are constructed using polynomial exactness criteria and admit an elegant mathematical theory~\citep{Gautschi2004}; easy-to-use and widely applicable \emph{Monte Carlo integration} based on random sampling that is, no doubt, the most popular approach to approximating an integral~\citep{Caflisch1998}; and \emph{quasi-Monte Carlo methods} that use low-discrepancy sequences and can compute high-dimensional integrals effectively~\citep{DickKuoSloan2013}.
Other approaches include sparse grid methods, Clenshaw--Curtis quadrature, the trapezoidal rule, and the topic of this survey, \emph{Bayesian quadrature}.

Although its history can be traced further back, Bayesian quadrature was not popularized until the late 1980s by Persi Diaconis and Anthony O'Hagan~\citep{Diaconis1988, OHagan1991}.
Since then, there has been a steady stream of interest and contributions from machine learning and statistics communities~\citep[e.g.,][]{Kennedy1998, RasmussenGhahramani2002, Osborne2012a, Briol2019} and, more recently, from the point of view of numerical analysis and approximation theory~\citep{Rathinavel2019, Kanagawa2020, Santin2021}.
Bayesian quadrature falls within \emph{probabilistic numerics}~\citep{pnbook22}, its defining characteristic being the interpretation of numerical integration as a statistical inference problem to which the Bayesian paradigm and methods can be brought to bear.

In Bayesian quadrature, a stochastic process prior placed on the integrand is conditioned on the ``data'' consisting of integrand evaluations.
The mean of the resulting posterior distribution provides a point estimate for the integral, while the spread of the posterior quantifies uncertainty (see Figure~\ref{fig:bq-illustration}).
That \emph{explicit prior information}, such as
smoothness or structural properties, is easy to encode both conceptually and in practice by selection of an appropriate prior for the integrand distinguishes Bayesian quadrature from other numerical integration techniques.
Priors that are ``correct'' or ``good'' result in fast rates of convergence and reliable quantification of uncertainty (see Figure~\ref{fig:intro-integration-example}).
Non-Bayesian integration methods too encode various types of prior information, but typically in a non-systematic and implicit manner.
For example, Gaussian quadratures implicitly assume that the integrand is well approximated by polynomials; the trapezoidal rule that a sum of trapezoids approximates the integral; and Monte Carlo encodes no assumptions whatsoever besides square-integrability.
While quasi-Monte Carlo methods use node placement to encode certain assumptions, Bayesian quadrature can couple \emph{any nodes with any prior information}, making it extremely versatile.
That any nodes proposed in the vast literature on numerical integration can be used in Bayesian quadrature is one of its great advantages.
Another distinguishing feature is the \emph{statistically principled uncertainty quantification} that Bayesian quadrature provides: the spread of the posterior tells one how much to trust the integral estimate.
Note that this uncertainty quantification arises from the prior and is wholly distinct in character from the sample variance of Monte Carlo.

\begin{figure}[t]
  \centering
  \includegraphics[width=0.8\textwidth]{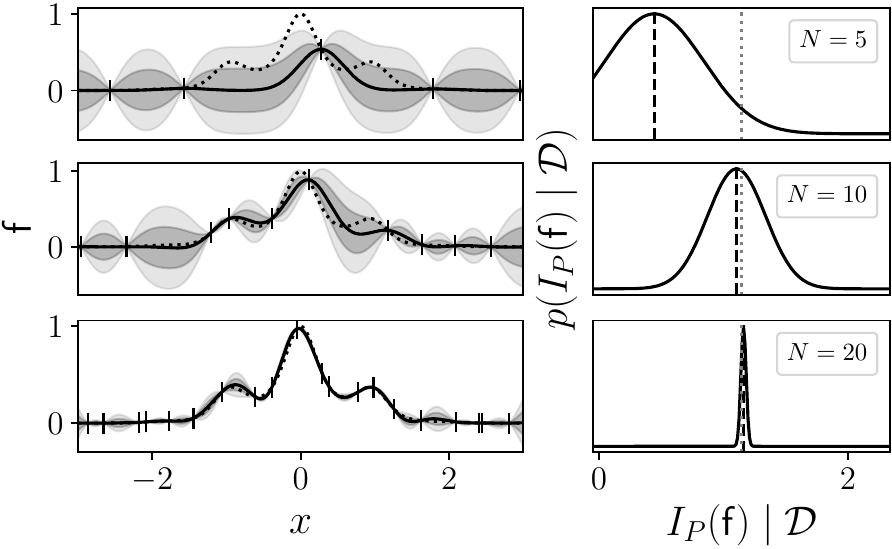}
  \caption{A sketch of Bayesian quadrature: A distribution over a function $\f$ gives rise to a distribution over the integral $I_\measure(\f) = \int_\domain f(x) \dif \measure(x)$.}
  \label{fig:bq-illustration}  
\end{figure}

Much methodology and theory has been developed for Bayesian quadrature over the years (worthy mentions include \citealt{Xi2018}; \citealt{Chai2019a}; and \citealt{Kanagawa2020}), and applications have ranged from computer graphics~\citep{Marques2015} and engineering~\citep{Kumar2008, Dang2021} to cardiac and tsunami modelling~\citep{Oates2017a, Li2022}.
Despite a wealth of such contributions, no comprehensive treatment of Bayesian quadrature as a whole has ever been published.
Perhaps partly for this reason, confusion continues to plague the field.
For example, ``Bayesian quadrature'' is often used to refer exclusively to a single \emph{method} (or some collection of methods) to tackle integration problems of a very specific form.
Although comparisons of Bayesian quadrature methods to other numerical integration techniques abound in the literature, no systematic empirical study in a controlled environment on the effect of the multitude of choices on modelling, inference, and sampling that go into designing a Bayesian quadrature method has been undertaken.
This is a pity, for the richness of these choices is one of the main attractions (and often a great source of frustration in application) of Bayesian quadrature.
The purpose of this survey is to fill this gap and provide a comprehensive survey of Bayesian quadrature that covers both the theory and practice of Bayesian quadrature.

\begin{figure}[t]
  \centering
  \includegraphics[width=\textwidth]{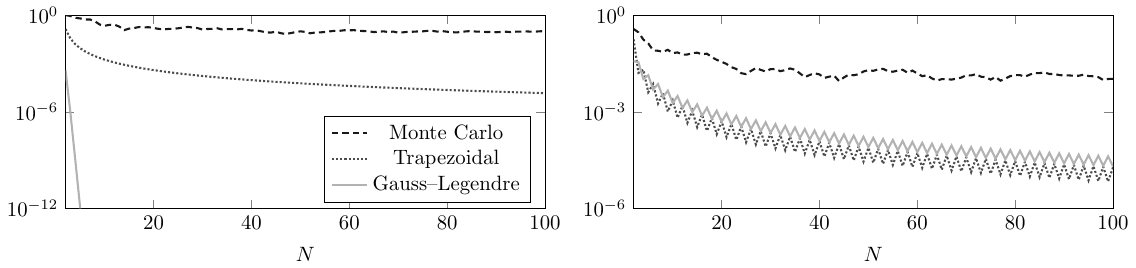}
  \caption{Errors when the smooth function $f_1(x) = \exp(x)$ (\emph{left}) and non-smooth function $f_2(x) = \operatorname{exp}(-\lvert x - \tfrac{1}{2} \rvert )$ (\emph{right}) are integrated on the interval $[0, 1]$ with Monte Carlo, the trapezoidal rule, and the Gauss--Legendre quadrature. Because $f_1$ has a Taylor series that converges fast, Gauss--Legendre reaches machine precision almost immediately as it assumes that $f_1$ resembles a polynomial. The only assumption the trapezoidal rule makes is that the area under the graph of $f_1$ can be approximated by trapezoids. Monte Carlo assumes essentially nothing and, as a result, converges very slowly.
  The function $f_2$ does not resemble a polynomial and so the assumptions Gauss--Legendre do not benefit it.
  \emph{Bayesian quadrature makes prior assumptions such as these explicit and puts them on a systematic probabilistic footing.}
  }
  \label{fig:intro-integration-example}
\end{figure}

\section{Contents}
\label{sec:contributions}

There are six main elements that make up the contents of this survey.

\paragraph{I --- \Cref{sec:bq}: Foundations.}
  This chapter contains a foundational introduction to Bayesian quadrature.
  In particular, the chapter contains what appears to be the first attempt at properly defining Bayesian quadrature (see Definition~\ref{def:bq}) in a way that encompasses the modern usage of the term in machine learning literature.
  In our parlance a Bayesian quadrature refers to numerical integration based on \emph{Gaussian process modelling}.
  That is, a Bayesian quadrature method is any numerical integration method that computes or approximates a posterior distribution for the integral $I_\measure(\fGP)$ of a potentially transformed Gaussian process $\fGP$ (i.e., $\fGP = \varphi \circ \gGP$ for a mapping $\varphi$ and a Gaussian process $\gGP$).
  The definition gives rise to a natural division of Bayesian quadrature methods in (a) \emph{conjugate methods} for which $\fGP$ is a Gaussian process (i.e., $\varphi$ is affine) and (b) \emph{non-conjugate methods} for which $\fGP$ is not a Gaussian process (i.e., $\varphi$ is non-affine).
  This anticipates the development of more detailed taxonomy in \Cref{sec:taxonomy}.
  We also discuss how the theory of reproducing kernel Hilbert spaces can be used to interpret Bayesian quadrature, an interpretation which underlies most theoretical analysis, and provide high-level theory that is independent of the underlying Gaussian process model.

\paragraph{II --- \Cref{sec:taxonomy}: Taxonomy.}
  This chapter contains an extensive taxonomy for Bayesian quadrature that provides a well-defined framework for classifying and categorizing concrete algorithms and implementations.
  Inspiration for the taxonomy comes from the literature on global optimization~\citep{Jones2001}.
  We classify Bayesian quadrature methods along three main axes based on~(i) what kind of \emph{model} they employ, (ii) how they approach \emph{inference}, and~(iii) how they select, or \emph{sample}, the nodes $\vx_i$ at which the integrand is evaluated.
  \Cref{fig:tax-graph} contains an overview of the taxonomy and its most important labels.
  The taxonomy illustrates that Bayesian quadrature is an umbrella term and that there is a great variety of different Bayesian quadrature \emph{methods} that differ significantly from one another.

\paragraph{III --- \Cref{sec:practical-issues}: Practice.}
Like all statistical and numerical methods, Bayesian quadrature methods are faced with practical problems that limit their applicability. For Bayesian quadrature, the dominant problems are those of computational cost and computation of kernel means, as one needs to solve linear systems of $N$ equations and compute integrals that involve the covariance function of the Gaussian process model. Although every practitioner of Bayesian quadrature is inevitably confronted by these problems, they are rarely emphasised or treated in a systematic manner. \Cref{sec:practical-issues} reviews what we consider the most important practical problems that a practitioner needs to solve in order to implement a Bayesian quadrature method and offer some solutions. One should recognise that Bayesian quadrature can never hope to be as flexible or general as, for example, Monte Carlo integration. This is understandable, for Bayesian quadrature is based on principled statistical modelling, which must always have a cost.

\paragraph{IV --- \Cref{sec:guarantees}: Theoretical guarantees.}
Various theoretical convergence guarantees, given in terms of upper bounds on the integration error under certain assumptions, have appeared in the literature~\citep[e.g.,][]{Briol2019, Kanagawa2020, Wynne2021}.
In \Cref{sec:guarantees}, we provide somewhat more practical versions of these guarantees by (i) making the dependency of the bounds explicit on those parameters of the covariance function that are commonly estimated from the data and by (ii) including the effect regularisation via a nugget term has. The proofs, which appear in \Cref{sec:proofs}, are based on results in the scattered data approximation literature and do not materially differ from those of the existing results.

\paragraph{V --- \Cref{sec:experiments}: Empirical illustrations.}   
  Here we put the taxonomy and theory from \Cref{sec:taxonomy,sec:guarantees} in action by conducting extensive and controlled numerical experiments along the three main axes of the taxonomy.
  That is, we systematically study how the model, inference methodology, and sampling scheme affect the accuracy and calibration of a Bayesian quadrature method, as well as how the different choices interact.
  In the experiments we use a collection of test integrands whose properties, such as differentiability, are well-understood.
  The experiments validate some of the theoretical guarantees.
  While almost every article on Bayesian quadrature contains an empirical study, these typically focus on comparing the new method to other candidates, such as Monte Carlo or quasi-Monte Carlo, in realistic---and hence opaque---integration problems that do not easily lend themselves to untangling the reasons behind one method outperforming the other.
  The purpose of this chapter is not to demonstrate the superiority of Bayesian quadrature in general or any particular method to alternative numerical methods but to highlight the effect of different approaches and choices \emph{within} Bayesian quadrature.

\paragraph{VI --- Comprehensive bibliography.}
We do not conduct a systematic literature review. However, as a survey article such as this is naturally equipped with a long bibliography, we could not resist taking this to extremes and including every reference, no matter how obscure, on Bayesian quadrature or equivalent approaches known to us in machine learning, statistics, engineering, and mathematics literature. Consequently, in most cases we make little to no effort at providing more than a modicum of context for the references.

\section{Notational conventions}
\label{sec:notation}

The set of non-negative integers is denoted $\N$.
We strive to follow the following type conventions:
objects in \emph{italics} are scalars or scalar-valued functions ($x$ or $f$); {\sffamily sans serif} is used to indicate random variables and stochastic processes ($\fGP$); \textbf{boldface} objects are vectors (if lowercase; $\vx$) or matrices (if uppercase; $\mK$); and \textbf{{\sffamily boldface sans serif}} stands for a random vector ($\rvf$).
In \Cref{sec:proofs-isotropic-matern}, we use $\lesssim$ to denote asymptotic inequality.
This notation is equivalent to big $O$ notation.

\chapter{Bayesian quadrature}
\label{sec:bq}

This chapter is an introduction to Bayesian quadrature, which we define (see Definition~\ref{def:bq}) as any method for computing or approximating the integral of a potentially transformed Gaussian process $\fGP$ conditioned on a number of evaluations of the integrand $\f$.
We further divide Bayesian quadrature methods to
\begin{itemize}
\item \emph{conjugate methods} when the posterior over the integral is normally distributed and 
\item \emph{non-conjugate methods} when the posterior is non-normal.
\end{itemize}
A detailed classification of the existing Bayesian quadrature methods is given in \Cref{sec:taxonomy}.
To precede the construction of Bayesian quadrature in \Cref{sec:bq-construction} we include a short review of Gaussian processes in \Cref{sec:gps}.
Moreover, in \Cref{sec:conn-appr-theory} we review the equivalence between conjugate Bayesian quadrature and worst-case optimal quadrature rules in reproducing kernel Hilbert spaces.
This equivalence is crucial for deriving the convergence guarantees,as well as many other theoretical results, that we present in \Cref{sec:guarantees}.
Throughout this chapter it is assumed that both the set of nodes $\points$ at which the integrand is to be evaluated and the transformed Gaussian process $\fGP$ (in particular, its covariance function $\kernel$) used to model the integrand are fixed; different ways to select them are reviewed in \Cref{sec:taxonomy,sec:practical-issues}.

\section{Gaussian processes}
\label{sec:gps}

\citet{RasmussenWilliams2006} provide a standard introduction to Gaussian processes for machine learning.
For more theoretical expositions, see \citet{Bogachev1998} and \citet{Berlinet2004}.
Let $\kernel \colon \domain \times \domain \to \R$ be a symmetric positive-semidefinite \emph{covariance function} (or \emph{kernel}) on $\domain \subseteq \R^d$, which is to say that the \emph{kernel Gram matrix}
\begin{equation*}
  \mK_{\points\points} = (\kernel(\vx_i, \vx_j))_{i,j=1}^N \in \R^{N \times N}
\end{equation*}
is positive-semidefinite for any $N \in \N$ and any $\points = \{\vx_1, \ldots, \vx_N\} \subseteq \domain$.
Throughout this survey we tacitly assume that $\points$ are such that the Gram matrix is positive-definite and hence invertible.
Most covariance function that we consider are \emph{positive-definite}, in that their Gram matrix is positive-definite whenever the points $\points$ are pairwise distinct.

Let $\mean \colon \domain \to \R$ be a function.
A Gaussian process $\gGP$ with mean $\mean$ and covariance function $\kernel$ is denoted
\begin{equation*}
  \gGP \sim \GP(\mean, \kernel).
\end{equation*}
Suppose that one wishes to do regression and is given a dataset
\begin{equation*}
  \dataset = \{ (\vx_1, y_1), \ldots (\vx_N, y_N) \}
\end{equation*}
consisting of some noiseless observations $\vy = (y_1, \ldots, y_N) \in \R^N$ at pairwise distinct nodes $X = \{\vx_1, \ldots, \vx_N\} \subseteq \domain$.
The mean and covariance of the conditional Gaussian process, $\gGP \mid \dataset$, are
\begin{equation}
  \begin{split}
  \label{eq:gp-posterior-mean}
  \mean_\dataset(\vx) = \Exp[ \gGP(\vx) \mid \dataset ] = \mean(\vx) + \vk_\points(\vx)\Trans \mK_{\points\points}^{-1} ( \vy - \vm_\points)
  \end{split}
\end{equation}
and
\begin{equation}
  \begin{split}
  \label{eq:gp-posterior-covariance}
  \kernel_\dataset( \vx, \vx' ) &= \Cov[ \gGP(\vx), \gGP(\vx') \mid \dataset ] \\
  &= \kernel(\vx, \vx') - \vk_\points(\vx)\Trans \mK_{\points\points}^{-1} \vk_\points(\vx'),
  \end{split}
\end{equation}
where $\vm_\points = (\mean(\vx_1), \ldots, \mean(\vx_N)) \in \R^N$ and $\vk_\points \colon D \to \R^N$ is a vector-valued function defined as
\begin{equation*}
  \vk_\points(\vx) = ( \kernel(\vx, \vx_1), \ldots, \kernel(\vx, \vx_N) ) \in \R^N.
\end{equation*}

A covariance function is \emph{stationary} if it only depends on the difference of the inputs and \emph{isotropic} if it depends only on their distance. 
That is, a covariance $\kernel$ is stationary if there is a function $\Phi_k \colon \domain \to \R$ such that $k(\vx, \vx') = \Phi_k(\vx - \vx')$ for all $\vx, \vx' \in \domain$ and isotropic if $\Phi_\kernel \colon [0, \infty) \to \R$ and $\kernel(\vx, \vx') = \Phi_\kernel( \norm[0]{ \vx - \vx'} )$.
Commonly used covariance functions on any $D \subseteq \Rd$ include the \emph{square exponential}
\begin{equation}
  \label{eq:se-kernel}
  \kernel(\vx, \vx') = \sigma^2 \exp\bigg( \! - \frac{\norm[0]{\vx - \vx'}^2}{2 \ell^2} \bigg)
\end{equation}
and covariance functions of the \emph{Matérn class}
\begin{equation}
  \label{eq:matern-kernel}
  \kernel(\vx, \vx') = \sigma^2 \frac{2^{1-\nu}}{\Gamma(\nu)} \bigg( \frac{\sqrt{2\nu} \norm[0]{\vx - \vx'}}{\ell} \bigg)^\nu K_\nu \bigg( \frac{\sqrt{2\nu} \norm[0]{\vx - \vx'}}{\ell} \bigg),
\end{equation}
where $\Gamma$ denotes the gamma function and $K_\nu$ the modified Bessel function of the second kind of order $\nu > 0$.
\begin{figure}[t]
  \centering
  \includegraphics[width=\textwidth]{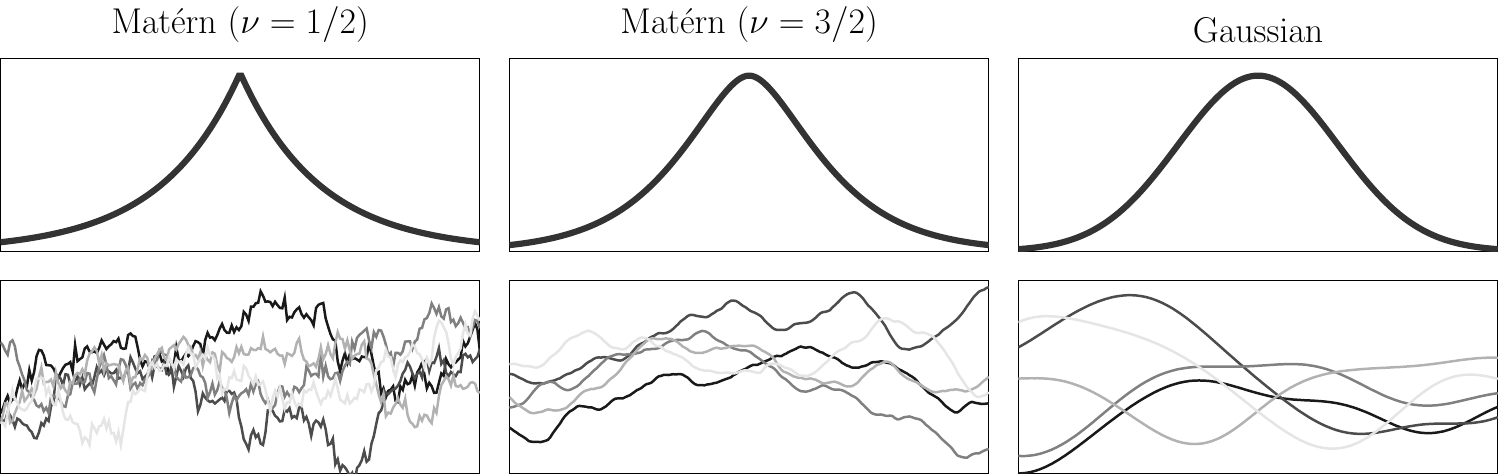}
  \caption{\emph{Top:} Three isotropic covariance functions $\kernel(\cdot, 0)$. \emph{Bottom:} Samples from the corresponding GPs. Note how smooth covariance functions yield smooth samples.}
  \label{fig:matern-illustration}
\end{figure}
Both the square exponential and all Matérn covariance functions are isotropic.
\emph{Product Matérn} covariance functions, which are defined as dimension-wise products of one-dimensional Matérn kernels (and are therefore stationary),
\begin{equation}
  \label{eq:product-matern-kernel}
  \kernel(\vx, \vx') = \sigma^2 \prod_{i=1}^d \frac{2^{1-\nu}}{\Gamma(\nu)} \bigg( \frac{\sqrt{2\nu} \abs[0]{x_i - x_i'}}{\ell} \bigg)^\nu K_\nu \bigg( \frac{\sqrt{2\nu} \abs[0]{x_i - x_i'}}{\ell} \bigg),
\end{equation}
are also often used.
Note that the square exponential kernel~\eqref{eq:se-kernel} is the product of one-dimensional square-exponential kernels because $\norm[0]{\vx}^2 = x_1^2 + \cdots + x_d^2$.\footnote{In fact, the square exponential covariance function is the \emph{only} isotropic covariance function that is the product of its one-dimensional versions~\citep[pp.\@~55-56]{Stein1999}.}
The \emph{Brownian motion} covariance function
\begin{equation}
  \label{eq:brownian-kernel}
  k(x, x') = \min\{x, x'\},
\end{equation}
which is defined on $D = [0, T] \subseteq \R$ for any $T > 0$ is important due to its connections to spline interpolation and the trapezoidal rule.
The Brownian motion kernel is not stationary.
\Cref{fig:bq-illustration} shows three different covariance functions and GPs.
Although most covariance functions used in Bayesian quadrature are stationary, there are several exceptions besides the Brownian motion kernel.
For example, \citet{Fitzsimons2017} use a piecewise constant covariance function.

\section{Construction of Bayesian quadrature}
\label{sec:bq-construction}

Let $\transformation \colon \R \to \R$ be a function. Bayesian quadrature methods proceed by taking a Gaussian process prior $\gGP \sim \GP(\mean, \kernel)$ with a prior mean $m$ and covariance function $\kernel$ and defining a prior $\fGP$ for the integrand $\f$ by mapping $\gGP$ through $\transformation$:
\begin{equation} \label{eq:fGP-transformation}
  \fGP = \transformation \circ \gGP.
\end{equation}
Let the dataset consist of noiseless evaluations of the integrand $\f$ at nodes $\vx_1, \ldots, \vx_N$:
\begin{equation*}
  \dataset = \{ (\vx_1, \f(\vx_1)), \ldots, (\vx_N, \f(\vx_N)) \}.
\end{equation*}
We call any method that computes or approximates the posterior distribution over $I_\measure(\fGP) \mid \dataset$ a Bayesian quadrature (BQ) method.

\begin{definition}[Bayesian quadrature] 
  \label{def:bq}
  Let $\gGP$ be a Gaussian process and $\transformation \colon \R \to \R$ a function.
  Any method that computes or approximates the posterior distribution
  \begin{equation}
    \label{eq:bq-posterior-def}
    I_\measure(\fGP) \mid \dataset = \int_\domain \fGP(\vx) \dif \measure(\vx) \: \Bigl\vert \: \dataset
  \end{equation}
  over the integral of the process $\fGP = \transformation \circ \gGP $ is a \emph{Bayesian quadrature method}\footnote{The term \emph{Bayesian quadrature} was popularised by \citet{OHagan1991}. The earliest occurrence that we have found is in \citet{OHagan1988}. \citet[p.\@~164]{Diaconis1988} indicates an origin no later than 1985. In the early literature, \emph{Bayes--Hermite quadrature} is used to refer to conjugate BQ when the GP prior has a square exponential covariance and $\measure$ is a Gaussian distribution~\citep{OHagan1991,KennedyOHagan1996,Kennedy2000}.}.
\end{definition}
The integral posterior mean $\Exp[ I_\measure(\fGP) \mid \dataset ]$ or its approximation is used as a point estimate of the true integral $I_\measure(f)$.
The integral posterior variance, $\Var[ I_\measure(\fGP) \mid \dataset ]$, is typically used to summarise the spread of $I_\measure(\fGP) \mid \dataset$.
We shall always tacitly assume that the process $\fGP$, the domain $\domain$, and the probability measure $\measure$ are such that the posterior distribution in~\eqref{eq:bq-posterior-def} is well-defined and has finite mean and variance.
If the function $\varphi$ is affine, it suffices that $\mean$ and $\kernel(\cdot, \vx)$ be integrable with respect to $\measure$ for every $\vx \in \domain$ and \smash{$\int_\domain \sqrt{\smash[b]{\kernel(\vx, \vx)}} \dif \measure(\vx) < \infty$}.
For example, the integral is finite if the covariance function is stationary and $\measure$ has a bounded density on a bounded $\domain$ since in that case $k(\vx, \vx) = \Phi_k(\bm{0})$ for all $\vx \in \domain$.

The selection of $\gGP$ and $\transformation$ is supposed to be informed about any prior information one may have about the true integrand $\f$.
The prior information that is typically encoded includes different structural properties, such as stationarity, non-negativity or periodicity, and smoothness, described by the number of continuous derivatives that one expects $\f$ to possess.
For example, when $\fGP = \gGP$ and the covariance function of $\gGP$ is a Matérn in~\eqref{eq:matern-kernel} with smoothness $\nu$, the samples of $\fGP$ essentially have smoothness $\nu$ in a Sobolev sense; see~\citet{Scheuerer2010, Steinwart2019, Henderson2022}, as well as~\citet[Corollary~4.15]{Kanagawa2018}.
However, \emph{model misspecification}, to be understood here in a broad sense that $\f$ ``looks'' different than the samples from $\fGP$, is inevitable in anything more complicated than toy examples.
Practical selection of the prior is a general challenge in Bayesian inference and not limited to Bayesian quadrature or inference with Gaussian processes.
Prior selection, when simplified to mean the selection of a limited number of kernel parameters, such as $\sigma$ and $\ell$ of the square exponential kernel in~\eqref{eq:se-kernel}, is briefly reviewed in \Cref{sec:param-estimation}.

\subsection{Conjugate Bayesian quadrature (affine $\transformation$)}
\label{sec:conj-bq}

The prior $\fGP$ is a GP if the function $\transformation$ is affine.
In this case Gaussian conjugate inference is possible and we accordingly call the resulting methods \emph{conjugate Bayesian quadrature methods}.
This class of methods is sometimes called \emph{standard Bayesian quadrature}~\citep[Sec.\@~3.2]{Karvonen2019a} or \emph{vanilla Bayesian quadrature}~\citep[Sec.\@~2.1]{Gessner20}.
Because an affine $\transformation$ merely shifts and scales the mean and covariance of $\gGP$, we can set $\transformation$ to be the identity function without loss of generality.
This gives the GP prior
\begin{equation*}
  \fGP = \mathrm{id} \circ \gGP = \gGP \sim \GP(\mean, \kernel).
\end{equation*}
Due to the linearity of integration, the posterior for $I_\measure(\fGP)$ is a normal random variable that is obtained by simply integrating the posterior process $\fGP \mid \dataset = \gGP \mid \dataset$:
\begin{equation*}
  I_\measure(\fGP) \mid \dataset = I_\measure(\fGP \mid \dataset) = \int_\domain [ \fGP(\vx) \mid \dataset ] \dif \measure(\vx) \sim \Normal( \mu_{\dataset}, \Sigma_{\dataset} ).
\end{equation*}
The integral mean
\begin{equation}
  \label{eq:bq-mean}
  \begin{split}
  \mu_{\dataset} = \Exp[ I_\measure(\fGP) \mid \dataset ] &= \int_\domain m_\dataset(\vx) \dif \measure(\vx) \\
  &= m_\measure + \vk_{\measure\points}\Trans \mK_{\points\points}^{-1} ( \vf_\points - \vm_\points),
  \end{split}
\end{equation}
which is the integral of the posterior mean in~\eqref{eq:gp-posterior-mean}, is expressed in terms of the integral $m_\measure = I_\measure(m)$ of the prior mean function $m$, the integrand evaluations $\vf_\points = (\f(\vx_1), \ldots, f(\vx_N)) \in \R^N$, and the evaluations $\vk_{\measure\points} = I_\measure(\vk_\points) = ( \kernel_\measure(\vx_1), \ldots \kernel_\measure(\vx_N)) \in \R^N$ of the \emph{kernel mean embedding}
\begin{equation}
  \label{eq:kernel-mean}
  \kernel_\measure(\vx) = \int_\domain \kernel(\vx', \vx) \dif \measure(\vx').
\end{equation}
The integral variance
\begin{equation}
  \begin{split}
  \label{eq:bq-var}
    \Sigma_{\dataset} = \Var[ I_\measure(\fGP) \mid \dataset ] &= \int_\domain \int_\domain \kernel_\dataset(\vx, \vx') \dif \measure(\vx) \dif \measure(\vx') \\
    &= \kernel_{\measure\measure} - \vk_{\measure\points}\Trans \mK_{\points\points}^{-1} \vk_{\measure\points},
  \end{split}
\end{equation}
where
\begin{equation}
  \label{eq:kernel-var}
  \kernel_{\measure\measure} = \int_\domain \int_\domain \kernel(\vx, \vx') \dif \measure(\vx) \dif \measure(\vx') = \int_\domain \kernel_\measure(\vx) \dif \measure(\vx)
\end{equation}
is the \emph{initial variance}, is the integral of the posterior covariance in~\eqref{eq:gp-posterior-covariance}.
We encapsulate conjugate Bayesian quadrature and its posterior equations in the following definition.

\begin{definition}[Conjugate Bayesian quadrature]
  If $\transformation \colon \R \to \R$ is affine, so that $\fGP= \transformation \circ \gGP$ is a Gaussian process, Bayesian quadrature is \emph{conjugate}.
  Without a loss of generality we can set $\gGP \sim \GP(m, k)$ and $\transformation = \operatorname{id}$.
  Then
  \begin{equation*}
    I_\measure(\fGP) \mid \dataset = I_\measure(\fGP \mid \dataset) = \int_\domain [ \fGP(\vx) \mid \dataset ] \dif \measure(\vx) \sim \Normal( \mu_{\dataset}, \Sigma_{\dataset} ),
  \end{equation*}
  where $\mu_{\dataset}$ and $\Sigma_{\dataset}$ given in~\eqref{eq:bq-mean} and~\eqref{eq:bq-var} are the \emph{integral mean} and \emph{variance} of conjugate Bayesian quadrature.
\end{definition}

Both the integral mean and variance can be expressed in terms of the weight vector $\vw = \mK_{\points\points}^{-1} \vk_{\measure\points}$ as
\begin{equation}
  \label{eq:conj-bq-mean-var}
  \mu_\dataset = m_\measure + \vw\Trans ( \vf_\points - \vm_\points) \quad \text{ and } \quad \Sigma_\dataset = k_{\measure\measure} - \vw\Trans \vk_{\measure\points}.
\end{equation}
Some of the weights $\vw = (w_1, \ldots, w_N)$ can be negative~\citep{HuszarDuvenaud2012, KarvonenKanagawa2019}.
For a zero-mean prior the mean is simply $\mu_\dataset = \vw\Trans \vf_\points$, which means that it is a \emph{quadrature rule}, an approximation of the integral in the form of a weighted sum of integrand evaluations.
Although it is customary to use the term \emph{cubature rule} to describe such an approximation when $d > 1$~\citep[see][for some discussion]{Cools1997}, this distinction is not consistently made in the literature on Bayesian quadrature.
For articles where \emph{Bayesian cubature} is used, see \citet{Karvonen2018, Rathinavel2019, Jagadeeswaran2022, Fisher2020}.

\subsection{Non-conjugate Bayesian quadrature (non-affine $\transformation$)}
\label{sec:non-conj-bq}

When $\transformation$ is not affine, we speak of non-conjugate Bayesian quadrature.

\begin{definition}[Non-conjugate Bayesian quadrature] 
  If $\transformation \colon \R \to \R$ is non-affine and $\fGP = \transformation \circ \gGP$ is non-Gaussian\footnote{See \citet{Mase1977} and \citet{KhatriMukerjee1987} for some theorems on preservation of Gaussianity under transformations.}, Bayesian quadrature is \emph{non-conjugate}.
  No general expression exists for the posterior $I_\measure(\fGP) \mid \dataset$.
\end{definition}

Non-conjugate Bayesian quadrature is usually motivated by a desire to model the non-negativity of the integrand, a property which cannot be encoded by a GP model~\citep{Osborne2012a, Osborne2012b, HamrickGriffiths2013, Gunter2014}.
Functions used to produce a non-negative $\fGP$ include
\begin{equation*}
  \transformation(x) = \alpha + x^2 \quad \text{ and } \quad \transformation(x) = \exp(x),
\end{equation*}
where $\alpha > 0$, and a general framework for encoding range constraints is described by \citet{Chai2019a}.
See \citet{ZhouPeng2020} and \citet{Hamid2023} for applications to engineering and neural networks.
However, when $\transformation$ is non-affine (as the two functions above are), the prior $\fGP = \transformation \circ \gGP$ is usually not a Gaussian process and the posterior $I_\measure(\fGP) \mid \dataset$ is not in general available in closed form, which precludes exact inference.
Intractability of the posterior and its moments introduces a significant complication over conjugate Bayesian quadrature, forcing one to approximate the posterior.
Gaussian approximations based on linearisation or moment-matching are typically used, but one can employ any posterior approximation available in the literature.

\subsection{Bayesian probabilistic numerical integration}
\label{sec:bpni}

In the definition of Bayesian quadrature (\Cref{def:bq}) the integrand is modelled as a potentially non-linearly transformed Gaussian process $\fGP = \transformation \circ \gGP$.
Nothing stops one from using \emph{any} suitable stochastic process $\fGP$.
The following definition is adapted from \citet[Sec.\@~1]{Zhu2020}.

\begin{definition}[Bayesian probabilistic numerical integration]
Let $\fGP$ be a stochastic process.
Any method that computes or approximates the posterior distribution
\begin{equation} \label{eq:BPNI}
  I_\measure(\fGP) \mid \dataset = \int_\domain \fGP(\vx) \dif \measure(\vx) \: \Bigl\vert \: \dataset
\end{equation}
is a \emph{Bayesian probabilistic numerical integration (BPNI) method}.
\end{definition}

All Bayesian quadrature methods are BPNI methods.
Due to computational tractability issues, few non-Gaussian BPNI methods exist.
Student's $t$-processes, which differ little from Gaussian processes in practice, have been used in numerical integration by \citet[Sec.\@~2.2]{OHagan1991} and \citet{Pruher2017}.
Integration methods based on Bayesian additive regression trees introduced by \citet{Zhu2020} appear to be the only non-Gaussian BPNI methods that substantially differ from Bayesian quadrature.
There is much room for the development of new BPNI methods, particularly when integrand evaluations constitute the computational bottleneck.
One can consider even more general and difficult integration problems in which the distribution $P$ is unknown and accessible only via sampling or black-box evaluation of the density function~\citep{Oates2017a,Fernandez2020}.

\section{Connections to approximation theory}
\label{sec:conn-appr-theory}

Given the well known equivalences of Gaussian process regression to optimal approximation in a \emph{reproducing kernel Hilbert space} (RKHS) and to \emph{kernel interpolation}~\citep{KimeldorfWahba1970, Scheuerer2013, Kanagawa2018}, it is no surprise that these equivalences extend to conjugate Bayesian quadrature.
We briefly describe the equivalences because of their historical value and because much of the theory for Bayesian quadrature, summarised in Sections~\ref{sec:basic-theory} and~\ref{sec:guarantees}, exploits them.
See \citet[Sec.\@~2]{Briol2019} and \citet[Sec.\@~2]{Kanagawa2020} for earlier reviews on the topic.

\subsection{Optimal integration in a reproducing kernel Hilbert space}
\label{sec:integration-in-rkhs}

Let $\rkhs(\kernel)$ denote the RKHS of $k$.
This is a certain Hilbert space consisting of real-valued functions defined on $\domain$ in which the kernel has the following \emph{reproducing property}: $\langle g, k(\cdot, \vx) \rangle_{\rkhs(k)} = g(\vx)$ for every $g \in \rkhs(\kernel)$ and $\vx \in \domain$.
For example, the RKHS of a Matérn kernel~\eqref{eq:matern-kernel} of order $\nu$ is, up to an equivalent norm, the Sobolev space $H^{\nu + d/2}(\R^d)$, a fact that is discussed in more detail in Section~\ref{sec:guarantees-isotropic-matern}.
We refer to \citet{Berlinet2004}, \citet{Steinwart2008}, and \citet{Paulsen2016} for the theory of RKHSs.
Section~\ref{sec:rkhs-appendix} contains a very concise review.

The \emph{worst-case integration error} of any weights $\vv \in \R^N$ and points $\points = \{\vx_1, \ldots, \vx_N\} \subset \domain$ in $\rkhs(k)$ is defined as
\begin{equation}
  \label{eq:wce-definition}
  e_k(\vv, \points) = \sup_{ \norm[0]{g}_{\rkhs(k)} \leq 1} \, \Big \lvert \int_\domain g(\vx) \dif \measure(\vx) - \sum_{i=1}^N v_i g(\vx_i) \Big \rvert.
\end{equation}
It follows from the reproducing property that the worst-case error can be expressed in terms of the kernel.

\begin{lemma} \label{lemma:wce}
  Suppose that $k(\cdot, \vx)$ is measurable with respect to $\measure$ for every $\vx \in \domain$ and that $\int_\domain \sqrt{\smash[b]{k(\vx, \vx)}} \dif \measure(\vx) < \infty$.
  Then 
  \begin{equation}
  \begin{split}
  \label{eq:wce-explicit}  
  e_k(\vv, \points) &= \norm[0]{\kernel_\measure - \textstyle\sum_{i=1}^N v_i \kernel(\cdot, \vx_i) }_{\rkhs(\kernel)} \\
  &= \sqrt{\smash[b]{ \kernel_{\measure\measure} - 2\vv\Trans\vk_{\measure\points} + \vv\Trans \mK_{\points\points} \vv}}.
  \end{split}
\end{equation}
\end{lemma}
\begin{proof}
  The measurability assumption implies that all $g \in \rkhs(k)$ are measurable~\citep[Lem.\@~4.24]{Steinwart2008}.
  Let $g \in \rkhs(k)$. 
  By the reproducing property and the Cauchy--Schwarz inequality,
  \begin{equation*}
    \begin{split}
    \Big \lvert \int_\domain g(\vx) \dif \measure(\vx) \Big\rvert &= \Big \lvert \int_\domain \langle g, k(\cdot, \vx) \rangle_{\rkhs(k)} \dif \measure(\vx) \Big\rvert \\
    &\leq \lVert g \rVert_{\rkhs(k)} \int_\domain \sqrt{\smash[b]{k(\vx, \vx)}} \dif \measure(\vx),
    \end{split}
  \end{equation*}
  where we used $\lVert k(\cdot, \vx) \rVert_{\rkhs(k)}^2 = \langle k(\cdot, \vx), k(\cdot, \vx) \rangle_{\rkhs(k)} = k(\vx, \vx)$.
  It follows that every $g \in \rkhs(k)$ is integrable with respect to $\measure$.
  As $k(\cdot, \vx) \in \rkhs(k)$ for every $\vx \in \domain$, the mean embedding given by $k_\measure(\vx) = \int_\domain k(\vx', \vx) \dif \measure(\vx')$ is well-defined.
  From the Riesz representation theorem it follows that $k_\measure \in \rkhs(k)$ and $I_\measure(g) = \langle g, k_\measure \rangle_{\rkhs(k)}$ for every $g \in \rkhs(k)$; see Section~\ref{sec:rkhs-appendix} for details.
  Therefore the RKHS norm in~\eqref{eq:wce-explicit} is well-defined.

  To prove the equality, use the Cauchy--Schwarz inequality to obtain
  \begin{equation*}
    \begin{split}
    e_k(\vv, \points) &= \sup_{ \norm[0]{g}_{\rkhs(k)} \leq 1} \, \lvert \langle g, k_\measure \rangle_{\rkhs(k)} -  \langle g, \textstyle\sum_{i=1}^N v_i k(\cdot, \vx_i) \rangle_{\rkhs(k)} \rvert \\
    &\leq \norm[0]{\kernel_\measure - \textstyle\sum_{i=1}^N v_i \kernel(\cdot, \vx_i) }_{\rkhs(\kernel)} ,
    \end{split}
  \end{equation*}
  which equals \smash{$\sqrt{\smash[b]{ \kernel_{\measure\measure} - 2\vv\Trans\vk_{\measure\points} + \vv\Trans \mK_{\points\points} \vv}}$} by the reproducing property and \smash{$I_\measure(g) = \langle g, k_\measure \rangle_{\rkhs(k)}$}.
  To verify that the upper bound is an equality, select $g_0 = (\kernel_\measure - \textstyle\sum_{i=1}^N v_i \kernel(\cdot, \vx_i)) / \norm[0]{\kernel_\measure - \textstyle\sum_{i=1}^N v_i \kernel(\cdot, \vx_i) }_{\rkhs(\kernel)}$.
  As this function has unit norm, $e_k(\vv, \points) \geq \lvert \int_\domain g_0(\vx) \dif \measure(\vx) - \sum_{i=1}^N v_i g_0(\vx_i) \rvert$.
  A short and straightforward computation reveals that the lower bound equals $\norm[0]{\kernel_\measure - \textstyle\sum_{i=1}^N v_i \kernel(\cdot, \vx_i) }_{\rkhs(\kernel)}$.
\end{proof}

See, for example, \citet[Sec.\@~10.2]{NovakWozniakowski2010} or \citet[Cor.\@~3.6]{Oettershagen2017} for generalisations of Lemma~\ref{lemma:wce}.
The worst-case error measures the largest possible integration error by a quadrature rule of the form $\sum_{i=1}^N v_i g(\vx_i)$ for functions $g \in \mathcal{H}(\kernel)$ normalised in the RKHS norm.
If the function $\f$ the quadrature rule is used to integrate is \emph{not} an element of $\mathcal{H}(\kernel)$, the worst-case error is still well-defined, as is evident from~\eqref{eq:wce-explicit}.
However, in this case the worst-case error needs to bear no relation to the integration error $\abs[0]{\int_\domain \f(\vx) \dif \measure(\vx) - \sum_{i=1}^N v_i \f(\vx_i)}$.
As squared worst-case error is quadratic in $\vv$, it is easy to see that, for fixed points, Equation~\eqref{eq:wce-explicit} is minimised by the weights $\vv = \vw = \mK_{\points\points}^{-1} \vk_{\measure\points}$, which are precisely the weights that define the integral mean and variance of conjugate Bayesian quadrature in~\eqref{eq:conj-bq-mean-var}.
Moreover, for these weights the worst-case error simplifies to $e_k(\vw, \points) = \Sigma_\dataset^{1/2} = \sqrt{\smash[b]{k_{\measure\measure} - \vw\Trans \vk_{\measure\points}}}$.
The posterior mean of a conjugate Bayesian quadrature method is thus given by weights which minimise the worst-case error in the RKHS of the kernel $\kernel$ and the posterior variance equals the corresponding minimal squared worst-case error.
These facts are summarised in the following proposition.

\begin{proposition}
  \label{prop:wce}
  Let $\fGP \sim \GP(\mean, \kernel)$.
  Then $I_\measure(\fGP) \mid \dataset = \Normal(\mu_\dataset, \Sigma_\dataset)$ for
  \begin{equation*}
    \mu_\dataset = m_\measure + \vw\Trans ( \vf_\points - \vm_\points) \: \text{ and } \: \Sigma_\dataset = k_{\measure\measure} - \vw\Trans \vk_{\measure\points} = \min_{ \vv \in \R^N } e_k(\vv, X)^2,
  \end{equation*}
  where
  \begin{equation} \label{eq:KQ-weights}
    \vw = \mK_{\points\points}^{-1} \vk_{\measure\points} = \argmin_{ \vv \in \R^N } e_k(\vv, X).
  \end{equation}
\end{proposition}

If the prior mean $\mean$ is zero and the integrand $\f$ is an element of $\rkhs(\kernel)$, we have the convenient error estimate
\begin{equation}
  \label{eq:rkhs-error-bound}
  \abs[0]{ I_\measure(\f) - \mu_\dataset } = \langle f , \kernel_\measure - \textstyle\sum_{i=1}^N w_i \kernel(\cdot, \vx_i) \rangle_{\rkhs(\kernel)} \leq \norm[0]{\f}_{\rkhs(k)} \Sigma_\dataset^{1/2},
\end{equation}
which follows from $\langle \f, \kernel_\measure \rangle_{\rkhs(\kernel)} = I_\measure(\f)$ (see the proof of Lemma~\ref{lemma:wce}, Section~\ref{sec:rkhs-appendix}, or~\citealp[Lem.\@~3.1]{Muandet2017}) and the Cauchy--Schwarz inequality.
This bound allows one to use the magnitude of $\Sigma_\dataset$ to assess the quality of $\mu_\dataset$ as a point estimator of the integral; its use is the reason behind the assumption $f \in \mathcal{H}(k)$ in several theorems in Section~\ref{sec:guarantees}.
If $f \notin \mathcal{H}(k)$, the bound~\eqref{eq:rkhs-error-bound} is not valid because the norm $\norm[0]{f}_{\mathcal{H}(k)}$ does not exist.

The construction of general worst-case optimal quadrature rules in RKHSs goes back at least to the work of Larkin in the 1970s~\citep{Larkin1970, Larkin1972, Larkin1974}, while the related average-case framework was developed by \citet{Suldin1959, Suldin1960}.\footnote{\citet[Sec.\@~2.2]{OatesSullivan2019} provide a historical account on the pioneering role played by Suldin and Larkin in the development of probabilistic numerical computation. We note that earlier constructions of optimal quadrature rules that do not utilise RKHSs may be found in \citet{Sard1949} and \citet{Nikolsky1950}.}
For early analyses, mostly in RKHSs of analytic functions, of these quadratures we refer to \citet{Richter1970, RichterDyn1971a, RichterDyn1971b}.
See also \citet[Sec.\@~2.3 in Ch.\@~5]{Traub1988}, \citet{Oettershagen2017,Ehler2019, Gavrilov1998, Gavrilov2007}.
A different tradition may be found in the quasi-Monte Carlo literature~\citep{DickKuoSloan2013}, where spaces of periodic functions are often considered.
The optimal weights in these spaces can be occasionally computed in closed form~\citep{Kaarnioja2025}.
Certain quasi-Monte carlo rules can be interpreted as Bayesian quadratures~\citep{Rathinavel2019, Jagadeeswaran2022, Sorokin2025}.

Whenever making use of the RKHS framework one should keep in mind that the samples from a Gaussian process are \emph{not} in the RKHS of its kernel.
Specifically, under fairly non-restrictive assumptions one can show that if $\fGP \sim \GP(\mean, \kernel)$, then $\Pr[ \fGP \in \mathcal{H}(\kernel)] = 0$; see~\citet{Driscoll1973} and \citet[Sec.\@~4]{Kanagawa2018}.
That is, the RKHS is too small to contain the samples.
One cannot therefore simply substitute the process $\fGP$ for the RKHS element $\f$ in results such as the error estimate~\eqref{eq:rkhs-error-bound}.

\subsection{Kernel interpolation}
\label{sec:kernel-interpolation}

A different tradition may be found in the literature on \emph{kernel interpolation} or, when the kernel is isotropic, \emph{radial basis function interpolation}~\citep{Wendland2005}.
Given a positive-definite kernel $\kernel$ and evaluations of $\f$ at a collection $\points = \{\vx_1, \ldots, \vx_N\}$ of pairwise distinct points, the \emph{kernel interpolant} $s_{\f, \points}$ is the unique function in the span of the \emph{kernel translates} $\{\kernel(\cdot, \vx_i)\}_{i=1}^N$ that interpolates $\f$ at $\points$, in the sense that $s_{\f, \points}(\vx_i) = \f(\vx_i)$ for every $i = 1, \ldots, N$.
From the requirement that the interpolant be in the span of the translates it is easy to deduce that $s_{\f, \points}(\vx) = \vk_\points(\vx)\Trans \mK_{\points\points}^{-1} \vf_\points$.
We see that the kernel interpolant coincides with the GP conditional mean~\eqref{eq:gp-posterior-mean} under a zero-mean prior when $\vy = \vf_\points$.
The kernel interpolant may be alternatively defined as the minimum-norm interpolant in the RKHS of $\kernel$:
\begin{equation} \label{eq:minimum-norm-interpolation}
  s_{\f, \points} = \argmin_{s \in \rkhs(\kernel)} \Set[]{ \norm[0]{s}_{\rkhs(\kernel)} }{ s(\vx_i) = \f(\vx_i) \text{ for } i = 1, \ldots, N}.
\end{equation}
See, for example, Theorem~3.5 in \citet{Kanagawa2018}.
Due to the minimum-norm property kernel interpolants are often called \emph{spline interpolants}~(\citealt{Traub1988}, Sec.\@~5.7; \citealt{Wahba1990}).
As the kernel interpolant $s_{\f, \points}$ depends on the value of $\f$ at a finite point set, it is well-defined whether $f$ is in $\mathcal{H}(\kernel)$ or not.
However, the useful inequality \smash{$\norm[0]{s_{\f, \points}}_{\mathcal{H}(\kernel)} \leq \norm[0]{\f}_{\mathcal{H}(\kernel)}$}, a consequence of the minimum-norm interpolation property, holds only if $\f \in \mathcal{H}(\kernel)$.

Integrating the kernel interpolant yields the integral approximation $I_\measure(\f) \approx I_\measure(s_{\f, \points}) = \vk_{\measure\points}\Trans \mK_{\points\points}^{-1} \vf_\points$, which is equal to the conjugate BQ integral mean~\eqref{eq:bq-mean} in the mean-zero case.
\citet{Bezhaev1991} appears to have been the first to suggest this approach.
More recent work on integrated kernel interpolants includes \citet{SommarivaWomersley2005, SommarivaVianello2006, SommarivaVianello2006b, Punzi2008, Siraj2012, Fuselier2014, ReegerFornberg2015, ReegerFornbergWatts2016, ReegerFornberg2018, SommarivaVianello2021, GlaubitzReeger2021}.
\citet[Sec.\@~6]{Minka2000} discusses the interpretation of spline-based quadrature as BQ.
The connection to kernel interpolation is important because many estimates for the integration error $\lvert I_\measure(\f) - \mu_\dataset \rvert$ of Bayesian quadrature are derived from $L^p$-error estimates for kernel interpolation~\citep{Kanagawa2020,Wynne2021}.

\subsection{Kernel quadrature}

Quadrature rules with the weights~\eqref{eq:KQ-weights} that minimise the worst-case error in an RKHS are often called \emph{kernel quadratures}~\citep{Briol2017, KarvonenSarkka2018, Karvonen2019d, Chen2025}.
However, this nomenclature is more commonly (and somewhat ambiguously) applied to any quadrature rule that seeks to achieve small worst-case error in an RKHS and does this by making use of the kernel or objects derived from it~\citep{Bach2017, Epperly2023}.\footnote{Kernel quadrature is occasionally used to refer \emph{any} quadrature rule designed so as to minimize its worst-case error in an RKHS~\citep[e.g.,][Sec.\@~3.4.3]{Hayakawa2023-thesis}. Such unconditional usage of the term is somewhat ambitious: Every Hilbert space of functions that one can construct is an RKHS because such a Hilbert space is an RKHS if and only if the point evaluation functionals $f \mapsto f(\vx)$ are bounded, the existence of unbounded functionals depends on the axiom of choice~\citep{Wright1973, AlpayMills2003}. For example, under this definition quasi-Monte Carlo integration would fall under the umbrella of kernel quadrature~\citep[see in particular][Sec.\@~3]{DickKuoSloan2013}.}
For example, \citet{Belhadji2019, Belhadji2020, Belhadji2020-thesis, Belhadji2021, Hayakawa2021, Hayakawa2022-hyper, Hayakawa2023, Hayakawa2023-thesis, Roustant2022} use Mercer expansions of $\kernel$ to construct quadrature rules.
See also~\citet{Tanaka2021,TsujiPokutta2022,Tsuji2022}.

\subsection{Bayesian quadrature in disguise} \label{sec:in-disguise}

Many existing and seemingly non-Bayesian integration methods are, or can be seen as, Bayesian quadrature methods.\footnote{Connections exist beyond mere numerical integration: \citet{Snell2025} connect Bayesian quadrature and conformal prediction.}
It is a well-known and established fact that the integral mean of conjugate Bayesian quadrature is equivalent to the classical trapezoidal rule if $k(x, x') = \min\{x, x'\}$ is the covariance function of the Brownian motion~\citep{DucJacquet1973,Diaconis1988}.
Recall that the trapezoidal rule is the integral approximation
\begin{equation*}
  \int_0^1 f(x) \dif x \approx \sum_{i=1}^N \frac{f(x_i) + f(x_{i-1})}{2} (x_i - x_{i-1}),
\end{equation*}
where $0 = x_0 < x_1 < \cdots < x_{N-1} < x_N = 1$.
The geometric interpretation is that the integral is approximated with $N$ trapezoids with base lengths $x_i - x_{i-1}$.
\citet[Sec.\@~3 in Ch.\@~II]{Ritter2000} is an accessible source for computations that constitute the proof of the following proposition.

\begin{proposition}
  Let $\measure$ be uniform on $\domain = [0, 1]$ and $\fGP \sim \GP(0, k)$ with $k(x, x') = \min\{x, x'\}$.
  If $0 < x_1 < \cdots < x_{N-1} < x_N  = 1$ and $f(0) = 0$, then 
  \begin{equation*}
    \mu_\dataset = \sum_{i=1}^N \frac{f(x_i) + f(x_{i-1})}{2} (x_i - x_{i-1}) \: \text{ and } \: \Sigma_\dataset = \frac{1}{12} \sum_{i=1}^n (x_i - x_{i-1})^3,
  \end{equation*}
  where $x_0 = 0$.
\end{proposition}

We refer to \citet{Naslidnyk2025} for more explicit computations of this kind for the Brownian motion covariance.
Higher-order versions of the trapezoidal rule and their connection to Bayesian quadrature with the $p$ times integrated Brownian motion covariance
\begin{equation*}
  k(x, y) = \int_0^1 \frac{(x - t)_+^m (y - t)_+^m}{(m!)^2} \dif t,
\end{equation*}
where $(x)_+ = \max\{0, x\}$, are discussed in \citet{Schober2014} and \citet[Sec.\@~5.5]{Karvonen2019a}.
See also \citet{Wahba1990}.

 Many classical quadrature rules can be cast as conjugate Bayesian quadrature if the covariance function is selected appropriately~\citep{Schober2014,KarvonenSarkka2017,Karvonen2018, Belhadji2021}. 
 See also \citet[Sec.\@~11.4]{pnbook22}.
 Recall that a quadrature rule defined by weights $\vv$ and nodes $X$ is the integral approximation $\vv\Trans\vf_X = \sum_{i=1}^N v_i f(\vx_i) \approx \int_\domain f(\vx) \dif \measure(\vx)$.
  If $m \equiv 0$, the integral mean $\mu_\dataset$ is a quadrature rule with weights $\vv = \vw = \mK_{\points\points}^{-1} \vk_{\measure\points}$.

\begin{proposition} \label{prop:classical}
  Let $D \subseteq \R$ and $\fGP \sim \GP(0, k)$ with covariance function $\kernel(x, y) = \sum_{p=0}^{N-1} a_p x^p y^p$, where $a_0, \ldots, a_{N-1} \neq 0$.
  Then $\mu_\dataset$ is the unique quadrature rule that integrates all polynomials of degree at most $N-1$ exactly. 
  That is, $\mu_\dataset = I_P(f)$ whenever $f$ is a polynomial of degree at most $N-1$.
  Moreover, $\Sigma_\dataset = 0$.
\end{proposition}
\begin{proof}
  The kernel Gram matrix can be written as 
  \begin{equation*}
    \mK_{XX} = \mV_X \mA \mV_X\Trans ,
  \end{equation*}
  where $\mA = \operatorname{diag}(a_0, \ldots, a_{N-1})$ is a diagonal matrix and $\mV_X$ the invertible $N \times N$ Vandermonde matrix with elements $(\mV_X)_{ij} = x_i^{j-1}$.
  Additionally, $\vk_{\measure\points} = \mV_X \mA \vp$ and $k_{\measure\measure} = \vp\Trans \mA \vp$, where $\vp = (\int_\domain x^i \dif \measure(x))_{i=0}^{N-1}$.
  By substituting these expressions in~\eqref{eq:bq-mean} and~\eqref{eq:bq-var} we obtain
  \begin{equation*}
    \mu_\dataset = (\mV_X^{-\intercal} \vp)\Trans \vf_X \quad \text{ and } \quad \Sigma_\dataset = 0.
  \end{equation*}
  By the definitions of $\mV_X$ and $\vp$, the weights $\vv = \mV_X^{-\intercal} \vp$ are the unique quadrature weights such that $\sum_{i=1}^N v_i x_i^p = \int_\domain x^p \dif \measure(x)$ for every $p \in \{0, \ldots, N-1\}$.
  It follows that $\mu_\dataset$ is the unique quadrature rule that integrates all polynomials of degree at most $N-1$ exactly.
\end{proof}

By selecting the nodes correctly one recovers from Proposition~\ref{prop:classical} every Gaussian quadrature rule (nodes are roots of an orthogonal polynomial) and other polynomial quadrature rules, such as the Newton--Cotes rule (nodes are equispaced on an interval).
Proposition~\ref{prop:classical} does not provide means to attach uncertainty to classical quadratures because the integral variance is always zero.
Another way to recover classical quadratures is by considering the length-scale limit $\ell \to \infty$~\citep{OHagan1991,Minka2000,Sarkka2016,Karvonen2020a}.
See also \citet[Sec.\@~5.4]{Karvonen2019a}; \citet[Thm.\@~2.6]{KarvonenOates2023}; and \citet{Barthelme2023}.
This too results in zero integral variance.

Estimation of the prior mean (more on this in Section~\ref{sec:prior-mean-estimation}) via a flat prior can be used to ``force'' the mean $\mu_\dataset$ to assume almost any desired form~\citep[Thm.\@~2.10]{Karvonen2018} while keeping the variance $\Sigma_\dataset$ positive.

\begin{proposition} \label{prop:bayes-sard}
  Let $\domain \subseteq \R^d$ and suppose that $\varphi_1, \ldots, \varphi_N$ are any functions such that the matrix $\mPhi_X = (\varphi_i(\vx_j))_{i,j=1}^N$ is invertible.
  If $\fGP \mid \vgamma \sim \GP(m_\vgamma, k)$ for $m_\vgamma(\vx) = \sum_{i=1}^q \gamma_i \varphi_i(\vx)$ and $\vgamma \sim \mathrm{N}(\mathbf{0}, \mSigma)$, then
  \begin{equation*}
    \mu_\dataset = \vv\Trans\vf_X \quad \text{ and } \quad \Sigma_\dataset = e_k(\vv, X)^2
  \end{equation*}
  at the limit $\mSigma^{-1} \to \mathbf{0}$, where $\vv = \mPhi_X^{-1} \va$ and $\va = (I_P(\varphi_i))_{i=1}^N$.
\end{proposition}

An expression for the worst-case error $e_k(\vv, X)$ is given in Lemma~\ref{lemma:wce}.
This construction has an air of artificiality as the variance depends on the covariance function but the mean does not.

\section{Fundamental theory for Bayesian quadrature} \label{sec:basic-theory}

With the results and tools introduced above it is possible to establish some fundamental theory for conjugate Bayesian quadrature that holds regardless of the properties the covariance kernel has.
The following proposition consists of the intuitively obvious fact that new data decrease the integral variance.

\begin{proposition} \label{prop:monotone-var}
  Let $\dataset$ and $\dataset'$ be datasets such that $\dataset \subseteq \dataset'$.
  Then the integral variance in~\eqref{eq:bq-var} satisfies $\Sigma_{\dataset'} \leq \Sigma_{\dataset}$.
\end{proposition}
\begin{proof}
  Suppose that $\dataset$ uses nodes $\points = \{\vx_1, \ldots, \vx_N\}$ and $\dataset'$ uses nodes $\points' = \{\vx_1, \ldots, \vx_{N'} \}$ for $N' \geq N$. 
  Let $\vw \in \R^N$ and $\vw' \in \R^{N'}$ be the Bayesian quadrature weights based on $X$ and $X'$, respectively. 
  Additionally, let $\vw_0 = (\vw, 0, \ldots, 0) \in \R^{N'}$.
  Then Proposition~\ref{prop:wce} yields
  \begin{equation*}
    \begin{split}
    \Sigma_{\dataset'} = e_k(\vw', \points')^2 = \min_{\vv \in \R^{N'}} e_k(\vv, \points')^2 \leq e_k(\vw_0, \points')^2 &= e_k(\vw, \points)^2 \\
    &= \Sigma_{\dataset},
    \end{split}
  \end{equation*}
  which is the claim.
\end{proof}

Note that both $\Sigma_\dataset$ and $\Sigma_{\dataset'}$ in Proposition~\ref{prop:monotone-var} are based on the same covariance function.
In practice it can happen that the variance increases as more data are obtained if the covariance parameters are being learned from the data.
For example, new data could indicate that the integrand is much rougher than previously thought.
We discuss parameter estimation in Section~\ref{sec:param-estimation}.
The next proposition describes how the covariance function affects the variance via its RKHS.

\begin{proposition}
  Let $\Sigma_{\dataset,k}$ be the integral variance in~\eqref{eq:bq-var} under the prior $\fGP \sim \GP(m, k)$.
  If $k$ and $r$ are covariance kernels such that their RKHSs satisfy $\rkhs(k) \subseteq \rkhs(r)$, then there is a constant $c > 0$ such that 
  \begin{equation*}
    \Sigma_{\dataset, k} \leq c \cdot \Sigma_{\dataset, r} \quad \text{ for every dataset $\dataset$.}
  \end{equation*}
\end{proposition}
\begin{proof}
  The inclusion $\rkhs(k) \subseteq \rkhs(r)$ implies that there is $c > 0$ such that $\lVert g \rVert_{\rkhs(r)} \leq \sqrt{c} \cdot \lVert g \rVert_{\rkhs(k)}$ for all $g \in \rkhs(k)$; see~\citet[Thm.\@~5.1]{Paulsen2016}.
  It follows that
  \begin{equation*}
    \Set{g \in \rkhs(k)}{ \lVert g \rVert_{\rkhs(k)} \leq 1 } \subseteq \Set{g \in \rkhs(r)}{ \lVert g \rVert_{\rkhs(r)} \leq \sqrt{c} }.
  \end{equation*}
  By the definition of the worst-case error and Proposition~\ref{prop:wce},
  \begin{equation*}
    \begin{split}
    \sqrt{\smash[b]{\Sigma_{\dataset, k}}} &= \min_{\vv \in \R^n} \sup_{ \norm[0]{g}_{\rkhs(k)} \leq 1} \, \Big \lvert \int_\domain g(\vx) \dif \measure(\vx) - \sum_{i=1}^N v_i g(\vx_i) \Big \rvert \\
    &\leq \min_{\vv \in \R^n} \sup_{ \norm[0]{g}_{\rkhs(r)} \leq \sqrt{c}} \, \Big \lvert \int_\domain g(\vx) \dif \measure(\vx) - \sum_{i=1}^N v_i g(\vx_i) \Big \rvert \\
    &= \sqrt{c} \cdot \min_{\vv \in \R^n} \sup_{ \norm[0]{g}_{\rkhs(r)} \leq 1} \, \Big \lvert \int_\domain g(\vx) \dif \measure(\vx) - \sum_{i=1}^N v_i g(\vx_i) \Big \rvert \\
    &= \sqrt{\smash[b]{ c \cdot \Sigma_{\dataset, r}}}, 
    \end{split}
  \end{equation*}
  which is the claim.
\end{proof}

Inclusions between RKHSs of standard covariance functions are well-understood~\citep{Kanagawa2018}.
For example, if $k_\nu$ is a Matérn covariance in~\eqref{eq:matern-kernel} with regularity $\nu$ and $k_\mathrm{se}$ the square exponential in~\eqref{eq:se-kernel}, then
\begin{equation*}
  \rkhs(k_{\nu_1}) \subseteq \rkhs(k_{\nu_2}) \subsetneq \rkhs(k_\mathrm{se})
\end{equation*}
if $\nu_1 \leq \nu_2$ and the domain $\domain$ has non-empty interior.
The kernels in this chain of inclusions are allowed to have different length-scales.
The final proposition of this section says that conjugate Bayesian quadrature integrates the kernel translates $k(\cdot, \vx_i)$ exactly.

\begin{proposition}
  Let $\fGP \sim \GP(0, k)$.
  If $f = k(\cdot, \vx_i)$ for $i \in \{ 1, \ldots, N\}$, then $\mu_\dataset = I_\measure(f)$.
\end{proposition}
\begin{proof}
  Let $f = k(\cdot, \vx_i)$ for $i \in \{1, \ldots, N\}$.
  When $m \equiv 0$, the integral mean is $\mu_{\dataset} = \sum_{i=1}^N \vw\Trans \vf_\points$ for $\vw \in \R^N$ that satisfies $\mK_{\points\points} \vw = \vk_{\measure\points}$.
  The $i$th row of this linear system is $\sum_{j=1}^n k(\vx_i, \vx_j) w_j = \int_\domain k(\vx, \vx_i) \dif \measure(\vx)$, which we can write as $\sum_{j=1}^n f(\vx_j) w_j = \int_\domain f(\vx) \dif \measure(\vx)$.
  Equivalently, $\mu_\dataset = \vw\Trans \vf_\points = I_\measure(f)$.
\end{proof}

\section{Generalisations}
\label{sec:generalisations}

We have elected to work in the setting described above to retain a degree of concreteness and to avoid introduction of overbearing notation.
However, a number of conceptually straightforward (though not necessarily so in practice) generalisations are possible and appear often in the literature.
These include the following:
\begin{itemize}
\item \emph{Domain.}
  There is nothing in Bayesian quadrature or Gaussian processes that ties them to subsets of $\R^d$.
  In general, one can let $D$ be any set and $P$ any measure on $D$.
  This generality of Bayesian quadrature has been recognised since the very beginning~\citep{Larkin1972, OHagan1991}.
  A recurring example in the literature is the computation of \emph{global illumination integrals} on the unit sphere $\mathbb{S}^2 = \Set{ \vx \in \R^3}{ \norm[0]{\vx} = 1}$, to which Bayesian quadrature has been applied in \citet{Brouillat2009, Marques2013, Marques2015, Xi2018, Briol2019, Karvonen2019b}.
  Other domains may be found in \citet{Froehlichl21} and \citet{Kang2025}.
  Integration on a sphere has also been considered in the literature on radial basis function integration~\citep{SommarivaWomersley2005, Fuselier2014, ReegerFornberg2015} whose connection to Bayesian quadrature we described in Section~\ref{sec:kernel-interpolation}.
\item \emph{Information.} 
It is straightforward to incorporate any linear information (i.e., observations that have been obtained via application of a linear functional) in Gaussian process regression~\citep{Sarkka2011} and, consequently, Bayesian quadrature. We refer to \citet[Sec.\@~3]{Oettershagen2017} for a rigorous treatment of the topic in the context of approximation theory, the connection to Gaussian processes of which we reviewed in Section~\ref{sec:conn-appr-theory}.
Derivative observations, which occur frequently in Gaussian process literature~\citep[e.g.,][]{Solak2002}, have been used in Bayesian quadrature by \citet[Sec.\@~2.5]{OHagan1992} and \citet{PruherSarkka2016}, while \citet{Weiland2026} use integral observations over sub-domains.

\item \emph{Noise.} 
Incorporating observations that are perturbed by Gaussian noise, such as when $\f$ is the output of a stochastic simulation, to Bayesian quadrature is straightforward. 
In the typical Gaussian conjugate setting this simply amounts to incrementing the diagonal of the covariance matrix with the noise variances (see the equations for the regularised mean and variance in \Cref{sec:num-ill-cond}); this is supported by most open source implementations of Bayesian quadrature.
Even though the setting of noisy observations is dominant in the Gaussian process regression literature, Bayesian quadrature with noisy observations appears to have received little explicit attention.
 The only substantial analysis the noisy case that we are aware of is by \citet{CaiLamScarlett2023}.
\citet{Ma2014} also incorporate noise in their method.
\item \emph{Prior.} 
  In our definition of Bayesian quadrature the integrand needs to be modelled as a potentially non-linearly transformed Gaussian process $\fGP = \transformation \circ \gGP$.
  As discussed in Section~\ref{sec:bpni}, one is in principle free to use any suitable stochastic process $\fGP$ to model the integrand.
  Due to computational tractability issues, priors other than transformed Gaussian processes are rarely used.
  Student's $t$-processes, which differ little from Gaussian processes in practice, have been used in numerical integration by \citet[Sec.\@~2.2]{OHagan1991} and \citet{Pruher2017}.
  What appears to be the most interesting example of non-Gaussian priors in numerical integration is due to \citet{Zhu2020}, who use Bayesian additive regression trees.
\end{itemize}

\chapter{Taxonomy of Bayesian quadrature}
\label{sec:taxonomy}

Definition 2.1 constructed Bayesian quadrature as a method for computing or approximating the distribution
$I_\measure(\fGP)\mid\dataset$. This abstract formulation admits several concrete realizations, each corresponding to a particular Bayesian quadrature method or algorithm. In recent years, the literature has proposed a wide variety of approaches consistent with Definition 2.1, differing mainly in their methodological choices and practical implementation. This chapter offers a concise, high-level overview of these alternatives. It may be read as a taxonomy of Bayesian quadrature methods or used as a reference for terminology. The taxonomy is illustrated visually in Figure~\ref{fig:tax-graph}, and Table~\ref{tab:references} classifies selected algorithms from the literature.

\begin{algorithm}[h]
  \caption{Bayesian quadrature}
  \label{alg:bq}
  \begin{algorithmic}[1]
    \footnotesize
    \Require
    \Statex \textbullet~The integration task:
    \Statex \quad\textbullet~The integrand $f$, given by a function handle.
    \Statex \quad\textbullet~The integration measure $\measure$, usually given by the density $p$ and domain $\domain$.
    \Statex \textbullet~The model $\fGP$:
    \Statex \quad\textbullet~The GP $\gGP$, given by function handles of $k$ and $m$.
    \Statex \quad\textbullet~The transformation $\varphi$ and its inverse given by function handles.
    \Statex \textbullet~The budget given by the dataset size $N$.
    \State $\points \gets \textsc{sample}(N, \fGP, \measure)$
    \Comment{Sample $N$ nodes.}
    \State $\vf_\points \gets f(\vx_i)$ for $\vx_i$ in $\points$ 
    \Comment{Evaluate integrand at nodes.}
    \State $\dataset \gets [\points, \vf_{\points}]$
    \Comment{The dataset (for notational convenience).}
    \State $\algres\gets \textsc{infer}(\fGP, \measure, \dataset)$
    \Comment{The result $\algres$ is a representation of $I_\measure(\fGP) \mid \dataset$.}
    \State \Return $\algres$
  \end{algorithmic}
\end{algorithm}

To motivate the taxonomy, we inspect a high-level pseudocode for Bayesian quadrature shown in Algorithm~\ref{alg:bq}.
The pseudocode reveals three core components: a model for $\fGP$ specified by $\gGP$ and $\varphi$ (see~\eqref{eq:fGP-transformation}); a sampling procedure (line~1) for selecting the nodes $\points$ that constitute the dataset; and an inference procedure (line~4) that computes the return object $\algres$ representing the distribution $I_\measure(\fGP) \mid \dataset$.\footnote{As $I_\measure(\fGP) \mid \dataset$ is a distribution, any datatype that can represent a distribution is valid. For Gaussians, a common way is to return the sufficient statistics: mean $\Exp[ I_\measure(\fGP) \mid \dataset ]$ and variance $\Var[ I_\measure(\fGP) \mid \dataset ]$.}

This naturally leads to three main axes along which Bayesian quadrature methods can be organized:
\begin{itemize}
\item a \emph{model axis}, which characterizes the Gaussian process prior $\gGP$ and the transformation $\transformation$ (\Cref{sec:tax-models});
\item an \emph{inference axis}, which specifies how the result $\algres$ is computed (\Cref{sec:tax-inf});
\item a \emph{sampling axis}, which describes how the nodes $\points$ are selected (\Cref{sec:tax-sampling}).
\end{itemize}
Along each axis, we introduce multi-label subcategories intended to reflect the current landscape of methods and algorithms. Here, “multi-label” indicates that categories are not necessarily mutually exclusive (when they are, this will be stated explicitly); consequently, a method may be assigned multiple labels per axis. An illustrative example is provided in \Cref{sec:tax-example}.

Not all Bayesian quadrature methods fit neatly into the pseudocode outlined above.
This primarily concerns approaches that actively manage data acquisition and computational budgets (\Cref{sec:tax-sampling}). Inevitably, any taxonomy involves a degree of subjectivity; ours is guided by practical considerations motivated by applications in applied science.

We occasionally use the big-$\gO$ notation when referring to computational complexity or memory requirements.
Accessible implementations of Bayesian quadrature algorithms are unfortunately still scarce and not on par with the richness of the literature.
Some are available in the open source libraries \emukit~(for Python; \citealp{Paleyes19, Paleyes2023}), \gail~(for \matlab; \citealp{GAIL2020}), and \probnum~(for Python; \citealp{Wenger21}).

\section{Model}
\label{sec:tax-models}

The first axis of the taxonomy is concerned with the prior model $\fGP = \transformation \circ \gGP$, in particular the choice of the function $\transformation$.
We differentiate between two mutually distinct classes: ``Conjugate'' and ``non-conjugate'' that in essence categorize the posterior $I_{\measure}(\fGP) \mid \dataset$ as Gaussian or non-Gaussian.

\begin{itemize}
\item {\it Conjugate: } Conjugate models are arguably the most well-known and the most straightforward Bayesian quadrature models. They are described in Section~\ref{sec:conj-bq}, where $\transformation$ is affine and the Gaussian process $\gGP$ implies a Gaussian process $\fGP$ and a Gaussian posterior $I_{\measure}(\fGP) \mid \dataset$. Conjugate Bayesian quadrature models inherit the conveniences of GP regression models, and $I_{\measure}(\fGP) \mid \dataset$ can be computed exactly under certain well-understood circumstances (Section~\ref{sec:tax-inf}). For these reasons, conjugate Bayesian quadrature is sometimes also called \emph{standard} Bayesian quadrature \citep{Karvonen2019a, Naslidnyk21} or \emph{vanilla} Bayesian quadrature \citep{Paleyes19, Gessner20, Adachi2022}. As $\fGP$ and $I_{\measure}(\fGP) \mid \dataset$ are confined to Gaussians, conjugate models are limited in their expressiveness.
\item {\it Non-conjugate: } Non-conjugate models are often motivated by the desire to encode physically meaningful, non-negative integrands $f$ and integral values $I_{\measure}(\fGP) \mid \dataset$ which Gaussians cannot provide.
  Steering away from Gaussianity is possible with an appropriate function $\transformation$. Non-conjugate models are described in Section~\ref{sec:non-conj-bq} where $\transformation$ is not affine, $\fGP$ is a random process other than a Gaussian process and the implied posterior $I_{\measure}(\fGP) \mid \dataset$ is non-Gaussian. For example, for $\transformation(x)=\alpha + x^2$ as in \citet{Gunter2014}, $\fGP$ is a process whose marginals are noncentral $\chi^2$-distributions of degree 1. Inference, however, is usually less straightforward and may require approximations (\Cref{sec:tax-inf}). Due to the presence of $\transformation$, non-conjugate Bayesian quadrature is sometimes also referred to as \emph{warped}\footnote{The name is likely inspired by \citet{Snelson2003} who learn the transformation $\transformation$ from data in the context of GP regression and call their method ``warped Gaussian processes''. In the context of Bayesian quadrature, the transformation $\transformation$ is usually fixed, but could in principle also be learned in a similar way.} Bayesian quadrature \citep{Gunter2014, Paleyes19}. While the space of possible $\transformation$ is rich in principle, only a limited number have been explored in the context of BQ.
\end{itemize}

Further, a kernel $k$ and mean function $m$ identify the Gaussian process $\gGP$. As mentioned in \Cref{sec:gps}, common choices for $k$ are the square exponential kernel and kernels of the Matérn family, either isotropic or of product form.
Any model may additionally be characterized by hyper-parameters (such as kernel parameters) that are either given or set by an appropriate method (\Cref{sec:param-estimation}).

\section{Inference}
\label{sec:tax-inf}

Once the model $\fGP$ has been defined and a dataset $\dataset$ has been obtained (node sampling forms the third axis of the taxonomy and is the content of \Cref{sec:tax-sampling}), one can attempt to compute the solution $I_\measure(\fGP) \mid \dataset$. The second axis of the taxonomy---\emph{inference}---distinguishes the quality of the solution as either ``exact'' or ``approximate''. A third label ``scalable'' emphasizes computational complexity of the inference algorithm with respect to the dataset size, $N$.

\begin{itemize}
\item {\it Exact: } The inference method is called exact if $I_\measure(\fGP) \mid \dataset$ can be computed without approximation error ($\algres$ in Algorithm~\ref{alg:bq} is an exact representation). This is possible if $I_\measure(\fGP) \mid \dataset$ has closed form and admits numerically exact computations.
 Exact inference can for example be realized under the conjugate BQ model if the measure $\measure$ and kernel $\kernel$ pair yields closed form solutions of Eqs.~\eqref{eq:bq-mean} and~\eqref{eq:bq-var}; see Section~\ref{sec:kern-mean-var}.
Since the expressions have closed form, exact inference in Bayesian quadrature is also called \emph{analytic}. \Cref{sec:num-ill-cond} discusses some common practical considerations concerning numerical precision.
\item {\it Approximate: } The inference method is called approximate if the returned distribution ($\algres$ in Algorithm~\ref{alg:bq}) differs from $I_\measure(\fGP) \mid \dataset$ due to methodological choices.\footnote{Bayesian quadrature is related to the field of \emph{probabilistic numerics} \citep{Hennig2015, Cockayne2019, pnbook22} which frames integration, among other numerical algorithms, as inference. If we were to follow \citet[Defs.\@~2.2 \& 2.5]{Cockayne2019} who characterize a probabilistic numerical method as ``Bayesian'' or ``non-Bayesian'', \emph{Bayesian quadrature} as in Definition~\ref{def:bq} admitting \emph{exact} or \emph{approximate} inference according to our taxonomy would naturally be called \emph{Bayesian Bayesian quadrature} or \emph{non-Bayesian Bayesian quadrature}, respectively.}
  Such approximations to $I_\measure(\fGP) \mid \dataset$ are necessary if (i)~the model $\fGP\mid\dataset$ does not admit analytic integration, or (ii)~if evaluating the analytic expressions to numerical precision is prohibitively expensive.
    In practice, inference methods for non-conjugate BQ models often entail approximations due to the former reason because the process $\fGP$ can be arbitrarily complicated. Some recent approaches use Gaussian approximations to $\fGP$ via linearisation or moment matching which are integrated instead of $\fGP$.
    Conjugate Bayesian quadrature, where $\fGP$ is a Gaussian process, may also require approximations when the kernel $\kernel$ and integration measure $\measure$ pair has no closed form kernel mean embedding (Section~\ref{sec:kern-mean-var}). Approximations of kernel embeddings have been scarcely studied in the context of BQ, but in principle, any means of approximation is admissible. Another reason arises when the closed form solution is too expensive to compute for large dataset sizes~$N$ because the computational complexity scales unfavorably. Ways to remedy this are discussed next.
\end{itemize}

Algorithm~\ref{alg:conj-bq} shows a pseudocode of the exact inference method for a conjugate Bayesian quadrature model.
From the pseudocode it is apparent that exact inference for the conjugate model is feasible if the kernel mean embedding $\kernel_{\measure}$, the initial variance $\kernel_{\measure\measure}$ and the prior mean integral $m_{\measure}$ are analytic. Algorithm~\ref{alg:conj-bq} first computes $\kernel_{\measure\measure}$ and $m_{\measure}$ (lines~2--3). Then, it evaluates the kernel $\kernel$, mean function $m$ and kernel mean embedding $\kernel_{\measure}$ at the given nodes (lines~4--5) and constructs the kernel Gram matrix $\mK_{\points\points}$ (line~6). The bulk of the computation is performed in line~8 where the algorithm solves a linear system of size $N$ in order to retrieve the weights $\vw = \mK_{\points\points}^{-1} \vk_{\measure\points}$. Lines~9--10 then perform simple vector multiplications and summations that yield the final solution.

\begin{algorithm}[H]
  \caption{Exact inference method for the conjugate BQ model}
  \label{alg:conj-bq}
  \begin{algorithmic}[1]    
    \footnotesize
    \Require 
    \Statex \textbullet~Function handles to the kernel $k$ and mean function $m$.
    \Statex \textbullet~Function handle to the kernel mean embedding $k_{\measure}$.
    \Statex \textbullet~Function handle to (or precomputed) initial variance $k_{\measure\measure}$ and prior mean integral $m_{\measure}$.
    \Statex \textbullet~The dataset $\dataset$ with nodes $\points$ and observations $\vf_\points$.
    \Function{infer}{$k$, $m$, $k_{\measure}$, $m_{\measure}$, $k_{\measure\measure}$, $\points$, $\vf_{\points}$}
    \State $m_{\measure} \gets m_{\measure}()$ 
    \Comment{Evaluate or assign prior mean integral $I_{\measure}(m)$.}
    \State $k_{\measure\measure} \gets k_{\measure\measure}()$ 
    \Comment{Evaluate or assign initial variance as in Eq.~\eqref{eq:kernel-var}.}
    \State $\vk_{\measure\points} \gets \kernel_{\measure}(\vx_i)$ for $\vx_i$ in $\points$ 
    \Comment{Evaluate kernel mean embedding at nodes.}
    \State $\vm_{\points} \gets m(\vx_i)$ for $\vx_i$ in $\points$ 
    \Comment{Evaluate prior mean at nodes.}
    \State $\mK_{\points\points} \gets \kernel(\vx_i, \vx_{j})$ for $\vx_i$ in $\points$ for $\vx_{j}$ in $\points$
    \Comment{Evaluate kernel for all node combinations.}
    \State
    \State $\vw \gets \textsc{SolveLinear}(\mK_{\points\points}, \vk_{\measure\points}) $
    \Comment{Solve linear system, $\gO(N^3)$.}
    \State $\mu_{\dataset}\gets m_\measure + \vw\Trans ( \vf_\points - \vm_\points)$
    \Comment{Compute mean of $I_{\measure}(\fGP)\mid\dataset$, Eq.~\eqref{eq:conj-bq-mean-var}.}
    \State $\Sigma_{\dataset}\gets \kernel_{\measure\measure} - \vw\Trans \vk_{\measure\points}$
    \Comment{Compute variance of $I_{\measure}(\fGP)\mid\dataset$, Eq.~\eqref{eq:conj-bq-mean-var}.}
    \State $\algres \gets (\mu_{\dataset}, \Sigma_{\dataset})$
    \Comment{$I_{\measure}(\fGP)\mid\dataset$ is Gaussian and described by $\mu_{\dataset}$ and $\Sigma_{\dataset}$.}
    \State \Return $\algres$
    \EndFunction
  \end{algorithmic}
\end{algorithm}

Algorithm~\ref{alg:conj-bq} shows arguably the most well-known BQ setting which inherits its quadratic memory $\gO(N^2)$ (line~6) and cubic compute cost $\gO(N^3)$ (line~8) from standard Gaussian process regression. Thus, the final label of the inference axis describes the scalability of the BQ inference method with respect to $N$ compared to the standard cubic cost. This label is not mutually exclusive to the previous two; in fact, both an ``exact'' and ``approximate'' inference method can be either scalable or not.

\begin{itemize}
\item {\it Scalable: } The Bayesian quadrature inference method is called scalable if its computational complexity is less than cubic in the number of datapoints $N$ and thus scales favorably compared to standard Gaussian process regression.
  Interestingly, so far scalable methods have been proposed primarily for \emph{exact} inference under the conjugate BQ model. For instance, \citet{KarvonenSarkka2018} exploit structure in the posterior mean equation for certain kernels, measures and a custom node sampling scheme; their algorithm scales sub-cubically.
  \citet{Rathinavel2019} use fast Fourier transform to obtain computational cost $\gO(N \log N)$.
  Scalable methods for \emph{approximate} inference under the conjugate BQ model have not been explored systematically.\footnote{There is a vast literature on scalable algorithms, such as random Fourier features, for approximate Gaussian process inference~\citep{QuinnoneroCandela2005, Bui2017, LiuOngShen2020}. Such algorithms have seen little use in the context of Bayesian quadrature.}
  \citet{Akella2021} use a scalable approximate BQ method based on structured kernel interpolation~\citep{WilsonNickisch2015} to estimate policy gradients and \citet{LiAcerbi2023} use sparse GPs within variational BQ.
  We are not aware of scalable methods for non-conjugate BQ models.
Scalable Bayesian quadrature is sometimes also referred to as \emph{fast} Bayesian quadrature \citep{Rathinavel2019, Jagadeeswaran2022}. This is not to be confused with the use of the term \emph{fast} BQ as in \citet{Gunter2014} or \citet{Adachi2022} where the term highlights sample efficiency or end-to-end wall-clock time performance of the BQ method and data retrieval which our taxonomy does not address.
Here we also mention a recent quantum algorithm for Bayesian quadrature by \citet{Florez2025}.
\end{itemize}

\section{Sampling (node  selection)}
\label{sec:tax-sampling}

The first two axes described the BQ model and inference method when a dataset $\dataset$ is provided. The third and last axis of the taxonomy focuses on the dataset $\dataset$, particularly emphasizing the \emph{selection of the nodes} $\points$ (as the corresponding function values $\vf_{\points}$ can be retrieved by evaluating $f$ at $\points$).
We start by observing that traditionally, to solve an intractable integral, practitioners may refer to classic quadrature rules where nodes
are predefined and cannot be chosen freely, or Monte Carlo estimators, where nodes must be random.
Bayesian quadrature is positioned in between these methods as it generally can
use any nodes $\vx_1,\dots,\vx_N$ as long they are pairwise distinct.
This fortunate property of BQ leads to a rich variety of practical sampling options wherefore this section is divided into several subsections.

\subsection{Mode of sampling}
\label{sec:tax-sampling-mode}

We begin with three high-level categories that are fundamental to the mode in which the nodes are obtained: ``Deterministic'', ``random'', and ``fixed''. 
The former two are mutually distinct, while the third is simply undetermined.
\begin{itemize}
  \item {\it Deterministic: } Deterministic node selection schemes are those schemes that \emph{ideally} do not use random or pseudo-random numbers, except potentially for an initialization step such as random selection of the generating vector for lattice nodes~\citep[Sec.\@~4]{Kaarnioja2025, Rathinavel2019}. In practice, deterministic schemes often comprise a degree of randomness, for example when the nodes are selected by maximizing an acquisition function and the optimizer makes use of randomness.\footnote{We are not aware of explicit use of stochastic optimizers in BQ, although stochastic optimizers have been used in the related field of Bayesian optimization~\citep[for an overview, see][]{botorch2020}. However, even deterministic optimizers usually contain a degree of randomness. For example, the initial point of the iteration needs to be chosen which may affect which maximizer is found and hence which node is sampled.}
\item {\it Random: } Random sampling schemes are those that explicitly use random or pseudo-random numbers to generate the nodes. In contrast to ``deterministic'' sampling as described above, ``random'' sampling has a degree of randomness intrinsic to its design and even under ideal practical conditions.
  An early use of random node sampling for BQ was proposed as part of the Bayesian Monte Carlo method by \citet{RasmussenGhahramani2002} who sampled from the integration measure $P$. In the meantime, also other distributions have been used~\citep{Bach2017, Briol2017}.
\item {\it Fixed / unknown:} Fixed or unknown sampling describes a set of nodes whose generating process is either unknown or beyond the control of the Bayesian quadrature practitioner. A fixed set of nodes could for example be provided by a (black-box) service.
\end{itemize}
The above three categories describe the degree of control that the practitioner has (or rather needs to have) over the nodes, strictly within the context of Bayesian quadrature: When the nodes are fixed, one has \emph{no control}; when they are random, one has only \emph{distributional control}; and when they are deterministic, one has \emph{full control} over the nodes (at least in principle).\footnote{The mode of sampling at this point simply describes characteristics of the node selection option. Unlike data-dependent estimators that may inherit the randomness property from the mode of sampling, Bayesian models (the BQ model is a Bayesian model) do not. However, estimators derived from the BQ model such as the mean or MAP estimator may be analyzed in this way.}
The remaining categories provide more granularity to node selection.

\subsection{Contextual sampling}
\label{sec:tax-sampling-model-dependendent}

We now inspect line~1 of Algorithm~\ref{alg:bq} more closely and observe that the model $\fGP$ is an optional argument to the \textsc{sample} method which means that some sampling procedures may depend on their context. Thus the first granular label of the sampling axis is called ``model-dependent''.

\begin{itemize}
\item \emph{Model-dependent: } It is easy to fathom sampling schemes that are independent of the model $\fGP$, such as grids or draws from the integration measure $\measure$. If the sampling method does not request the model $\fGP$ (the \textsc{sample} method in line~1 of Algorithm~\ref{alg:bq} does not require $\fGP$ as argument) we call it ``model-independent'' and if it does, ``model-dependent''. As of now, this definition is purely algorithmic in the sense that the nodes $\points$ could be created either before or after $\fGP$ is chosen. As we will see soon, the distinction is somewhat more intricate than it may initially appear since it disentangles (more so than other labels) the method from the algorithm. This is best explained with an example: It is a well-known fact \citep[e.g.,][Sec.\@~11.1.3]{pnbook22} that for $\fGP$ a Brownian motion and $\measure$ the uniform measure on the interval $[0, 1]$, equidistant nodes are the maximally-informative design in the sense that $\Var[ I_\measure(\fGP) \mid \dataset ]$ is minimized. The creation of equidistant nodes on $[0, 1]$ of course does not require $\fGP$ explicitly. However, since the practitioner promises maximally informative nodes, knowledge about $\fGP$ (i.e., it being a Brownian motion) is required before the nodes are created. In other words, if the model $\fGP$ changes, the sampling method will need to change as well.
  If the practitioner were to use the equidistant nodes no matter $\fGP$ (no ``link'' between $\fGP$ and the sampling method is claimed), the BQ method would be model-independent. To conclude the example, maximally informative nodes for processes $\fGP$ other than a Brownian motion are not generally analytic and their computation may require $\fGP$ explicitly in which case it being model-dependent is apparent.
  Model-dependent sampling often implies a notion of informativity of the nodes under the model.
A recent example of model-dependent node selection is given in \citet{Adachi2022} who construct a sampling distribution that depends on $\fGP$ as well as $\measure$.\footnote{See \citet{Adachi2022b} for an application of this method to the selection of lithium-ion battery models.} Further, active sampling (Section~\ref{sec:tax-sampling-active}) is model-dependent by definition as the acquisition function depends on $\fGP$.
It is highly unclear when model-dependent or independent sampling for Bayesian quadrature should be preferred and it may be contingent on the application at hand.
Although model-independent node selection has no knowledge about $\fGP$ it often performs well, is generally fast as nodes can be pre-computed, and lends itself well to theoretical examination. Model-dependent node selection has the benefit of being tailored to the integration problem at hand.
\end{itemize}

Model-dependent sampling, as mentioned, usually implies that the nodes are chosen due to their informativity under the model $\fGP$. Another type of dependence may occur if the sampler is interlaced with the model additionally or exclusively via the \emph{inference method}. Such nodes may or may not be chosen due to their informative content, but most likely they are chosen for making the inference method ``work''.
An example are inference methods that are only scalable (have computational cost less than cubic in $N$) if the nodes follow a specific sampling strategy. 
The inference method can even be connected to the sampler in such a way that choosing any other sampler would not lead to an executable algorithm.
To highlight the association we will call these sampling schemes ``inference-dependent''.
This is a good point to reiterate that the taxonomy labels should be attached to specific instances of BQ methods that comprise a model, an inference and a sampling method. Hence labels need to be interpreted with respect to the BQ method as a whole and in the context of the practitioner's aim. A context-free taxonomy with merit is almost impossible as BQ methods only allow to plug-and-play under restrictions.\footnote{The situation is slightly different in the related Bayesian optimization (BO) method where plug-and-play is often feasible. This is because the inference method in BO has a less prominent role and influences node sampling but not the return object, which is  $\argmin_{ \vx_i \in \points} \f(\vx_i)$. For the same reason, BO methods always require model-dependent sampling in order to be meaningful.}
\begin{itemize}
\item \emph{Inference-dependent: } Some BQ methods require specific node sampling schemes that are motivated by the unique properties of their inference method.
We call those samplers ``inference-dependent''.
As mentioned, examples are the sampling schemes of the exact, scalable BQ methods in \citet{KarvonenSarkka2018} and \citet{Rathinavel2019}.
The concept is similar in spirit (albeit not as strict) to how nodes in classical Gaussian quadrature rules are derived.
\end{itemize}

\subsection{Sequential sampling}
\label{sec:tax-sampling-sequential}

So far, we have worked with Algorithm~\ref{alg:bq} where the nodes $\points$ are sampled (line~1) as one large batch of size $N$. We are now ready to study ``sequential'' sampling where nodes are acquired repeatedly in a loop in smaller sized batches (and often even only one at a time). Algorithm~\ref{alg:bq-seq} shows the general layout of a sequential method.
\begin{algorithm}[H]
  \caption{Sequential Bayesian quadrature}
  \label{alg:bq-seq}
  \begin{algorithmic}[1]
    \footnotesize
    \Require
    \Statex \textbullet~The integration task given by $f$ and $\measure$.
    \Statex \textbullet~The model $\fGP$ given by $\gGP$ and $\transformation$.
    \Statex \textbullet~The stopping condition that may represent a budget.
    \Statex \textbullet~The batch size $n$.
    \Statex \textbullet~Optional, an initial design $\pointsi$.
    \State $\vf_{\pointsi} \gets f(\vx_i)$ for $\vx_i$ in $\pointsi$ 
    \Comment{Evaluate integrand at initial design.}
    \State $\dataset \gets [\pointsi, \vf_{\pointsi}]$
    \Comment{The dataset is initialized with the initial design.}
    \While {not $\textsc{stopping\_condition}(\dataset, \fGP, \measure)$}
    \State $\pointsb \gets \textsc{sample}(n, \fGP, \measure, \dataset)$
    \Comment{Sample a batch of $n$ nodes.}
    \State $\vf_{\pointsb} \gets f(\vx_i)$ for $\vx_i$ in $\pointsb$ 
    \Comment{Evaluate integrand at batch.}
    \State $\dataset \gets \dataset \cup [\pointsb, \vf_{\pointsb}] $
    \Comment{Augment the dataset.}
    \EndWhile
    \State $\algres\gets \textsc{infer}(\fGP, \measure, \dataset)$
    \Comment{The result $\algres$ is a representation of $I_\measure(\fGP) \mid \dataset$.}
    \State \Return $\algres$
  \end{algorithmic}
\end{algorithm}
We observe that Algorithm~\ref{alg:bq-seq} is essentially a loop (lines~3--7). Before each step (one step equals executing lines~4--6) a \textsc{stopping\_condition} checks if the iteration should be halted (line~3). The stopping condition may use information of the dataset $\dataset$ containing the data collected thus far, the model $\fGP$, and the measure $\measure$.
  During each step, a \textsc{sample}r produces a batch of nodes $\pointsb$ of size $n$ (line~4) at which the integrand is evaluated (line~5). Finally, the dataset $\dataset$ is augmented with the newly acquired nodes and observations $\vf_{\pointsb}$ (line~6).
  Once the stopping condition is fulfilled, the loop terminates (the while-statement in line~3 is \textsc{false} and the loop breaks) and the solution $\algres$ is computed and returned (lines~8--9).

Algorithm~\ref{alg:bq-seq} not only has different structure than Algorithm~\ref{alg:bq} but also requires different inputs, most prominently a \emph{stopping condition}, which alludes to one of the general motivations of sequential BQ methods: the sampling (and hence the evaluation of $f$) is paced in order to test for some user-defined stopping criterion. For example, the practitioner may want to terminate the BQ method when the variance $\Var[ I_\measure(\fGP) \mid \dataset ]$ falls below a tolerance value. The BQ method samples as many (small) batches as needed until the stopping condition is fulfilled and it is unknown a priori how many nodes the method eventually collects in total (this is elaborated further under \emph{dynamic budget} in Section~\ref{sec:tax-other}).
  A second, orthogonal motivation for sequential node acquisition emerges if the sampling method in line~4 is model-dependent and changes with the collected dataset up to that point (since $\fGP \mid \dataset$ changes). A sequential approach is then able to gather informative nodes at each step of the loop depending on the current evidence and prior model. An example of the latter are some versions of `active' sampling introduced in the next subsection.

\begin{itemize}
\item {\it Sequential: } Sequential sampling produces nodes in an iterative fashion in small batches of size $n$ as sketched in Algorithm~\ref{alg:bq-seq}. Sequential sampling requires a stopping condition that indicates when the data acquisition should terminate. Often, nodes are sampled one at a time ($n=1$); if $n>1$ then sequential sampling is also called \emph{batched}. 
\citet{CookClayton1998} were the first to focus on sequential BQ methods.
 \end{itemize}

 \subsection{Active sampling}
\label{sec:tax-sampling-active}

The last label of the sampling axis of the taxonomy is ``active'' sampling for BQ. The term ``active'' is used in various machine learning fields and generally describes that the model can query useful or informative data. The idea is that the model may perform better if it can actively select data rather than just passively accept data that is handed to it. Choosing the data thus constitutes a \emph{decision problem} that the model needs to solve. In the context of BQ this means that the model $\fGP$ influences the sampling method and ``requests'' $f$ to be evaluated at certain nodes. For probabilistic models (the BQ model is a probabilistic model), an explicit formulation of this decision problem is possible via Bayesian decision theory. There, node selection is cast as the solution of the following optimization problem (we will use this formulation throughout for active sampling):
  \begin{equation}
  \label{eq:active-sampling}
  \begin{split}
    \vx_1,\dots, \vx_N
    &=\argmax_{\tilde{\vx}_1,\dots, \tilde{\vx}_N} \, \Exp_{\fGP}[u(\tilde{\vx}_1,\dots, \tilde{\vx}_N\mid\fGP, \measure)].
  \end{split}
\end{equation}
In the above formula, the nodes $\vx_1,\dots, \vx_N$ are the maximizer of an expected utility called an \emph{acquisition function}
\begin{equation*}
  a(\tilde{\vx}_1,\dots, \tilde{\vx}_N\mid \measure) \coloneqq \Exp_{\fGP}[u(\tilde{\vx}_1,\dots, \tilde{\vx}_N\mid\fGP, \measure)] 
\end{equation*}
that, together with the $\argmax$ operator, defines the node selection policy and hence the sampling method of the active BQ method. The function $u(\cdot\mid\fGP, \measure)$ is a \emph{utility function} that quantifies the usefulness of evaluating the integrand $f$ at certain nodes under the model $\fGP$.\footnote{The formulation of the active sampler is supported by the von Neumann–Morgenstern utility theorem \citep{VonNeumann2004} which shows, under some conditions, that the decision of a rational agent under uncertainty can be constructed as a policy that maximizes an expected utility function. For an in-depth derivation and motivation of active sampling in the related field of Bayesian optimization see, for example, Section~5.1 in \citet{Garnett2023}. The treatment in BQ is analogous and not repeated here. For Bayesian decision theory for GPs, see Section~2.4 in \citet{RasmussenWilliams2006}.}

While the utility function is arbitrary in the sense that it describes the practitioner's conviction of usefulness, the active BQ method internally assumes the resulting sampling policy to provide the best course of action. 
Once again, this is best explained through an example.
 Let $\tilde{\dataset} = \{ (\tilde{\vx}_1, \fGP(\tilde{\vx}_1)), \ldots (\tilde{\vx}_N, \fGP(\tilde{\vx}_N)) \}$ denote a hypothetical dataset, meaning that the targets are the hypothetical values $\fGP(\tilde{\vx}_i)$ under the model [instead of the integrand values $f(\tilde{\vx})$], and let $\fGP\mid\tilde{\dataset}$ be the model conditioned on the hypothetical dataset. If we choose $u(\cdot\mid\fGP, \measure) = -(I_{\measure}(\fGP\mid \tilde{\dataset}) - \mu_{\tilde{\dataset}})^2$ to be the negative [Eq.~\eqref{eq:active-sampling} is a maximization problem] \emph{square error} between the integral $I_{\measure}(\fGP\mid \tilde{\dataset})$ of the model and its expectation $\mu_{\tilde{\dataset}}$ (the square error is a natural and decent choice of utility), the resulting acquisition function $a(\tilde{\vx}_1,\dots, \tilde{\vx}_N\mid \measure)=-\Var[ I_\measure(\fGP) \mid \tilde{\dataset}]$ equals the negative variance of the integral over the conditioned model and its maximization yields maximally informative nodes \citep[e.g.,][Sec.\@~10.2]{pnbook22}. Likewise, other utility functions can be linked to well-known BQ acquisition functions, some of which will be discussed further in Section~\ref{sec:sampling-active}.

Eq.~\eqref{eq:active-sampling} assumes that the practitioner aims to sample a set of $N$ nodes at once. Often, active node sampling is deployed in a sequential fashion due to the mentioned potential benefits of sequential node acquisition of model-dependent sampling schemes. In such cases, the utility function takes into account sampling decisions under the model in future steps of the iteration when making the sampling decision in the current step \citep[e.g.,][Section~5.2]{Garnett2023}. In practice, for tractability reasons
  often only an isolated decision is considered, meaning that the utility function encodes the usefulness of sampling the current step (or even just a single next node) as if it were the last step/node that the model can ever request. Such an approach is called \emph{one-step-lookahead}
  or \emph{(maximally) myopic} sequential active sampling and is usually the default over more elaborate schemes due to its practicality and tractability \citep[e.g.,][Sec.\@~7.2]{Garnett2023}.\footnote{Non-myopic sequential, active sampling has so far received little attention in BQ. The only approach we are aware of is due to \citet{Jiang19}. A simple example can illustrate the difference between myopic and non-myopic sampling: Consider a sampling horizon of $N=2$ and let $\fGP$ be a Brownian motion as in the example of Section~\ref{sec:tax-sampling-model-dependendent}. A non-myopic sampler would place the first node at $x=4/5$, anticipating a future node on $x=2/5$, while the initial node of the myopic sampler would be placed sub-optimally at $x=2/3$ \citep[e.g.,][Sec.\@~11.1.3]{pnbook22}.} Due to its importance, Algorithm~\ref{alg:active-sampl} shows this approach.
\begin{algorithm}[H]
  \caption{Myopic active sampling for Bayesian quadrature}
  \label{alg:active-sampl}
  \begin{algorithmic}[1]
    \footnotesize
    \Require
    \Statex \textbullet~The integration measure $\measure$.
    \Statex \textbullet~The model $\fGP$ given by $\gGP$ and $\transformation$.
    \Statex \textbullet~The so-far collected data $\dataset$. 
    \Function{sample}{$\fGP$, $\measure$, $\dataset$}
    \State $u_{\fGP}\gets u(\cdot \mid \fGP, \measure, \dataset)$
    \Comment{Define the utility function (only for notation).}
    \State $a \gets\Exp_{\fGP}[u_{\fGP}(\cdot)]$
    \Comment{Define the acquisition function.}
    \State $\vx \gets \textsc{policy}(a)$
    \Comment{The policy maximizes the acquisition function.}
    \State \Return $\vx$
    \EndFunction
    \State 
    \Function{policy}{$a$}
    \State $\vx \gets \argmax_{\tilde{\vx}} a(\tilde{\vx})$
    \Comment{Analytic or with off-the-shelf optimizers.}
    \State \Return $\vx$
    \EndFunction
  \end{algorithmic}
\end{algorithm}
Let us inspect the return object (line~5) of Algorithm~\ref{alg:active-sampl} which is a single node $\vx$ since the active sampler is maximally myopic with batch size $n=1$.\footnote{Note that non-myopic sampling and batched sampling are orthogonal concepts. In principle, there exist batched ($n>1$) and non-batched ($n=1$) methods that can be either myopic or not. Batched BQ methods that use maximally myopic acquisition functions are proposed in \citet{Wagstaff2018} and \citet{HongWeiBeer2024}. Further, batched and non-batched methods need of course not even be active.}
The sampled node is the return object of a maximum acquisition policy (line~4).
The policy, which is highlighted in a separate function on lines~8--11, takes an acquisition function as input that is defined like a lambda function on lines~2 and 3 inside the sampling method based on $\fGP$, $\measure$ and $\dataset$. The dependence of $u(\cdot \mid \fGP, \measure, \dataset)$ on the current dataset $\dataset$ in Algorithm~\ref{alg:active-sampl} is due to the sequential nature of the sampler and was not present in Eq.~(\ref{eq:active-sampling}). The maximizer of the optimization problem on line~9 can either be implemented directly if an analytically solution is available, or be found by off-the-shelf optimizers, especially if gradients of $\fGP$ are accessible.\footnote{It is not necessary to have access to gradients of the integrand $f$ since the gradients of the acquisition function $a$ depend only on gradients of the model $\fGP$, which is under the control of the practitioner.}

We are finally ready to define the ``active'' label.

\begin{itemize}
\item {\it Active: } Actively sampled nodes are the solution to an optimization problem as defined in Eq.~\eqref{eq:active-sampling}. They require the choice of a utility function $u$ by the practitioner. Often active samplers are deployed in combination with an iterative, sequential sampling scheme.
  Similar concepts exist in \emph{experimental design} \citep{Fisher1936} and \emph{Bayesian optimization} \citep{Mockus1989, Garnett2023} where the utility functions are tailored to the respective quantity of interest. Section~\ref{sec:sampling-active} discusses utility functions specific to Bayesian quadrature.
  If the resulting acquisition function depends on some previously collected dataset $\dataset$,
  we further distinguish \emph{implicitly} and \emph{explicitly active} sampling which describes the type of dependence.
  \begin{itemize}
  \item[$\circ$] \emph{Implicitly active: }
    Active sampling is implicit if the acquisition function depends on the observations $\vf_\points$ only via a data-dependent model for $\f$.
    This is for example the case when the kernel-parameters of a conjugate BQ model are estimated with empirical Bayes after each step, while performing sequential active sampling (Section~\ref{sec:param-estimation} will discuss kernel-parameter estimation).
    Implicitly active sampling is sometimes also called \emph{semi-active},
    \emph{non-adaptive}, \emph{open-loop} \citep[e.g.,][]{pnbook22} or simply \emph{active}~\citep[e.g.,][]{Gessner20}.
    
\item[$\circ$] \emph{Explicitly active: } Active sampling is explicit if the acquisition function depends on $\vf_\points$ directly, beyond a data-dependent model.
  An example is uncertainty sampling in the WSABI method~\citep{Gunter2014}. 
  In contrast to implicitly active sampling, explicitly active sampling schemes cannot be precomputed if data-dependency of the model is dropped. 
  Explicitly active sampling thus naturally combines with sequential sampling. Explicitly active sampling is also referred to as \emph{adaptive}\footnote{The term \emph{adaptive} has several additional meanings in numerical integration literature: in \emph{adaptive quadrature} the integral is approximated by a sum of quadratures rules on sub-domains that are adaptively refined until a desired error tolerance is reached~\citep[Sec.\@~4.7]{Press2007}, while \emph{dimension-adaptive quadrature} attempts to find the dimensions which contribute most to the integral~\citep{GerstnerGriebel2003}. Although the modelling framework would be well-suited to the task, no dimension-adaptive Bayesian quadrature methods appear to have been developed so far.} or \emph{closed-loop} \citep[both, e.g.,][]{pnbook22} sampling.
  \end{itemize}
\end{itemize}

The taxonomy has illustrated that \emph{Bayesian quadrature} should be understood as an umbrella term (under Definition~\ref{def:bq}) and that Bayesian quadrature \emph{methods} differ significantly in model assumption, inference method, sampling procedure and even algorithmic requirements and structure. They indeed range from methods that accept fixed datasets to agent-like methods that interact with the data collection mechanism.
Figure~\ref{fig:tax-graph} illustrates the taxonomy and Table~\ref{tab:references} categorizes some existing BQ methods.

\newcommand{\fcolshade}{15}
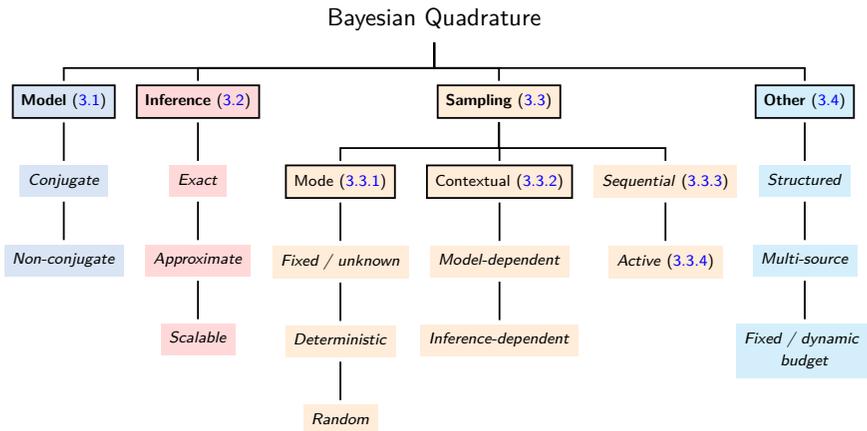
\begin{figure}[t]
  \centering
  \scalebox{0.85}{
    \begin{forest}
    for tree={          
         font = \sffamily\scriptsize\linespread{0.8}\selectfont,
        align = center,
        fill = white,
        minimum size = 1.1em,
         grow = south,
         edge = {thick},
         outer sep = 2pt,
      forked edge,        
        s sep = 2mm,    
        l sep = 6mm,    
     fork sep = 4mm,    
               }
[\large Bayesian Quadrature,
    [\textbf{Model} (\ref{sec:tax-models}), fill=bluegraph!\fcolshade, draw, thick,
        [\emph{Conjugate}, fill=bluegraph!\fcolshade
          [\emph{Non-conjugate}, fill=bluegraph!\fcolshade
          ]
        ]
    ]
    [\textbf{Inference} (\ref{sec:tax-inf}), fill=red!\fcolshade, draw, thick,
        [\emph{Exact}, fill=red!\fcolshade
          [\emph{Approximate}, fill=red!\fcolshade
            [\emph{Scalable}, fill=red!\fcolshade
            ]
          ]
        ]
    ]
    [\textbf{Sampling} (\ref{sec:tax-sampling}), fill=orange!\fcolshade, draw, thick, before drawing tree={x-=1.75pt}, 
        [Mode (\ref{sec:tax-sampling-mode}), fill=orange!\fcolshade, draw, thick,
          [\emph{Fixed / unknown}, fill=orange!\fcolshade
            [\emph{Deterministic}, fill=orange!\fcolshade
              [\emph{Random}, fill=orange!\fcolshade
                ]
              ]
            ]
        ]
        [Contextual (\ref{sec:tax-sampling-model-dependendent}), fill=orange!\fcolshade, draw, thick,        
          [\emph{Model-dependent}, fill=orange!\fcolshade
              [\emph{Inference-dependent}, fill=orange!\fcolshade
            ]
          ]
        ]
        [\emph{Sequential} (\ref{sec:tax-sampling-sequential}), fill=orange!\fcolshade
        [\emph{Active} (\ref{sec:tax-sampling-active}), fill=orange!\fcolshade,
        ]
      ]
    ]
    [\textbf{Other} (\ref{sec:tax-other}), fill=cyan!\fcolshade, draw, thick,
      [\emph{Structured}, fill=cyan!\fcolshade
        [\emph{Multi-source}, fill=cyan!\fcolshade
          [\emph{Fixed / dynamic} \vspace{0.1cm} \\ \emph{budget}, fill=cyan!\fcolshade
          ]
        ]
      ]
    ]
]
  \end{forest}  
  }
    \caption{Taxonomy of Bayesian quadrature methods along three main axes: \emph{Model}, \emph{inference} and \emph{sampling} as introduced in Chapter~\ref{sec:taxonomy}. The labels under \emph{other} are listed for completeness. The labels in each axis do not follow a hierarchy and are not necessarily mutually exclusive.
    Section numbers are given in parentheses.      
  }
  \label{fig:tax-graph}
\end{figure}

\section{Other terms}
\label{sec:tax-other}
This section lists terms that are occasionally used to describe some characteristics of Bayesian quadrature algorithms but are not part of our taxonomy. We list them for completeness.

\begin{itemize}
\item {\it Structured: }
 Bayesian quadrature is sometimes called structured\footnote{Our use of the term has nothing in common with \emph{structured kernel interpolation} that is related to scalable GP regression~\citep{WilsonNickisch2015}.} when the model $\fGP$ encodes a known structural property, such as an invariance or range constraint, of the integrand $f$ via the transformation $\transformation$ or otherwise.
  This is to highlight that the method is designed to address certain prior information that may be available about $f$.
  For example, \citet{Naslidnyk21} encode input invariances of the integrand $f$ in a conjugate BQ model. \citet{Chai2019a} and \citet{Gunter2014} use exponential and square transforms $\varphi$, respectively, to encode non-negativity in a non-conjugate model (see Section~\ref{sec:non-conj-bq}).
\item {\it Fixed \& dynamic budget: } The computational budget of a Bayesian quadrature algorithm is \emph{fixed} if the total number of integrand evaluations $N$ is determined before the algorithm runs. This is the case if $N$ is set by the practitioner by, for example, providing a fixed set of nodes $\vx_1, \dots, \vx_N$ or by providing $N$ directly alongside a sampler.
  The budget is \emph{dynamic} if the number of evaluations of $f$ is not known before run time,
  and the algorithm determines an appropriate $N$ automatically.\footnote{Apart from a safeguard condition that caps the number of $f$-evaluations at some large $N_\textup{max}$.}
Methods with a dynamic budget are sometimes called \emph{automatic} as for example in \citet{Rathinavel2019, Jagadeeswaran2022, Rathinavel2019-thesis}.
 A dynamic budget does not imply that node selection is active.
 Whether a fixed or dynamic budget is preferable in practice depends on the requirements and constraints of the application.
 \item {\it Multi-source: } Multi-source Bayesian quadrature is an umbrella term that implies that not only one, but multiple, usually related integrals are being considered simultaneously.
  Several settings exist: \emph{Multi-fidelity} BQ assumes a natural ordering of the integrals (e.g., from cheap to expensive) and usually aims to infer the integral with the highest fidelity  (e.g., the most expensive one) from integrand evaluations of all fidelity levels. The ordering can either be encoded in the model directly \citep{Li2022} or implicitly via the sampling strategy \citep{Gessner20}. \emph{Multi-source} BQ \citep{Gessner20} generalizes multi-fidelity BQ in that it aims to learn a primary integral but does not assume a natural ordering of the related integrals. \cite{Xi2018} assume no primary integral and aim to learn all related integrals equally well. \emph{Active} multi-source BQ \citep{Gessner20} can, via its acquisition function, naturally be tailored towards aiming to learn a primary, some, or all related integrals.
  Other BQ methods that feature multiple integrals include \emph{Bayesian quadrature optimization} methods which seek to optimise functions of the form $g(\vy) = \int_\domain \f(\vx, \vy) \dif P(\vx)$~\citep{Nguyen2020, Iwazaki2021, Sadeghi2022, Toscano2022} and the conditional BQ method of \citet{Chen2024} that estimates integrals with respect to a parametric family of distributions $\{ P_\hparams \}_{\hparams \in \Theta}$. 
  While multi-source BQ does not strictly comply with Definition~\ref{def:bq}, resulting algorithms fit well into our taxonomy.
\end{itemize}

\begin{table}[t]
  \def\arraystretch{1.5} 
  \scriptsize
    \centering
    \begin{tabular}[t]{%
      p{0.24\textwidth}%
      |p{0.115\textwidth}%
      p{0.09\textwidth}%
      p{0.15\textwidth}%
      >{\raggedright\arraybackslash}p{0.225\textwidth}%
      }
{\bf Method}  & {\bf Model} & {\bf Inference}& {\bf Sampling} & {\bf Reference}\\
 \hline
 \rowcolor{black!05} \makecell{Bayesian Monte Carlo\\(BMC)} & \makecell{Conjugate} & \makecell{Exact}& \makecell{Random} & \cite{RasmussenGhahramani2002} \\
 \makecell{Warped sequential \\ active Bayesian \\ integration (WSABI)}  & \makecell{Non-conjug.\\Structured} & \makecell{Approx.}& \makecell{Deterministic\\Sequential\\Active (expl.)} & \cite{Gunter2014}\\
 \rowcolor{black!05} \makecell{BQ for multiple \\related integrals}  &  \makecell{Conjugate\\ Multi-output}& \makecell{Exact} & \makecell{Random or\\ Deterministic}& \cite{Xi2018}\\
 \makecell{Variational BMC} & \makecell{Conjugate} & \makecell{Exact} & \makecell{Deterministic \\ Sequential \\ Active (expl.)} & \makecell{\cite{Acerbi2018b}; \\ \cite{Acerbi2018a}} \\
 \rowcolor{black!05} \makecell{Fully symmetric\\kernel quadrature} &\makecell{Conjugate}  & \makecell{Exact\\Scalable}& \makecell{Deterministic\\Infer-depend.} & \cite{KarvonenSarkka2018}\\
 \makecell{Fast automatic BQ} &\makecell{Conjugate}  & \makecell{Exact\\Scalable}& \makecell{Deterministic\\Infer-depend.\\Sequential} & \cite{Rathinavel2019}\\
 \rowcolor{black!05} \makecell{Active multi-source BQ} &\makecell{Conjugate\\Multi-output}  & \makecell{Exact}& \makecell{Deterministic\\Sequential\\Active (impl.)} & \cite{Gessner20}\\
 \makecell{Locally adaptive BQ}& \makecell{Conjugate} & \makecell{Exact}  & \makecell{Deterministic\\Sequential\\Active (impl.)} & \cite{Fisher2020}\\
 \rowcolor{black!05} \makecell{Invariant BQ} &\makecell{Conjugate\\Structured}  & \makecell{Exact}& \makecell{Deterministic\\Sequential\\Active (impl.)} & \cite{Naslidnyk21}\\
 \makecell{Bayesian alternately \\ subsampled quadrature} &\makecell{Non-conjug.}  & \makecell{Approx.\\ Scalable}& \makecell{Random\\Model-depend.} & \cite{Adachi2022}
\end{tabular}
    \caption{Some well-known BQ methods categorized according to the taxonomy of Chapter~\ref{sec:taxonomy}.
    }
\label{tab:references}
\end{table}

\section{Example}
\label{sec:tax-example}

Here we present an example of a BQ method and examine how modififying the method would affects its classification under the taxonomy.
The method is only a conceptual example. We have not implemented it and there may be practical issues that render it impractical.
Proper numerical experiments are found in \Cref{sec:experiments}.

Let $\domain = [0, 1]^d$ and suppose that integration problem is to compute $I_\measure(\f) = \int_D f(\vx) \dif \vx$ for some integrand function $\f \colon D \to \R$.
Our illustrative BQ method is the following:

\begin{itemize}
\item[] \emph{Model:} To model $\f$, we select the zero-mean GP model $\fGP \sim \GP(0, k_\hparams)$ whose covariance function is the square exponential kernel in~\eqref{eq:se-kernel} and $\hparams = (\sigma, \ell)$ are the kernel parameters. This choice means that the method is \emph{conjugate} (\Cref{sec:tax-models}).
  
\item[] \emph{Sampling:} Let $\{x_1, \ldots, x_n\}^d \subset [0, 1]^d$ be the $d$-ary Cartesian product of a set points $\{x_1, \ldots, x_n \} \subset [0, 1]$. Select the number of iterations $i_{\max} \geq 1$ and sample the nodes by iterating the following steps for $i = 1, \ldots, i_{\max}$:
  \begin{enumerate}
  \item Denote $X(x) = \{x_1, \ldots, x_{i-1}, x\}^d \subset [0, 1]^d$ and solve the optimization problem
    \begin{equation} \label{eq:acquisition-example}
      x_i = \argmax_{x \in [0, 1]} \rho^2(x)
    \end{equation}
    for
    \begin{equation*}
      \rho^2(x) = \vk_{\hparams,\measure\points(x)}\Trans \mK_{\hparams,\points(x)\points(x)}^{-1}\vk_{\hparams,\measure\points(x)} / k_{\hparams,\measure\measure}.
    \end{equation*}
    That the quantities defined in \Cref{sec:conj-bq} depend on $\hparams$ is made explicit here.
  \item 
    Set $\points = \points(x_i)$ and \smash{$\dataset = \{(\vx, f(\vx))\}_{\vx \in \points}$} and note that $X$ contains $N_i = i^d$ points. In the notation of \Cref{alg:bq-seq}, \smash{$\pointsb = \{ x_1, \ldots, x_i \}^d \setminus \{ x_1, \ldots, x_{i-1} \}^d$} and the batch size $n = i^d - (i - 1)^d = N_i - N_{i-1}$ depends on the step.
  \item Use maximum likelihood to update the kernel parameters (see \Cref{sec:param-estimation}):
    \begin{equation*}
      \hparams = \argmax_{ \tilde{\hparams} \in \Theta} \bigg\{ -\frac{1}{2} \vf_{\points}\Trans \mK_{\tilde{\hparams},\points\points}^{-1} \vf_{\points} -\frac{1}{2} \log{\det \mK_{\tilde{\hparams},\points\points}} \bigg\},
    \end{equation*}
    where $\Theta$ is a suitable subset of $[0, \infty) \times (0, \infty)$.
  \end{enumerate}
  The method is \emph{sequential} and \emph{batched} (\Cref{sec:tax-sampling-sequential}).
  It is also \emph{deterministic} (\Cref{sec:tax-sampling-mode}) because a perfect implementation would contain no randomness, though in practice there may be a degree of randomness in the optimization performed during Steps~1 and~3.
  Because $x_i$ in Step~1 depends on the model $\fGP$ via the covariance function $\kernel$, the method is \emph{model-dependent} (\Cref{sec:tax-sampling-model-dependendent}).
  In fact, as we shall see in \Cref{sec:sampling-active}, on each step the optimization problem~\eqref{eq:acquisition-example} can be written in the form~\eqref{eq:active-sampling}, which makes the method \emph{active} (and, in this case, \emph{myopic}; see \Cref{sec:tax-sampling-sequential}).
  Since there is no ``direct'' dependency on the observations $\vf_\points$ in Step~1, the method is not explicitly active.
  But because $\rho^2$ depends on the kernel parameters, which in turn depend on $\vf_\points$ via Step~3, the method is \emph{implicitly active}.
  
\item[] \emph{Inference:} After the nodes have been sampled, the method returns the posterior $I_\measure(\fGP) \mid \dataset \sim \Normal(\mu_\dataset, \Sigma_\dataset)$ by computing $\mu_\dataset$ and $\Sigma_\dataset$ from~\eqref{eq:bq-mean} and~\eqref{eq:bq-var}.
  Because both $\kernel_{\hparams,\measure}(\vx)$ and $\kernel_{\hparams,\measure\measure}$ are available in closed form for the square exponential kernel and the uniform measure on $[0, 1]^d$ (see \Cref{sec:kern-mean-var}), inference is \emph{exact} (\Cref{sec:tax-inf}) potential numerical problems are ignored.
  The method becomes \emph{scalable} (\Cref{sec:tax-inf}) but remains exact if one does not naively solve~\eqref{eq:bq-mean} and~\eqref{eq:bq-var} but instead uses the well-known Kronecker decompositions of $\mK_{\hparams,\points\points}$ and $\vk_{\hparams,\measure\points}$ for Cartesian product nodes and product kernels, such as the square exponential~\citep[Sec.\@~4]{OHagan1991}.
  Sampling would no longer be model-dependent, active, or sequential should the method be modified so that $\points$ were simply some pre-specified product design, such as the $d$-ary product of equispaced points on $[0, 1]$.
  However, this sampling could be considered \emph{inference-dependent} (\Cref{sec:tax-sampling-model-dependendent}) as a product design is needed for scalable inference.
\end{itemize}

Next we turn to practical issues, many of them present in the method above, that one needs to negotiate in order to successfully design and implement a BQ method.

\chapter{Practical issues}
\label{sec:practical-issues}

In this chapter we describe a number of practical issues that one encounters and how to solve them when implementing Bayesian quadrature algorithms. 
Some, such as numerical ill-conditioning and the need for kernel parameter estimation, need to be taken into account in all applications of Gaussian process regression and interpolation, while other, in particular the computation of kernel means and variances, are intrinsic to Bayesian quadrature.
We conclude the section with an overview of what we view as the main limitations of Bayesian quadrature and how these have affected---and continue to affect---algorithmic development and the range of integration problems that Bayesian quadrature can or should solve.

\section{Kernel means and variances}
\label{sec:kern-mean-var}

Conjugate Bayesian quadrature requires the computation of $\mu_{\dataset}$ and $\Sigma_{\dataset}$ as in Eq.~\eqref{eq:conj-bq-mean-var} in order to characterize $I_{\measure}(\fGP)\mid\dataset$.
Eq.~\eqref{eq:conj-bq-mean-var} poses the challenge of computing the vector $\vk_{\measure\points}= ( \kernel_\measure(\vx_1), \ldots \kernel_\measure(\vx_N))\in\R^N$ of kernel means in Eq.~\eqref{eq:kernel-mean} and the scalar kernel variance $\kernel_{\measure\measure}$ in Eq.~\eqref{eq:kernel-var}. A straightforward approach is to use a kernel $\kernel$ that admits an analytic solution of both $\kernel_{\measure\measure}$ and $\kernel_\measure(\vx)$.
A well-known example is the square exponential kernel in Eq.~\eqref{eq:se-kernel} integrated with respect to the standard Lebesgue measure on the domain $\domain=[0, 1]^d$, which, for $d=1$ yields
\begin{equation*}
  \begin{split}
    \kernel_{\measure}(x)
    &= \Big(\frac{\pi\ell^2}{2}\Big)^{\frac{1}{2}}\Big[
    \erf\Big(\frac{1 - x}{\ell\sqrt{2}}\Big)
    + \erf\Big(\frac{x}{\ell\sqrt{2}}\Big)
  \Big], \\
  \kernel_{\measure\measure}
  & = 2\ell^2(e^{-\frac{1}{2\ell^2}}-1)
  + (2\pi\ell^2)^{\frac{1}{2}}\erf\Big(\frac{1}{\ell\sqrt{2}}\Big).
  \end{split}
\end{equation*}
Comprehensive lists of closed form kernel embeddings can be found in \citet{Briol2025} and \citet[Section~10.1]{pnbook22}.
The former is accompanied with a kernel mean embedding library in Python\footnote{\url{https://github.com/mmahsereci/kernel_embedding_dictionary}}.
A variety of kernel embeddings are also implemented in the Python libraries \emukit~\citep{Paleyes19, Paleyes2023} and \probnum~\citep{Wenger21}. Approximations to non-analytic kernel embeddings for use in BQ have been less explored in the literature. Notable exceptions are \citet{Warren2022} and \citet{Warren2024}, who use fast Fourier transforms to approximate kernel means of stationary kernels.
See also \citet{Tronarp2018} for a different approximation for isotropic Matérns.

Non-conjugate BQ often requires the computation of further quantities, such as $\int_{\domain}\int_{\domain}\kernel^2(\vx,\vx')\dif \measure(\vx)\dif \measure(\vx')$, the integral of the kernel squared, and other combinations of kernel products \citep[e.g., WSABI in][]{Gunter2014}. 
Hence, methods based on non-conjugate BQ models are often specific to measure-kernel combinations which meet these requirements.
Such integrals are also required in application of Bayesian quadrature to non-linear Kalman filtering~\citep{Deisenroth2009, Deisenroth2012, SarkkaHartikainen2014, KerstingHennig2016, PruherSimandl2016, Pruher2018, Pruher2021}.

\section{Sampling and node selection}
\label{sec:sampling}

Section~\ref{sec:tax-sampling} described how the nodes may be sampled (i.e., selected).
This section details practical choices and issues one is faced with when designing a sampling procedure for BQ.

\subsection{Model-independent deterministic designs}
\label{sec:deterministic-sampling}

\begin{figure}
  \centering
  \includegraphics[width=0.28\textwidth]{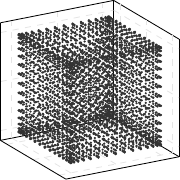}
  \hspace{0.5cm}
  \includegraphics[width=0.28\textwidth]{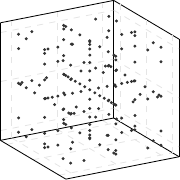}
  \hspace{0.5cm}
  \includegraphics[width=0.28\textwidth]{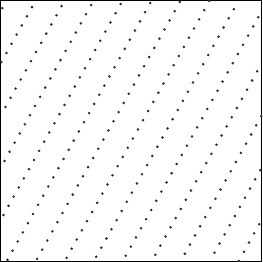}
  \caption{Three deterministic sampling designs mentioned in Section~\ref{sec:deterministic-sampling}: Tensor grid ($d=3$), sparse grid ($d=3$), and a lattice ($d=2$).}
  \label{fig:deterministic-sampling-designs}
\end{figure}

If computational or other reasons prevent active sampling and better coverage of the domain than what is achievable via random sampling is desired, it is common to use deterministic designs that are model-independent\footnote{However, it is always good to keep in mind some properties of the prior when selecting the nodes. For example, equispaced nodes are not well suited for approximation of infinitely differentiable functions~\citep{PlatteTrefethenKuijlaars2011}. One should avoid equispaced nodes if the kernel is, for example, the square exponential.} (see Section~\ref{sec:tax-sampling-model-dependendent}).
If the domain is a hypercube (say, $\domain = [0, 1]^d$), arguably the simplest deterministic design is the regular \emph{tensor grid}
\begin{equation} \label{eq:regular-tensor-grid}
  \begin{split}
  \points_{N_1}^d &= X_{N_1} \times \cdots \times X_{N_1} \\
  &= \bigg\{ \frac{1}{N_1+1}, \ldots, \frac{N_1}{N_1+1} \bigg\} \times \cdots \times \bigg\{ \frac{1}{N_1+1}, \ldots, \frac{N_1}{N_1+1} \bigg\},
  \end{split}
\end{equation}
which consists of $N = N_1^d$ nodes.
More advanced designs are obtained by replacing each one-dimensional node sets $X_{N_1}$ with the nodes of an appropriate Gaussian quadrature rule, such as Gauss--Hermite if $\measure$ is standard Gaussian on $\R^d$ or Gauss--Legendre if $\measure$ is uniform on $[0, 1]^d$.
The resulting designs are sometimes called \emph{classical} in the BQ literature.
The computation of these nodes is efficiently implemented by most scientific computing libraries, including \scipy~for Python.
As the number of points depends exponentially on the dimension, tensor grids are inflexible and difficult to use in high dimensions.
\emph{Sparse grids}~\citep[Sec.\@~4.2]{KarvonenSarkka2018} as well as \emph{lattice nodes}~\citep{Rathinavel2019} and other designs from quasi-Monte Carlo literature may be used in high dimensions.
Some of the above designs are depicted in Figure~\ref{fig:deterministic-sampling-designs} for $d \in \{2, 3\}$.
It is a greate advantage of Bayesian quadrature that one may use any nodes that have been proposed in the literature.

In using \emph{non-nested} nodes, many deterministic designs implicitly assume that the computational budget is fixed (see Section~\ref{sec:tax-other}).
For example, $X_{N_1} \not\subset X_{N_1+1}$ if $X_{N_1}$ are the equispaced nodes in~\eqref{eq:regular-tensor-grid} or any classical Gaussian quadrature nodes~\citep[nested modifications of which exist; see][Sec.\@~2.7.1.1]{DavisRabinowitz1984}.
A non-nested design can be used as an initial design in sequential Bayesian quadrature where more nodes are sampled (e.g., actively) until some stopping criterion is met.
\emph{Nested} deterministic designs include the extensible integration lattices used in \citet{Rathinavel2019} and the Clenshaw--Curtis nodes~\citep{Trefethen2008} when $N$ is a power of two (i.e., $X_{2^n} \subset X_{2^{n+1}}$).

\subsection{Random nodes}
\label{sec:sampling-random}

An early instance of a BQ method---Bayesian Monte Carlo \citep{RasmussenGhahramani2002}---uses nodes that are random samples from the integration measure $\measure$. As such it puts higher value on observing the integrand in areas where the density $p$ of $\measure$ is large than where it is small. This sampling method is model-uninformed.
\citet{Adachi2022} recently proposed a sampling distribution for BQ that is inspired by importance sampling and also covers areas where $f$ is supposedly large (even if the density $p$ is small). This sampling method, in contrast, is model-dependent.
Purposely sampling the nodes from a distribution different from $\measure$ has also been proposed and studied~\citep{Briol2017, Bach2017}.

The space-filling Latin hypercube design \citep{MacKay1979} is a random, scalable equivalent to a grid. 
Being model-independent, it is used as initial design in sequential active BQ methods.
Determinantal point processes, used in BQ by \citet{Belhadji2019}, provide model-dependent random designs constructed using the kernel function.
Nodes sampled from a determinantal point process are repulsive and space-filling.

\subsection{Actively sampled nodes}
\label{sec:sampling-active}

As promised in Section~\ref{sec:tax-sampling-active}, this section reviews some commonly used utility functions for active BQ and their corresponding acquisition functions. For the conjugate model we give explicit expressions whenever possible. Recall from Section~\ref{sec:tax-sampling-active} that active nodes maximize an expected utility according to Eq.\@~\eqref{eq:active-sampling}. The equation is repeated here for convenience:
\begin{equation*}
  \begin{split}
    \vx_1,\dots, \vx_N
    &\in \argmax_{\tilde{\vx}_1,\dots, \tilde{\vx}_N} \Exp_{\fGP}[u(\tilde{\vx}_1,\dots, \tilde{\vx}_N\mid\fGP, \measure)].
  \end{split}
\end{equation*}
Again, the function $u$ is the utility function and 
\begin{equation*}
  a(\tilde{\vx}_1,\dots, \tilde{\vx}_N\mid \measure) \coloneqq \Exp_{\fGP}[u(\tilde{\vx}_1,\dots, \tilde{\vx}_N\mid\fGP, \measure)] 
\end{equation*}
the corresponding acquisition function. The above equation, at heart, considers what would happen under the model $\fGP$ if the integrand $f$ were hypothetically evaluated at some yet unseen nodes $\tilde{\vx}_1,\dots, \tilde{\vx}_N$, expressed quantitatively as an average utility $a$. The active sampler returns the \emph{maximizer} of $a$. In other words, the sampler produces nodes that yield the maximal average utility under the model $\fGP$.
A question of course arises: How to choose $u$?

From an information theoretic perspective, a meaningful acquisition function is the \emph{mutual information} (MI) between the integral $I_{\measure}(\fGP)$ and yet unseen integrand values $\rvf_{\tilde{\points}} \coloneqq (\fGP(\tilde{\vx}_1),\dots, \fGP(\tilde{\vx}_N))\in\R^N$ at nodes $\tilde{\points} = \{\tilde{\vx}_1,\dots, \tilde{\vx}_N\}$. Let $\tilde{\dataset} =  \{ (\tilde{\vx}_1, \fGP(\tilde{\vx}_1)), \ldots (\tilde{\vx}_N, \fGP(\tilde{\vx}_N)) \}$ be the corresponding hypothetical dataset.
The MI acquisition function is
\begin{equation*}
  \begin{split}
    a(\tilde{\vx}_1,\dots, \tilde{\vx}_N\mid \measure)
    &= \MI(I_{\measure}(\fGP); \rvf_{\tilde{\points}}) \\
    &\coloneqq \Exp_{\fGP}[ \KL{ \rvf_{\tilde{\points}}\mid I_{\measure}(\fGP), \tilde{\points} }{\rvf_{\tilde{\points}}\mid\tilde{\points}} ] \\
  &= \Exp_{\rvf_{\tilde{\points}}}[\KL{I_{\measure}(\fGP)\mid \tilde{\dataset}}{I_{\measure}(\fGP)}],
  \end{split}
\end{equation*}
where $\KL{Q}{Q'}$ is the Kullback–Leibler (KL) divergence between two distributions $Q$ and $Q'$.
The MI acquisition function measures the expected change in distribution between the integral $I_{\measure}(\fGP)$ over the prior model and the integral $I_{\measure}(\fGP)\mid\tilde{\dataset}$ over the model that is condition on the hypothetical dataset $\tilde{\dataset}$.
Maximizing it yields maximally informative nodes (in the Shannon sense) about the integral value under the model $\fGP$.
For conjugate Bayesian quadrature, the mutual information is analytic and has the form
\begin{align}
  \label{eq:mi-conj}
  \MI(I_{\measure}(\fGP); \rvf_{\tilde{\points}})
   = -\frac{1}{2}\log( 1 - \rho^2(\tilde{\vx}_1,\dots, \tilde{\vx}_N) ),
\end{align}
where
\begin{align}
  \label{eq:rho2}
  \rho^2(\tilde{\vx}_1,\dots, \tilde{\vx}_N)
  \coloneqq (\vk_{\measure\tilde{\points}}\Trans \mK_{\tilde{\points}\tilde{\points}}^{-1}\vk_{\measure\tilde{\points}})/\kernel_{\measure\measure} \in [0, 1]
\end{align}
is the scalar squared correction.
Here $\kernel_{\measure\measure}$ is the initial variance defined in~Eq.\@~\eqref{eq:kernel-var}, $\vk_{\measure\tilde{\points}} = ( \kernel_\measure(\tilde{\vx}_1), \ldots \kernel_\measure(\tilde{\vx}_N)) \in \R^N$ is a vector, and $\mK_{\tilde{\points}\tilde{\points}} = (\kernel(\tilde{\vx}_i, \tilde{\vx}_j))_{i,j=1}^N \in \R^{N \times N}$ a matrix.
For example \citet{Chai2019b} and \citet{Hamid2021} use information-theoretic acquisition functions.

Other frequently used BQ acquisition functions include the \emph{integral variance reduction} (IVR) and the \emph{negative integral variance} (NIV) mentioned in \Cref{sec:tax-sampling-active} that arise from the square error utility:
  \begin{align*}
    \IVR(&I_{\measure}(\fGP); I_{\measure}(\fGP)\mid\tilde{\dataset}) \\
     &\coloneqq s^{-1}\Exp_{\fGP}\big[
      (
      I_\measure(\fGP) - \Exp_{\fGP}[ I_\measure(\fGP)]
      )^2
      -
      (
      I_\measure(\fGP) \mid \tilde{\dataset} - \Exp_{\fGP}[ I_\measure(\fGP) \mid \tilde{\dataset} ]
      )^2
      \big]\\
    &= s^{-1}(s - \Var[ I_\measure(\fGP) \mid \tilde{\dataset} ])
  \end{align*}
  and
  \begin{align*}
    \NIV(I_{\measure}(\fGP)\mid\tilde{\dataset})
    & \coloneqq -s^{-1}\Exp_{\fGP}\big[
      (
      I_\measure(\fGP) \mid \tilde{\dataset} - \Exp_{\fGP}[ I_\measure(\fGP) \mid \tilde{\dataset}]
      )^2
      \big]\\
    & =  - s^{-1}\Var[ I_\measure(\fGP) \mid \tilde{\dataset} ].
\end{align*}
The inverse of the factor $s = \Var[ I_\measure(\fGP)]$ (not necessarily equal to $\kernel_{\measure\measure}$ for arbitrary $\fGP$) is a positive constant with respect to $\tilde{\points}$ and serves as normalization. The IVR is simply the (normalized) difference between the variance $s$ of the integral over the prior model and the variance of the integral over the model conditioned on the hypothetical dataset $\tilde{\dataset}$. Hence IVR selects nodes that shrink the variance of the integral maximally. Analogously, NIV chooses nodes that minimize the variance of the integral. 
The IVR and NIV yield the same maximizer and hence the same active nodes as both functions are strictly increasing transformations of each other.
  For the conjugate model, these two acquisition functions are analytic and simplify to 
  \begin{alignat}{3}
  \label{eq:ivr-conj}
   \IVR(I_{\measure}(\fGP); I_{\measure}(\fGP)\mid\tilde{\dataset})
    &= \rho^2(\tilde{\vx}_1,\dots, \tilde{\vx}_N)\quad &&\in [0, 1], \\
  \label{eq:niv-conj}
   \NIV(I_{\measure}(\fGP)\mid\tilde{\dataset})
    &= \rho^2(\tilde{\vx}_1,\dots, \tilde{\vx}_N) - 1 \quad &&\in [-1, 0].
\end{alignat}
By comparing Eqs.\@~\eqref{eq:mi-conj}, \eqref{eq:ivr-conj}, and~\eqref{eq:niv-conj} it is apparent that, in the conjugate case, all three acquisition functions (IVR, NIV, MI) can be expressed as functions of the squared correlation $\rho^2$ and as strictly increasing transformations of each other. Hence all three yield the same nodes \citep{Gessner20}. Figure~\ref{fig:acquisition-funcs} shows the acquisitions as functions of $\rho^2$.\footnote{Active sampling typically assumes that the cost of evaluating $f(\vx)$ is the same for all $\vx$ making maximizing the acquisition function reasonable. However, if the cost varies with $\vx$, an alternative approach may be preferable. The only one we are aware of in the context of BQ is due to \citet{Gessner20} who maximize acquisition rates $r(\tilde{\vx})= a(\tilde{\vx})/c(\tilde{\vx})$ based on some positive cost function $c(\tilde{\vx})$. For example, if $a$ is MI and $c$ measures the cost in time or monetary value, then $r$ would yield nodes that maximize bits per time or bits per dollar. In this cost-sensitive scenario, even if $\argmax a_1(\tilde{\vx}) = \argmax a_2(\tilde{\vx})$ for two different acquisition functions $a_1$ and $a_2$ (such as IVR and NIV), generally $\argmax r_1(\tilde{\vx}) \neq \argmax r_2(\tilde{\vx})$, even if they share the same cost function $c$. In fact, certain acquisition functions, like NIV, cannot be converted to meaningful acquisition rates at all. Cost-sensitive evaluations further play a role in multi-source and multi-fidelity active BQ where the acquisition strategy requires careful consideration.}

\begin{figure}
 \centering
  \includegraphics[width=0.7\textwidth]{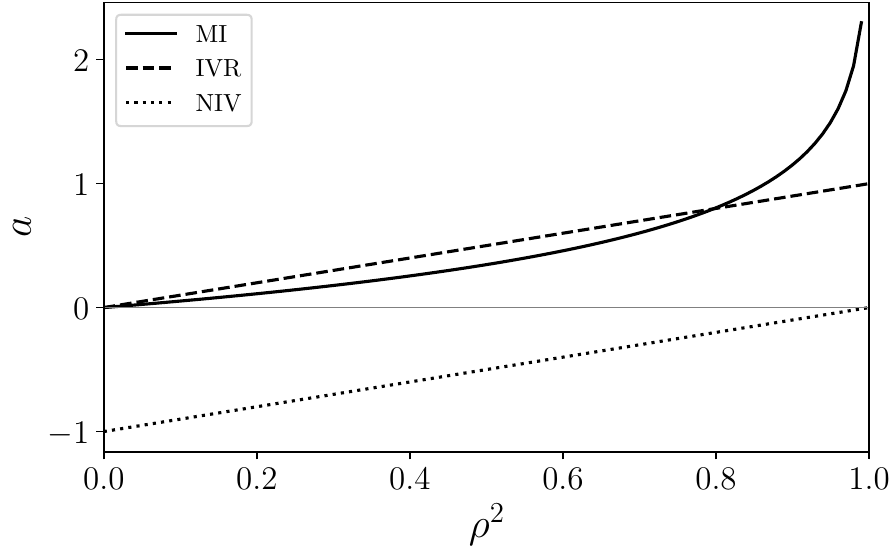}
  \caption{BQ acquisition functions for the conjugate model as functions of $\rho^2$. The three acquisition functions are given in Eqs.~\eqref{eq:mi-conj}, \eqref{eq:ivr-conj}, and~\eqref{eq:niv-conj}.}
  \label{fig:acquisition-funcs}
\end{figure}

  For \emph{sequential} active sampling as introduced in Algorithms~\ref{alg:bq-seq} and \ref{alg:active-sampl} in Sections~\ref{sec:tax-sampling-sequential} and \ref{sec:tax-sampling-active}, the formulas need only slight modification. Recall that a sequential sampler produces a small batch of nodes of size $n$ at every step of the iteration which is then added to the current dataset. Hence, denote the current dataset as $\dataset$ and the hypothetical batch dataset as $\tilde{\dataset} = \{ (\tilde{\vx}_1, \fGP(\tilde{\vx}_1)), \ldots (\tilde{\vx}_n, \fGP(\tilde{\vx}_n)) \}$, which overloads notation in this paragraph.
  Further, denote the current posterior as $\fGP_{\dataset}=\fGP\mid\dataset$. The acquisition function can be defined autoregressively (meaning that $\fGP_{\dataset}$ acts as the prior model in the current step) according to
    \begin{align*}
    \MI(I_{\measure}(\fGP_{\dataset}); \rvf_{\tilde{\points}})
  &= \Exp_{\rvf_{\tilde{\points}}}\big[\KL{ I_{\measure}(\fGP_{\dataset})\mid \tilde{\dataset}}{ I_{\measure}(\fGP_{\dataset})}\big], \\      
    \IVR(I_{\measure}(\fGP_{\dataset}); I_{\measure}(\fGP_{\dataset})\mid\tilde{\dataset})
    &= s^{-1}_{\dataset}(s_{\dataset} - \Var[ I_\measure(\fGP_{\dataset}) \mid \tilde{\dataset} ]), \\
    \NIV(I_{\measure}(\fGP_{\dataset})\mid\tilde{\dataset})
    & =  - s^{-1}_{\dataset}\Var[ I_\measure(\fGP_{\dataset}) \mid \tilde{\dataset} ], 
\end{align*}
with  $s_{\dataset} = \Var[ I_\measure(\fGP_{\dataset})]$ again a constant with respect to $\tilde{\points}$ that only serves to normalize the acquisition functions. The acquisition functions are myopic by definition as they do not consider the impact of future steps. For the conjugate model the autoregressive squared correlation is
\begin{align*}
  \label{eq:rho2-seq}
  \rho^2(\tilde{\vx}_1,\dots, \tilde{\vx}_n)
  &= (\vk_{\dataset; \measure\tilde{\points}}\Trans \mK_{\dataset;\tilde{\points}\tilde{\points}}^{-1}\vk_{\dataset; \measure\tilde{\points}})/\Sigma_{\dataset},
\end{align*}
where $\vk_{\dataset; \measure\tilde{\points}} = I_\measure(\vk_{\dataset; \tilde{\points}}) = ( \kernel_{\dataset; \measure}(\tilde{\vx}_1), \ldots \kernel_{\dataset; \measure}(\tilde{\vx}_n))$ is the kernel mean embedding of the posterior kernel function as in Eq.~\eqref{eq:gp-posterior-covariance} as a function of the $n$ hypothetical nodes, $\mK_{\dataset;\tilde{\points}\tilde{\points}}$ is the corresponding Gram matrix, and $\Sigma_\dataset$ is the integral variance in Eq.\@~\eqref{eq:bq-var}.
For the prevalent case $n=1$, the sequential squared correlation simplifies to
\begin{equation*}
  \rho^2(\tilde{\vx}) = \frac{\kernel_{\dataset;\measure}(\tilde{\vx})^2}{\Sigma_{\dataset}\kernel_{\dataset}(\tilde{\vx}, \tilde{\vx})}.
\end{equation*}

We mention two acquisition functions that are only indirectly related to the integral. \emph{Uncertainty sampling} (US), popular in experimental design, is defined for $n=1$ only by
  \begin{align*}
    \US(\fGP_{\dataset})
     \coloneqq \Exp_{\fGP}\big[
      (
      \fGP_{\dataset}(\tilde{\vx}) - \Exp_{\fGP_{\dataset}}[\fGP_{\dataset}(\tilde{\vx}) ]
      )^2p^2(\tilde{\vx})
      \big]
    =  \Var[\fGP_{\dataset}(\tilde{\vx})]p^2(\tilde{\vx}).
  \end{align*}
For the conjugate model we get $ \US(\fGP_{\dataset}) = k_{\dataset}(\tilde{\vx}, \tilde{\vx})p^2(\tilde{\vx})$.
Uncertainty sampling yields nodes at locations where the variance of $\fGP_{\dataset}$ and the density of the integration measure $p(\tilde{\vx})$ are both large.
By focusing on the variance of the model $\fGP$ instead of its integral $I_{\measure}(\fGP)$, US in a sense disregards the integral as the quantity of interest and may hence be more tractable or accessible than many other acquisition functions.
Uncertainty sampling is for example used in \citet{Gunter2014} and in a related (non-active) context in \citet{Adachi2022}.
The \emph{posterior variance contribution} (PVC) acquisition function \citep{WeiZhangBeer2020}
  \begin{align*}
   \PVC(\fGP_{\dataset}) \coloneqq \kernel_{\dataset; \measure}(\tilde{\vx}) p(\tilde{\vx})
  \end{align*}
  is a heuristic defined only for the conjugate model.
  Like US, $\PVC$ yields nodes likely to contribute much to the integral variance, but ignores spatial correlations between candidate nodes.
One can also proceed to the opposite direction and use an acquisition function developed for Bayesian quadrature to select evaluation locations for a purpose other than integration~\citep[e.g.,][]{Paul2018,Adachi2023,Adachi2023a,Adachi2024}.

Implementations of sequential active sampling loops and BQ acquisition functions for $n=1$ can for example be found in the Python library \emukit~\citep{Paleyes19, Paleyes2023} which we use in the experiments in Chapter~\ref{sec:experiments}.

Recall from \Cref{sec:conn-appr-theory} that BQ is equivalent to worst-case optimal integration in the RKHS $\rkhs(\kernel)$ of the covariance function.
This and related equivalences allow expressing all acquisition functions above as quantities involving different worst-case errors.
For example, we saw that the MI, IVR and NIV acquisition functions yield nodes that maximize the squared correlation in Eq.\@~\eqref{eq:rho2}
Equivalently, they minimize the integral variance $\Sigma_{\tilde{\dataset}} = \kernel_{\measure\measure} - \vk_{\measure\tilde{\points}}\Trans \mK_{\tilde{\points}\tilde{\points}}^{-1} \vk_{\measure\tilde{\points}}$.
By Proposition~\ref{prop:wce}, for given nodes the integral variance equals the smallest possible squared worst-case error in $\rkhs(\kernel)$ obtainable by a quadrature rule using these nodes.
Nodes $\vx_1, \ldots, \vx_N$ sampled by MI, IVR or NIV are thus worst-case optimal:
\begin{equation} \label{eq:obq-nodes}
  \begin{split}
  \vx_1, \ldots, \vx_N \in \argmin_{\tilde{\vx}_1, \ldots, \tilde{\vx}_N} \, \min_{ \vv \in \R^N} e_\kernel(\vv, \tilde{\points}),
  \end{split}
\end{equation}
where
\begin{equation*}
  e_\kernel(\vv, \tilde{\points}) = \sup_{ \norm[0]{g}_{\rkhs(k)} \leq 1} \, \abs[3]{\int_\domain g(\vx) \dif \measure(\vx) - \sum_{i=1}^N v_i g(\tilde{\vx}_i)} .
\end{equation*}
These nodes provide maximal information for integration in $\rkhs(\kernel)$.
Accordingly, \emph{optimal} Bayesian quadrature is sometimes used to refer to BQ with the nodes in Eq.\@~\eqref{eq:obq-nodes}.
These nodes are called optimal also in numerical analysis and approximation theory, where the corresponding worst-case error is known as the $N$th minimal error; see Chapters~2 and~3 in \citet{NovakWozniakowski2008} and Chapter~5 in \citet{Oettershagen2017}.

\section{Kernel parameter estimation}
\label{sec:param-estimation}

In Bayesian inference one is confronted with the challenging task of selecting an appropriate prior.
While other options exist, in Gaussian process modelling the prevalent approach is to select a parametrised kernel, such as the square exponential in~\eqref{eq:se-kernel} or a Matérn kernel in~\eqref{eq:matern-kernel} with a fixed smoothness parameter $\nu$, and estimate its parameters from the data~\citep[Ch.\@~5]{RasmussenWilliams2006}.
Most Bayesian quadrature methods use maximum likelihood estimation to set the kernel parameters.
Marginalisation and cross-validation are also occasionally used~\citep{Briol2019, Rathinavel2019}.
As our experiments in \Cref{sec:experiments} use maximum likelihood estimation, we briefly review it in the context of conjugate Bayesian quadrature; that is, when $\fGP = \gGP \sim \GP(m, k_\hparams)$ for a kernel with parameters $\hparams$.
As the conjugate Bayesian quadrature integral variance in~\eqref{eq:kernel-var} does not directly depend on the integrand evaluations $\vf_\points$, estimation of the kernel or its parameters from the data is essential if the variance is to measure integration uncertainty.

Let $\hparams$ be the kernel parameters and let the zeroth parameter $\hparam_0 = \sigma^2$ denote the kernel scale.
Let $\dataset$ be the dataset consisting of the nodes $\points = \{ \vx_1, \ldots, \vx_N \}$ and corresponding observations $\vf_{\points}$.
\emph{Maximum likelihood} (ML) estimation is a popular way to fit kernel parameters.
The ML estimate of $\hparams$ is 
\begin{equation*}
  \begin{split}
  \hat{\hparams}_\textup{ML}
  = \argmax_{\hparams} L(\hparams; \dataset)
  &= \argmax_{\hparams} p(\vf_{\points}\mid \points, \hparams),
  \end{split}
\end{equation*}
where $L(\hparams; \dataset)$ is the likelihood of the parameters.
From a probabilistic perspective, the ML estimator maximizes the probability $ p(\vf_{\points}\mid \points, \hparams)$ that the given dataset is observed under the marginalized model $\fGP$. 
Equivalently, the estimator $\hat{\hparams}_{ML}$ maximizes the log-likelihood $\log L(\hparams; \dataset)$.
For the conjugate BQ model the logarithm admits the closed form
\begin{equation}
  \label{eq:hyper-likelihood}
  \begin{split}
  \log L(\hparams; \dataset)
  ={}& -\frac{1}{2}(\vf_{\points} - \vm_{\points})\Trans \mK_{\points\points}^{-1}(\hparams)(\vf_{\points} - \vm_{\points})
  \\
  &-\frac{1}{2}\log{\det(\mK_{\points\points}(\hparams))} -\frac{N}{2}\log{2\pi},
  \end{split}
\end{equation}
where $\det(\mK_{\points\points}(\hparams))$ is the determinant of the Gram matrix $\mK_{\points\points}(\hparams)$ and the dependence on $\hparams$ has been made explicit. Combining the two equations yields
\begin{equation}
  \label{eq:hyper-ml}
  \begin{split}
  \hat{\hparams}_\textup{ML}
  ={}& \argmax_{\hparams} \bigg\{
   -\frac{1}{2}(\vf_{\points} - \vm_{\points})\Trans \mK_{\points\points}^{-1}(\hparams)(\vf_{\points} - \vm_{\points})
  \\
  &\hspace{2cm}-\frac{1}{2}\log{\det(\mK_{\points\points}(\hparams))} \bigg\}.  
  \end{split}
\end{equation}
In general, an analytic solution of Eq.~\eqref{eq:hyper-ml} only exists for the scale parameter
\begin{equation*}
  \hat{\sigma}^2_\textup{ML} = \frac{1}{N}(\vf_{\points} - \vm_{\points})\Trans \mK_{\points\points}^{-1}(\hparams_0)(\vf_{\points} - \vm_{\points}),
\end{equation*}
where $\hparams_0 = (\theta_0 = 1, \theta_1, \ldots)$, so that $\mK_{\points\points}(\hparams_0)$ is the unscaled kernel Gram matrix. 
If all parameters $\hparams$ need to be estimated, Eq.~\eqref{eq:hyper-ml} can be solved by an off-the-shelf optimizer as gradients of $ \log L(\hparams; \dataset)$ with respect to $\hparams$ are usually available.

The kernel parameters $\hparams$ are also called \emph{hyper-parameters} to distinguish them from the (possibly infinite) parameters of the model $\fGP$. Thus, maximizing  $L(\hparams; \dataset)$ is also called  \emph{maximum likelihood type~II} (instead of the standard type~I). Other common descriptors are \emph{maximum marginal likelihood} [since  $ p(\vf_{\points}\mid \points, \hparams)$ describes a marginal probability], or \emph{evidence maximization} [as $p(\vf_{\points}\mid \points, \hparams)$ is the evidence of the probabilistic model].

As $\hat{\hparams}_\textup{ML}$ is an estimator based on a limited dataset $\dataset$ overfitting can occur.
As a consequence, BQ methods with sequential active node selection as in \Cref{alg:bq-seq,alg:active-sampl} that compute the ML-estimate $\hat{\hparams}_\textup{ML}$ at every step anew require initial designs of appropriate size that are independent of the model $\fGP$. This is conceptually analogous to initial designs that are required by sequential active experimental design algorithms and Bayesian optimization.

\section{Numerical ill-conditioning}
\label{sec:num-ill-cond}

In practically all variants of Bayesian quadrature the computation of the posterior $I_\measure(\fGP) \mid \dataset$ and the fitting of the kernel parameters require one to solve---repeatedly in the case of parameter fitting---the linear system $\mK_{\points\points} \vz = \va$ for some vector $\va \in \R^N$.
Unless the kernel or the nodes have properties that can be exploited to simplify this task, one has to use some standard linear solver, the accuracy of which depends on floating-point accuracy and the magnitude of the condition number of the kernel Gram matrix,
\begin{equation*}
  \mathrm{cond}(\mK_{\points\points}) = \frac{\lambda_\textup{max}(\mK_{\points\points})}{\lambda_\textup{min}(\mK_{\points\points})} .
\end{equation*}
Here $\lambda_\textup{max}(\mA)$ and $\lambda_\textup{min}(\mA)$ denote the maximal and minimal eigenvalues of a symmetric positive-definite matrix $\mA$.
When the condition number is large, solving the linear system is prone to numerical errors and hence numerically unstable.
The condition numbers of commonly used kernels can grow fast when the nodes fill the domain.
Let $\domain = [0, 1]^d$ and
\begin{equation*}
  h = \min_{i = 1, \ldots, N } \, \sup_{ \vx \in [0, 1]^d } \, \min_{j \neq i} \norm[0]{\vx - \vx_j}, 
\end{equation*}
which quantifies how well the nodes fill the domain, being the radius of the largest open ball in $[0, 1]^d$ that intersects no subset of $\points$ containing $N-1$ nodes.
This quantity is closely related to the fill-distance that we shall use in \Cref{sec:guarantees}; see Eq.\@~\eqref{eq:fill-distance}.
The following theorem that links $h$ to $\mathrm{cond}(\mK_{\points\points})$ may be found in~\citet{Schaback1995}.
See also \citet[Ch.\@~12]{Wendland2005}, \citet{Bachoc2016}, and \citet{Diederichs2019}.
The general principle is that the smoother the kernel is, the faster the condition number grows as $h$ tends to zero, which is to say that the nodes fill the domain.

\begin{theorem}
  \label{thm:condition-number}
  Let $\domain = [0, 1]^d$.
If $\kernel$ is a Matérn kernel in~\eqref{eq:matern-kernel} with smoothness $\nu > 0$, then
  \begin{equation*}
    \mathrm{cond}(\mK_{\points\points}) \geq C h^{-\nu} ,
  \end{equation*}
  where $C > 0$ depends only on $\sigma$, $\ell$ and $\nu$.
If $\kernel$ is a square exponential kernel in~\eqref{eq:se-kernel}, then there is a constant $c > 0$, which depends only on $\sigma$ and $\ell$, such that
\begin{equation*}
  \mathrm{cond}(\mK_{\points\points}) \geq \exp(- c \, h \log h ).
  \end{equation*}
\end{theorem}

The assumption $\domain = [0, 1]^d$ can be relaxed as described in Remark~\ref{rmk:domain-generalisations}.
Numerical ill-conditioning issues are likewise present in Gaussian process interpolation and kriging, where they are usually dealt with by introducing a parameter $\lambda > 0$ called the \emph{nugget} (alternatively, \emph{regularisation parameter} or \emph{jitter}) and replacing each occurrence of the Gram matrix $\mK_{\points\points}$ with $\mK_{\points\points} + \lambda \Id_N$, where $\Id_N$ is the $N \times N$ identity matrix~(\citealp{Neil1999}; \citealp[Sec.\@~3.4.3]{RasmussenWilliams2006}; \citealp{Ranjan2011, AdrianakisChallenor2012}).
\citet{GramacyLee2012} are strongly in favour of nuggets.
If the kernel is stationary and $\kernel(\vx, \vx) = 1$ for every $\vx \in \domain$, it is straightforward to compute that
\begin{equation*}
  \mathrm{cond}( \mK_{\points\points} + \lambda \Id_N) = \frac{\lambda_\textup{max}(\mK_{\points\points}) + \lambda}{\lambda_\textup{min}(\mK_{\points\points}) + \lambda} \leq \frac{N + \lambda}{\lambda},
\end{equation*}
which means that the condition number may increase at most linearly in the number of nodes.
In the context of conjugate Bayesian quadrature, the use of a nugget corresponds to replacing the mean~$\mu_\dataset$ and variance~$\Sigma_\dataset$ in~\eqref{eq:bq-mean} and~\eqref{eq:bq-var} with their regularised versions
\begin{align}  
  \mu_{\dataset, \lambda} &= m_\measure + I_\measure(\vk_\points)\Trans (\mK_{\points\points} + \lambda \Id_N  )^{-1} ( \vf_\points - \vm_\points), \label{eq:bq-mean-regularised} \\ 
  \Sigma_{\dataset, \lambda} &= \kernel_{\measure\measure} - \vk_{\measure\points}\Trans (\mK_{\points\points} + \lambda \Id_N)^{-1} \vk_{\measure\points}. \nonumber
\end{align}
The nugget can be seen as an additional hyper-parameter and fitted alongside other parameters as described in Section~\ref{sec:param-estimation}.

Theorem~\ref{thm:condition-number} shows that nuggets are necessary when $N$ is sufficiently large.
However, in \Cref{sec:guarantees} we shall see that the nugget has a limiting effect on the accuracy of conjugate Bayesian quadrature, which suggests that its use ought to be avoided when very high accuracy is being sought.
That is, there is a trade-off between numerical stability and accuracy known as \emph{Schaback's uncertainty principle}~\citep{Schaback1995}.\footnote{This principle only applies when one uses the basis $\{\kernel(\cdot, \vx_i)\}_{i=1}^N$ of kernel translates. Numerically more stable bases exist~\citep[Sec.\@~7.1]{MullerSchaback2009, Fasshauer2011}.}
Whether to use a nugget in Bayesian quadrature therefore depends on the modelling choices, the desired accuracy of the particular method, and the required level of code robustness.
Most GP libraries hardcode a minimal \emph{default nugget} that may depend on the type (at the time of writing \gpy~uses $\lambda = 10^{-8}$ while GPtorch uses $\lambda = 10^{-6}$ for floats and $\lambda = 10^{-8}$ for doubles) plus possibly a \emph{dynamic nugget} that depends on the Gram matrix $\mK_{\points\points}$ and is increased successively if solving the linear system fails.
\Cref{sec:exp-code} describes how nuggets are used in our experiments.
\\

\section{Prior mean estimation}
\label{sec:prior-mean-estimation}

The prevalent approach in the literature on BQ is to use a zero-mean prior $\gGP \sim \GP(0, \kernel)$.
However, like the kernel or its parameters, also the prior mean $\mean$ can be estimated from the data.
Suppose that we are in the conjugate setting described in Section~\ref{sec:conj-bq} where $\fGP = \gGP$.
The prior mean is typically estimated by postulating a parametric form
\begin{equation*}
  m(\vx) = \sum_{i=1}^q \beta_i \, g_i(\vx)
\end{equation*}
for $q \in \N$ fixed basis functions $g_i \colon \domain \to \R$ and coefficients $\vbeta = (\beta_1, \ldots, \beta_q) \in \R^q$ which are assigned some prior.\footnote{The approach that we describe is closely related to ordinary and universal kriging~\citep[Ch.\@~3]{Cressie1993}.}
A common approach is to place a non-informative prior $p(\vbeta, \sigma^2) \propto \sigma^{-2}$ on $\vbeta$ and the kernel scale parameter $\sigma^2$ and subsequently marginalize out these parameters~\citep{OHagan1991}.
Doing so yields an integral posterior which is Student's $t$ with $N - q$ degrees of freedom.
The mean $\tilde{\mu}_\dataset$ and variance $\tilde{\Sigma}_\dataset$ of the posterior have closed form expressions similar to those of $\mu_\dataset$ and $\Sigma_\dataset$ in Section~\ref{sec:conj-bq}, but they also involve evaluations of the basis functions $g_i$ at the nodes and their integrals~\citep[Sec.\@~2.2]{OHagan1991}.

When only an unknown constant trend is postulated (i.e., $m \equiv \beta_1$) the posterior mean is $\tilde{\mu}_\dataset = \vw_{\mathbf{1}}\Trans \vf_\points$, where
\begin{equation}
  \label{eq:normalised-bq-mean}
  \vw_{\mathbf{1}} = \big( \mI_N -  \frac{\mK_{\points\points}^{-1} \mathbf{1}_N \mathbf{1}_N\Trans}{\mathbf{1}_N\Trans \mK_{\points\points}^{-1} \mathbf{1}_N } \Big) \mK_{\points\points}^{-1} \vk_{\measure\points} + \frac{\mK_{\points\points}^{-1} \mathbf{1}_N\Trans}{\mathbf{1}_N\Trans \mK_{\points\points}^{-1} \mathbf{1}_N }
\end{equation}
and $\mathbf{1}_N$ is the $N$-vector of ones, and one obtains \emph{normalized} Bayesian quadrature.
The terminology is motivated by the weights $\vw_{\mathbf{1}}$ in~\eqref{eq:normalised-bq-mean} summing up to one, which does not generally happen in conjugate Bayesian quadrature with a fixed prior mean.
Weight normalization is particularly useful as a precaution against poorly set length-scale in high dimensions.
If the length-scale $\ell$ of a stationary kernel, such as square exponential in~\eqref{eq:se-kernel} or Matérn in~\eqref{eq:matern-kernel}, is very small relative to the typical distance $\norm[0]{\vx_i - \vx_j}$ of the nodes, the covariance matrix is almost diagonal and $\vw_{\mathbf{1}} \approx (\frac{1}{N}, \ldots, \frac{1}{N})$.
In contrast, the unnormalized weights in~\eqref{eq:conj-bq-mean-var} are very close to zero in this case, which means that the posterior is close to the prior.

Prior mean estimation can be alternatively achieved via the use of a flat prior limit~(\citealp{Karvonen2018}; we have already seen this technique in action in Proposition~\ref{prop:bayes-sard}) or maximum likelihood estimation~\citep{Rathinavel2019}.
When maximum likelihood estimation is used to set the covariance scale parameter $\sigma^2$ (see Section~\ref{sec:param-estimation}), the resulting posteriors, except for being Gaussians, differ little from the Student's $t$ posterior described above; many relevant posteriors are collected in \citet[Ch.\@~4]{Santner2003}.
In addition to the works already mentioned, prior mean estimation or normalized Bayesian cubature has been used in \citet{OHagan1992, Kennedy1998, PronzatoZhigljavsky2020}.
We also refer to \citet{Bezhaev1991} and \citet{DeVore2019} for non-probabilistic versions, whose relation to Bayesian quadrature has been described in Section~\ref{sec:conn-appr-theory}, of prior mean estimation.
Finally, we note that the theoretical guarantees available in the literature usually assume that the prior mean is fixed.

\section{Main limitations}
\label{sec:limitations}

Several factors that limit the applicability of Bayesian quadrature emerge from the practical issues reviewed in \Cref{sec:kern-mean-var,sec:sampling,sec:param-estimation,sec:num-ill-cond,sec:prior-mean-estimation}. The main limitations that have restricted Bayesian quadrature to a relatively narrow range of applications are (i) the need to have access to kernel means and variances and (ii) the cubic computational complexity in the number of nodes associated to Gaussian process methods. Mainly due to these two challenges a typical application of Bayesian quadrature features a small number of evaluations of a low-dimensional integrand and a simple integration measure.

\subsection{Kernel means and variances}

That the kernel mean $\kernel_\measure$ and variance $\kernel_{\measure\measure}$ are available in closed forms for a limited collection of kernel-measure pairs is a severe limitation.
Most Bayesian quadrature applications involve integration with respect to the uniform or Gaussian measure, and sometimes Bayesian quadrature is interpreted in a narrow sense to mean an algorithm or a class of algorithms for integration with respect to Gaussian measures.
In this respect Bayesian quadrature resembles Gaussian quadratures and other classical approaches that are designed for integration with respect to particularly tractable measures.

While \citet{Oates2017a} have used Dirichlet process mixtures to extend Bayesian quadrature for intractable integration distributions that can only be sampled from, their approach has not caught on and is bound to sacrifice many desirable properties of Bayesian quadrature, such as a straightforward theoretical guarantees.
An approach that is in principle general and powerful is to use the measure to design a kernel for which $\kernel_\measure$ and $\kernel_{\measure\measure}$ are easy to compute.
\citet{Oates2017b} and \citet{South2021} use Stein's method to design such kernels.
However, this simply transfers the trouble from computing integrals to computing a kernel and results in a model that is not easy to interpret.
We believe that no satisfactory solution can exist that would obviate the need to compute kernel means or other typically intractable integrals.
To construct a quadrature rule that encodes some properties of the integrand one needs access to some functions that effect the encoding.
When $N$ is small (e.g., the integrand is a computationally expensive simulation), it may in some cases make sense to approximate $\kernel_\measure$ and $\kernel_{\measure\measure}$ with, for example, Monte Carlo, if doing so has negligible computational cost in comparison to the integrand evaluations.

\subsection{Cubic computational complexity}

Computational cost is the second major limitation.
As discussed in \Cref{sec:num-ill-cond}, to implement Bayesian quadrature one has to solve linear systems defined by the typically dense $N \times N$ kernel matrix $\mK_{\points\points}$.
The cubic computational complexity in the number of datapoints that this gives rise to is a well-known computational bottleneck in all Gaussian process methods.
Over the past 20 years much effort has gone into the development of computationally scalable Gaussian process methods~\citep[for a recent review, see][]{LiuOngShen2020}.
These methods are predominantly \emph{approximate} in the sense of the definition given in \Cref{sec:tax-inf} and developed under the assumption that the data are noisy, so that the linear systems that need to be solved are defined by the regularised kernel matrix $\mK_{\points\points} + \lambda \Id_N$ rather than $\mK_{\points\points}$.
Because the nugget $\lambda$ is typically either zero or very small in Bayesian quadrature, many scalable Gaussian process methods cannot be imported to Bayesian quadrature as such.
For example, methods based on low-rank approximations of the kernel matrix require a positive nugget.
It is nevertheless surprising that no systematic study is available on the feasibility of applying existing approaches to scalable Gaussian process regression to Bayesian quadrature.

There exist a number of scalable approaches specifically to Bayesian quadrature (see the \emph{scalable} label in \Cref{sec:tax-inf}).
While some have been highly successful, all are tied to certain measure-kernel pairs and do not generalise easily or at all.
Except for some special cases~\citep[e.g.,][]{Rathinavel2019}, Bayesian quadrature is presently not a serious competitor to Monte Carlo, quasi-Monte Carlo, and related approaches to numerical integration in high dimensions or in other situations in which a large number of integrand evaluations are necessary.\footnote{This is not because Bayesian quadrature converges slower than the alternatives but because constants hidden by asymptotic notation tend to grow with dimension and to implement Bayesian quadrature one must solve a linear system. Note that, despites its classical dimension-independent convergence rate $\mathcal{O}(N^{-1/2})$, even Monte Carlo suffers from the curse of dimensionality in this sense~\citep{Tang2024}.}
We direct the reader to \citet[Sec.\@~42.3]{pnbook22} for an assessment of computational challenges that is in alignment with ours.

  Because they incorporate prior information about the object being modelled, Bayesian methods can be expected to perform better than other methods when the amount of data is limited, a setting in which solving linear systems is not a computational problem.
  This has resulted in a marked emphasis on applications in which integrand evaluations are expensive~\citep{Kennedy2000, Briol2019}.
  For example, \citet[Sec.\@~1.2]{Briol2019} ``focus on numerical integrals where the cost of evaluating the integrand forms a computational bottleneck, and for which numerical error is hence more likely to be non-negligible.''
  Model selection naturally plays all the more important a role in the low-data regime in which non-conjugate Bayesian quadrature methods that would become intractable even for moderately large datasets may shine.

\chapter{Theoretical guarantees}
\label{sec:guarantees}

This chapter contains a number of mathematical results that provide theoretical guarantees on the accuracy of conjugate Bayesian quadrature under certain assumptions on the covariance function $\kernel$, the integrand $\f$ and the coverage of the domain~$\domain$ by the nodes.
We shall work with a sequence $(\points_N)_{N=1}^\infty$ of sets $\points_N = \{ \vx_1, \ldots, \vx_N\} \subset D$ consisting of $N$ pairwise distinct nodes.
In spite of our notation, the node sets need not be nested (i.e., despite both being denoted $\vx_1$, the first nodes in $\points_N$ and $\points_{N+1}$ need not be equal).
To further simplify notation, we often replace subscripted $\points_N$ and $\dataset$ with $N$ to emphasise the dependency on the number of nodes.
For example, under this convention $\mu_N$ stands for $\mu_\dataset$ in~\eqref{eq:bq-mean} with $X = X_N$ and $\mu_{N,\lambda}$ for its regularized version in~\eqref{eq:bq-mean-regularised}.
Most theoretical guarantees we give state that
\begin{equation} \label{eq:generic-guarantee}
  \abs[0]{ I_P(\f) - \mu_{N, \lambda} } \leq C( r_N + \sqrt{\lambda} \, ) 
\end{equation}
for a positive constant $C$ and a decreasing positive sequence $(r_N)_{N=1}^\infty$ which tends to zero.
The rate with which $r_N$ tends to zero reflects the smoothness of $\kernel$ and $\f$.
A large collection of theoretical guarantees for Bayesian quadrature may be found in references such as \citet{Kanagawa2016, Briol2019, Karvonen2019a, Kanagawa2019, Karvonen2020b, Kanagawa2020}.
Proofs are relegated to \Cref{sec:proofs}, which also contains a more detailed comparison of our results to those in the literature.

We prioritize simplicity and ease of presentation over generality.
Accordingly, throughout this section we assume that the prior mean $\mean$ is identically zero, an assumption that could be relaxed to $\mean$ being an element of the RKHS of $\kernel$.
Furthermore, in most results of this section we make use of the following assumption on the domain and the integration distribution.

\begin{assumption}
  \label{assumption:theory}
  Assume that (i) $\domain = [0, 1]^d \subset \R^d$ and (ii) the probability measure $\measure$ has a density function $p$ such that $p_\textup{max} = \sup_{\vx \in [0, 1]^d} p(\vx) < \infty$.
\end{assumption}

\begin{remark} 
\label{rmk:domain-generalisations}
Basic translation and scaling arguments can be used to allow for any hyperrectangular domain $[a_1, b_1] \times \cdots \times [a_d, b_d]$.
More generally, the results below hold when $\domain$ satisfies an \emph{interior cone condition} and has a \emph{Lipschitz boundary}.
See \citet[Sec.\@~3]{Kanagawa2020} or \citet[App.\@~A]{Wynne2021} for details in the context of Gaussian processes.
\end{remark}

\section{Guarantees for generic nodes}
\label{sec:conv-rates}

We first provide guarantees for \emph{generic} nodes, by which we mean that no reference is made to the process by which the nodes are obtained and upper bounds of the form~\eqref{eq:generic-guarantee} are given in terms of the computable \emph{fill-distance} $h_\points$ and \emph{separation radius} $q_\points$, which are defined as
\begin{equation}
  \label{eq:fill-distance}
  h_{\points} = \sup_{ \vx \in [0, 1]^d } \, \min_{i = 1,\ldots, N} \norm[0]{ \vx - \vx_i} \quad \text{ and } \quad q_\points = \frac{1}{2} \min_{ i \neq j } \norm[0]{ \vx_i - \vx_j }
\end{equation}
for the domain $\domain = [0, 1]^d$.
Fill-distance equals the radius of the largest ball in $[0, 1]^d$ that does not contain any of the points in $\points$ while the separation radius is half the smallest distance between any two points in $\points$.
Because the fill-distance and separation radius are, at least in principle, computable for any set of nodes, error bounds expressed in terms of these quantities apply to all relevant Bayesian quadrature methods regardless of how the nodes have been sampled (\Cref{sec:tax-sampling}).
Choices made in the sampling axis of taxonomy affect the bounds through fill-distance and separation radius.

If, for a sequence of node sets $(\points_N)_{N=1}^\infty$, the fill-distance and separation radius are asymptotically comparable as $N \to \infty$, the point design is said to be \emph{quasi-uniform}.
This means that there exists a constant $c_\textup{qu} > 0$ such that
\begin{equation*}
  q_N \leq h_N \leq c_\textup{qu} \, q_N
\end{equation*}
for every $N \geq 1$.
This implies that there are positive constants $c_1$ and $c_2$ such that
\begin{equation} \label{eq:quasi-uniformity-rate}
  c_1 N^{-1/d} \leq h_N \leq c_2 N^{-1/d} \quad \text{ and } \quad c_1 N^{-1/d} \leq q_N \leq c_2 N^{-1/d}
\end{equation}
for every $N \geq 1$~\citep[Sec.\@~14.1]{Wendland2005}.
Regular grids are quasi-uniform.
For example, the equispaced nodes $X_N = \{1/N, 2/N, \ldots, (N-1)/N, 1\}$ on $\domain = [0, 1]$ have $q_N = h_N = 1/N$ for every $N \geq 1$.
How the fill-distance and certain related quantities behave for random points is discussed in \citet{KriegSonnleitner2022}.

We begin with two generic propositions that guarantee the convergence of Bayesian quadrature under very general conditions.
Recall from Section~\ref{sec:integration-in-rkhs} that $\rkhs(\kernel)$ is the RKHS of $\kernel$ on the domain $\domain$.
The RKHS consists of certain functions defined on $\domain$.

\begin{proposition}
  \label{prop:generic-guarantee-1}
  Suppose that Assumption~\ref{assumption:theory} holds, the function $k(\cdot, \vx)$ is continuous for every $\vx \in \domain$, and $\sup_{\vx \in \domain} k(\vx, \vx) < \infty$.
  If $f \in \rkhs(\kernel)$, the node sets are nested (i.e., $\points_N \subset \points_{N+1}$) and $h_N \to 0$ as $N \to \infty$, then
  \begin{equation*}
    \abs[0]{ I_P(f) - \mu_{N} } \to 0 \quad \text{ as } \quad N \to \infty.
  \end{equation*}  
\end{proposition}
\begin{proof}
  See \Cref{sec:proofs-conv-rates}.
\end{proof}

Proposition~\ref{prop:generic-guarantee-1} guarantees that the integral mean of Bayesian quadrature tends to the true integral if the point design is space-filling in the sense that $h_N \to 0$.
However, the proposition gives us no inkling on the rate of this convergence.
The next proposition ensures that \emph{at least} the familiar square-root rate of Monte Carlo can be attained.
That is, conjugate Bayesian quadrature can always beat Monte Carlo.

\begin{proposition}
  \label{prop:generic-guarantee-2}
  Suppose that Assumption~\ref{assumption:theory} holds and the nodes are independent samples from $\measure$.
  If $f \in \rkhs(\kernel)$, then there is $C > 0$, which depends only on $k$ and $P$, such that
  \begin{equation} \label{eq:BQ-square-root-rate}
    \mathbb{E}\big[ \abs[0]{ I_P(f) - \mu_{N} } \big] \leq C N^{-1/2}
  \end{equation}
  for every $N \geq 1$, where the expectation is with respect to the random nodes.
\end{proposition}
\begin{proof}
  See \Cref{sec:proofs-conv-rates}.
\end{proof}

Propositions~\ref{prop:generic-guarantee-1} and~\ref{prop:generic-guarantee-2} provide slow rates because they make few assumptions about the integrand, covariance function, and the nodes.
The rest of this section shows how smoothness assumptions on $\kernel$ and $\f$ improve the square-root rate in~\eqref{eq:BQ-square-root-rate}.

\subsection{Isotropic Matérn kernels}
\label{sec:guarantees-isotropic-matern}

The \emph{Sobolev space} $H^\alpha([0, 1]^d)$ of order $\alpha \in \N$ is a Hilbert space that consist of those functions $g$ for which the \emph{weak derivative} $\mathrm{D}^{\vn} g$ exists and is square-integrable for every non-negative multi-index such that $\abs[0]{\vn} = n_1 + \cdots + n_d \leq \alpha$.
That is, $g \in H^\alpha([0, 1]^d)$ if and only if
\begin{equation} \label{eq:sobolev-norm-weak-derivative}
  \bigg( \sum_{ \abs[0]{\vn} \leq \alpha} \int_{[0, 1]^d} [ \mathrm{D}^{\vn} g(\vx) ]^2 \dif \vx \bigg)^{1/2} < \infty.
\end{equation}
We use the Fourier transform to define the norm $\lVert \cdot \rVert_{H^\alpha([0,1]^d)}$ of $H^\alpha([0,1]^d)$; see~\eqref{eq:sobolev-norm-fourier} in \Cref{sec:sobolev-spaces}.
The resulting norm is equivalent to the norm defined by~\eqref{eq:sobolev-norm-weak-derivative}.
Any function whose classical derivatives up to order $\alpha$ exist and are square-integrable on $[0, 1]^d$ is an element of $H^\alpha([0, 1]^d)$.
Conversely, if $\alpha > r + d/2$ for $r \in \N$, then all functions in $H^\alpha([0, 1]^d)$ are $r$ times differentiable in the classical sense.
In particular, all functions in $H^\alpha([0, 1]^d)$ are continuous if $\alpha > d/2$.
Sobolev spaces $H^\alpha([0, 1]^d)$ of \emph{fractional smoothness} $\alpha > 0$ may be easily defined via Fourier transforms; details are again given in Section~\ref{sec:sobolev-spaces}.
Theorems~\ref{thm:matern-generic} and~\ref{thm:matern-quasi-uniform} provide theoretical guarantees for isotropic Matérn kernels based on the fill-distance.
For the purposes of these theorems, define the \emph{mesh ratio}
\begin{equation} \label{eq:mesh-ratio}
  \rho_\points = \frac{h_\points}{q_\points} \geq 1.
\end{equation}

\begin{theorem}[Generic nodes for Matérns]
  \label{thm:matern-generic}
  Suppose that Assumption~\ref{assumption:theory} holds and let $\kernel$ be a Matérn kernel~\eqref{eq:matern-kernel} with smoothness parameter $\nu > 0$.
  Let $\alpha = \nu + d/2 \geq \beta > d/2$.
  If $\f \in H^{\beta}([0, 1]^d)$ and $\lambda \geq 0$, then there is a positive constant $C$, which does not depend on $f$, $X$, $\lambda$, $\ell$ or $\sigma$, such that
  \begin{equation} \label{eq:matern-generic-bound}
    \begin{split}
    \lvert I_P(\f&) - \mu_{\dataset, \lambda} \rvert \\
    &\leq C\Big( c_{1,\ell,\nu} h_X^{\beta} \rho_X^{\alpha-\beta} + \frac{c_{2,\ell,\nu}}{\sigma \ell^{d/2}} \sqrt{\lambda} \, h_X^{d/2} q_X^{-(\alpha-\beta)} \Big) \norm[0]{\f}_{H^\beta([0,1]^d)},
    \end{split}
  \end{equation}
  where $c_{1, \ell, \nu} = \max\{ \ell / \sqrt{2\nu}, \sqrt{2\nu} / \ell \}$ and $c_{2,\ell,\nu} = \max\{1, \ell/\sqrt{2\nu}\}$. 
\end{theorem}
\begin{proof}
  See \Cref{sec:proofs-isotropic-matern}.
\end{proof}

Theorem~\ref{thm:matern-generic} provides an error estimate for any given set of nodes $\points$ in terms of $h_\points$, $q_\points$, and $\rho_\points$.
Using~\eqref{eq:quasi-uniformity-rate} to control the fill-distance and separation radius and making sure that constants that depend on kernel parameters do not blow up yields the following variant of Theorem~\ref{thm:matern-generic}, which applies to a \emph{sequence} of node sets $(\points_N)_{N=1}^\infty$ produced by a hypothetical Bayesian quadrature method that is run indefinitely without a stopping criterion.

\begin{theorem}[Quasi-uniform nodes for Matérns]
  \label{thm:matern-quasi-uniform}
  Suppose that Assumption~\ref{assumption:theory} holds and let $(X_N)_{N=1}^\infty$ be a sequence of quasi-uniform node sets in $[0, 1]^d$.
  Let $\kernel$ be a Matérn kernel~\eqref{eq:matern-kernel} with fixed smoothness parameter $\nu > 0$ and scale and length-scale parameters $\sigma_N$ and $\ell_N$ which may vary with $N$ but satisfy $\sigma_N \geq \sigma_\textup{min}$ and $\ell_\textup{min} \leq \ell_N \leq \ell_\textup{max}$ for some positive $\sigma_\textup{min}$, $\ell_\textup{min}$ and $\ell_\textup{max}$.
  Let $\alpha = \nu + d/2 \geq \beta > d/2$.
  If $\f \in H^{\beta}([0, 1]^d)$ and $\lambda = \lambda_N \geq 0$, then there is a positive constant $C$, which does not depend on $f$, $N$ or $\lambda_N$, such that
  \begin{equation} \label{eq:matern-quasi-uniform-bound}
    \begin{split}
    \lvert I_P(\f) - {}& \mu_{N, \lambda_N} \rvert \\
    & \quad\leq C\big( N^{-\beta/d} + \sqrt{\lambda_N} N^{-1/2} N^{(\alpha-\beta)/d} \big) \norm[0]{\f}_{H^\beta([0,1]^d)}
    \end{split}
  \end{equation}
  for every $N \geq 1$.
\end{theorem}
\begin{proof}
It follows from~\eqref{eq:quasi-uniformity-rate} that both $h_N$ and $q_N$ are of order $N^{-1/d}$ and that $\rho_N \geq 1$ is bounded from above.
The claim follows by inserting these rates and bounds in Theorem~\ref{thm:matern-generic} and using the assumed bounds on the scale and length-scale parameters to bound the constants $c_{1, \ell, \nu}$ and $c_{2,\ell,\nu} / (\sigma \ell^{d/2})$ and incorporate them into $C$.
\end{proof}

Theorems~\ref{thm:matern-generic} and~\ref{thm:matern-quasi-uniform} comprise two somewhat different cases.
\emph{Firstly}, if $\beta = \alpha$, then $f$ is an element of the RKHS of $\kernel$ (see Section~\ref{sec:integration-in-rkhs}).
In this case the terms $\smash{\rho_X^{\alpha-\beta}}$ and $\smash{q_X^{-(\alpha-\beta)}}$ in~\eqref{eq:matern-generic-bound} and $\smash{N^{(\alpha-\beta)/d}}$ in~\eqref{eq:matern-quasi-uniform-bound}, which slow down convergence, vanish and the bounds depend on the nodes only via the fill-distance.
The theorems do not yield improved rates of convergence if $f \in H^\gamma([0, 1]^d)$ for $\gamma > \alpha$ (i.e., when the model \emph{undersmooths}).
Since $H^\gamma([0, 1]^d) \subset H^\alpha([0, 1]^d)$, in this case one simply applies the theorems with $\beta = \alpha$.\footnote{However, bounds of order $h_\points^{2\alpha}$ or $N^{-2\alpha}$ are available if $f$ has smoothness $2\alpha$ and satisfies certain boundary conditions~\citep{Schaback1999, Schaback2018, TuoWangWu2020, Karvonen2020b, SloanKaarnioja2023, KarvonenSantinWenzel2025}. This case is known as \emph{superconvergence} or \emph{improved rate} of kernel-based interpolation.}
\emph{Secondly}, if $\beta < \alpha$ and $f \notin H^\alpha([0, 1]^d)$, then $\f$ is not in $\mathcal{H}(\kernel)$, so that the model \emph{oversmooths}.
The bound~\eqref{eq:rkhs-error-bound} links integration error to the BQ variance but holds only for functions in the RKHS. This gives rise to the distinction between the cases $\f \in \mathcal{H}(\kernel)$ and $\f \notin \mathcal{H}(k)$.

When $\lambda_N = 0$, Theorem~\ref{thm:matern-quasi-uniform} gives the rate of convergence $N^{-\beta/d}$ that is worst-case optimal in the Sobolev space $H^\beta([0,1]^d)$~\citep{NovakTriebel2006}.
Bayesian quadrature therefore \emph{adapts to misspecified smoothness}, being capable of attaining the optimal rate of convergence even when $\alpha > \beta$. This is convenient, as eliciting the smoothness of $\f$ tends to be difficult.
As evidenced by the term $N^{-1/2 + (\alpha-\beta)/d}$, for a positive nugget the rate given by Theorem~\ref{thm:matern-quasi-uniform} depends on the degree of smoothness misspecification, $\alpha - \beta \geq 0$.
In particular, if $\alpha \geq \beta + d/2$, Theorem~\ref{thm:matern-quasi-uniform} does not guarantee convergence unless $\lambda_N$ tends to zero sufficiently fast.
This may be due to sub-optimality of some estimates the proof uses.

For a machine learning practitioner or a statistician, randomly sampled nodes are undoubtedly the most relevant ones.
The following theorem, which has not appeared in the context of Bayesian quadrature before, guarantees that for Matérns the expected error decays with the worst-case optimal rate.

\begin{theorem}[Random nodes for Matérns] 
\label{thm:matern-random}
  Suppose that $\measure$ is uniform on $\domain = [0, 1]^d$.
  Let $\kernel$ be a Matérn kernel~\eqref{eq:matern-kernel} with fixed smoothness parameter $\nu > 0$ and scale and length-scale parameters $\sigma_N$ and $\ell_N$ which may vary with $N$ but satisfy $\sigma_N \geq \sigma_\textup{min}$ and $\ell_\textup{min} \leq \ell_N \leq \ell_\textup{max}$ for some positive $\sigma_\textup{min}$, $\ell_\textup{min}$ and $\ell_\textup{max}$.
  If $\f \in H^\alpha([0, 1]^d)$ for $\alpha = \nu + d/2$ and the nodes are sampled independently from $\measure$, then there is a positive constant $C$, which does not depend on $f$ or $N$, such that
  \begin{equation*}
    \mathbb{E}\big[ \abs[0]{ I_P(\f) - \mu_{N} } \big] \leq C N^{-\alpha / d}
  \end{equation*}
  for every $N \geq 1$. The expectation is with respect to the random nodes.
\end{theorem}
\begin{proof}
  See \Cref{sec:proofs-isotropic-matern}.
\end{proof}

\subsection{Product Matérn kernels}

For a product Matérn kernel in~\eqref{eq:product-matern-kernel}, theoretical guarantees can be derived from the results in \citet{Teckentrup2020} if the nodes form a sparse grid on $\domain = [0, 1]^d$ and the integrand is in a \emph{mixed-order Sobolev space}~\citep[see also][]{RiegerWendland2017, Nobile2018, RiegerWendland2020}.
For $\alpha \in \N$, the mixed-order Sobolev space $H^{\alpha}_\textup{mix}([0, 1]^d)$ is the tensor product of univariate Sobolev spaces $H^\alpha([0,1])$ and consists of $g \in L^2([0, 1]^d)$ whose weak derivative $\mathrm{D}^{\vn} g$ exists and is in $L^2([0,1]^d)$ for every multi-index $\vn = (n_1, \ldots, n_d)$ such that $n_i \leq \alpha$ for every $i$.
If $\kernel$ is a product Matérn kernel~\eqref{eq:product-matern-kernel} with smoothness $\nu > 0$ such that $\alpha = \nu + 1/2$ is a positive integer $f \in H^{\beta}_\textup{mix}([0, 1]^d)$ for an integer $\beta$ such that $1 \leq \beta \leq \alpha$, it follows from \citet[Thm.\@~3.11]{Teckentrup2020} that there is $C > 0$, which does not depend on $f$ or $N$, such that
\begin{equation*}
  \abs[0]{ I_\measure(\f) - \mu_{N} } \leq C N^{-\beta} (\log N)^{(\beta + 1)(d-1)} \norm[0]{f}_{H^\beta_\textup{mix}([0,1]^d)}
\end{equation*}
when the nodes form a certain sparse grid (see Figure~\ref{fig:deterministic-sampling-designs} for an example).
Due to the requirement that the nodes form a sparse grid, this result is much more restrictive than Theorems~\ref{thm:matern-generic} to~\ref{thm:matern-random}.
We do not go into further details so as to avoid having to review the construction sparse grids.
In general, error analysis for product Matérns is much less mature than for isotropic Matérns.

\subsection{Square exponential kernel}
\label{sec:guarantees-se-kernel}

Here we provide theoretical guarantees for conjugate Bayesian quadrature based on the square exponential kernel in~\eqref{eq:se-kernel}.
Unlike those for Matérn kernels, these guarantees apply only to functions in the RKHS of the kernel, $\rkhs(\kernel)$.
One can show~\citep{Minh2010} that the RKHS on any $\domain \subseteq \R^d$ with non-empty interior consists precisely of the functions
\begin{equation*}
  g(\vx) = \exp\Big( \! - \frac{\norm[0]{\vx}^2}{2 \ell^2} \Big) \sum_{ \vn \in \N_0^d } c_{\vn} \vx^\vn \quad \text{ such that } \quad \sum_{\vn \in \N_0^d} \ell^{2\abs[0]{\vn}} \vn! c_\vn^2 < \infty,
\end{equation*}
where the sums are over all non-negative $d$-dimensional multi-indices $\vn = (n_1, \ldots, n_d)$ and we denote $\vx^\vn = x_1^{n_1} \cdots x_d^{n_d}$ and $\vn! = n_1 ! \cdots n_d!$.
The RKHS is minuscule: it consists of infinitely differentiable functions but contains no polynomials except the zero function.
The current state of the theory is inadequate to provide theoretical guarantees for $\f \notin \rkhs(\kernel)$.

For Matérns, the length-scale $\ell$ only affects the norm of the RKHS, not which functions are contained therein; see~\eqref{eq:norm-equivalence} and~\eqref{eq:matern-norm-constants}.
However, for the square exponential kernel $\ell$ does affect the size of the RKHS and thus plays a role similar to the smoothness parameter $\nu$ in Mat\'ern kernels.
We therefore require that $\ell$ be kept fixed in the results below.

\begin{theorem}[Generic nodes for SE]
\label{thm:se-generic}
  Suppose that Assumption~\ref{assumption:theory} holds and let $\kernel$ be a square exponential kernel~\eqref{eq:matern-kernel} with fixed scale $\sigma > 0$ and length-scale $\ell > 0$.
  If $f \in \rkhs(\kernel)$ and $\lambda \geq 0$, then there are positive $c_1$ and $c_2$, which do not depend on $\f$, $X$ or $\lambda$, such that
  \begin{equation*}
    \lvert I_\measure(\f) - \mu_{\dataset, \lambda} \rvert \leq \big( 2 \operatorname{exp} ( c_1 h_X^{-1} \log h_X ) + \sqrt{\lambda} \operatorname{exp}(c_2 h_X^{-1} ) \big) \norm[0]{\f}_{\rkhs(\kernel)}.
  \end{equation*}
\end{theorem}
\begin{proof}
See \Cref{sec:proofs-guarantees-se-kernel}.
\end{proof}

Like Theorem~\ref{thm:matern-generic}, the above theorem applies to any given set of nodes.
By instead considering a quasi-uniform sequence of node sets generated by some sampler we obtain an explicit rate of convergence given by the following theorem.

\begin{theorem}[Quasi-uniform nodes for SE]
\label{thm:se-quasi-uniform}
  Suppose that Assumption~\ref{assumption:theory} holds and let $\kernel$ be a square exponential kernel~\eqref{eq:matern-kernel} with fixed scale $\sigma > 0$ and length-scale $\ell > 0$.
  If $f \in \rkhs(\kernel)$ and $\lambda = \lambda_N \geq 0$, then there are positive $c_1$ and $c_2$, which do not depend on $\f$, $N$ or $\lambda_N$, such that
  \begin{equation*}
    \abs[0]{ I_\measure(\f) - \mu_{\dataset, \lambda} } \leq \big( 2 e^{-c_1 N^{1/d} \log N} + \sqrt{\lambda_N} e^{c_2 N^{1/d}} \big) \norm[0]{\f}_{\rkhs(\kernel)}
  \end{equation*}
  for every $N \geq 1$.
\end{theorem}
\begin{proof}
  The claim follows from Theorem~\ref{thm:se-generic} and~\eqref{eq:quasi-uniformity-rate}.
\end{proof}
Note that the terms dependent on the nugget in Theorems~\ref{thm:se-generic} and~\ref{thm:se-quasi-uniform} blow up very fast as $h_X \to 0$ or $N \to \infty$.
These theorems are useful only when the nugget is close to zero (or tends to zero sufficiently fast).
When $\lambda_N = 0$, the convergence in Theorem~\ref{thm:se-quasi-uniform} is super-exponential and hence much faster than in Theorems~\ref{thm:matern-quasi-uniform} to~\ref{thm:matern-random}, which yield polynomial rates for Matérns.
This is because the class of functions these theorems apply to, the RKHS of the square exponential kernel, is significantly smaller than any Sobolev space.
Theorems~\ref{thm:se-generic} and~\ref{thm:se-quasi-uniform} do not therefore demonstrate that the square exponential kernel is in any way superior to a Matérn (or any other) kernel unless one is certain that $f$ is in $\rkhs(\kernel)$.
We note that on dimension one and for product designs in higher dimensions it is  possible to obtain slightly improved rates with explicit constants $c_1$ and $c_2$ by using results in \citet{Yarotsky2013, Karvonen2022-power-series}.

\section{Guarantees for sequential Bayesian quadrature}
\label{sec:active-bq-guarantees}

Consider a sequential setting where the nodes are obtained by sequential maximization of an acquisition function that is a strictly increasing transformation $\psi \colon [0, \infty) \to \R$ of the scalar squared correction in~\eqref{eq:rho2}.
That is, for each $n \geq 1$, the $(n+1)$th node $\vx_{n+1}$ satisfies
\begin{equation} \label{eq:sequential-selection-for-guarantees}
  \begin{split}
  \vx_{n+1} &\in \argmax_{\tilde{\vx}} \psi(\rho^2(\tilde{\vx})) \\
  &= \argmax_{\tilde{\vx}} \psi\big( \vk_{\measure\points(\tilde{\vx})}\Trans \mK_{\points(\tilde{\vx})\points(\tilde{\vx})}^{-1}\vk_{\measure\points(\tilde{\vx})} \big),
  \end{split}
\end{equation}
where $\points(\tilde{\vx}) = X \cup \{ \tilde{\vx} \} = \{\vx_1, \ldots, \vx_n, \tilde{\vx}\}$.
Examples include the MI, IVR, and NIV acquisition functions from \Cref{sec:sampling-active}.
The following theorem demonstrates that sequential conjugate Bayesian quadrature converges at least as fast as Monte Carlo.

\begin{theorem}[Sequential nodes]
  \label{thm:active-bq-guarantee}
  Suppose that Assumption~\ref{assumption:theory} holds, the function $k(\cdot, \vx)$ is continuous for every $\vx \in \domain$, and $\sup_{\vx \in \domain} k(\vx, \vx) < \infty$.
  If $f \in \rkhs(\kernel)$ and $N$ nodes are selected sequentially as in~\eqref{eq:sequential-selection-for-guarantees}, then there is a positive constant $C$, which depends only on $\domain$, $\measure$ and $\kernel$, such that 
    \begin{equation}
      \label{eq:sequential-bq-rate}
    \abs[0]{ I_\measure(\f) - \mu_N} \leq C N^{-1/2} \norm[0]{f}_{\rkhs(\kernel)} .
  \end{equation}
\end{theorem}
\begin{proof}
See \Cref{sec:proofs-active-bq-guarantees}.
\end{proof}

The rate~\eqref{eq:sequential-bq-rate} is likely conservative~\citep[Secs.\@~5 \&~7]{Santin2021}.
See \citet{Kanagawa2019, Wenzel2021, Wenzel2022, Wenzel2022-arxiv} for results on the convergence of Bayesian quadrature and related methods with acquisition functions other than~\eqref{eq:sequential-selection-for-guarantees}.

\section{Other guarantees}

Various other guarantees for Bayesian quadrature are available in the literature.
These include Bayesian quadrature with (i) nodes having strong symmetry properties~\citep[Sec.\@~4.4]{KarvonenSarkka2018}; (ii) nodes obtained by optimization~\citep{Briol2015, Kanagawa2019, Pronzato2023}; (iii) nodes sampled from distributions different from $\measure$~(\citealt{Briol2017}; see also \citealt{Bach2017}); (iv) nodes sampled from a determinantal point process~\citep{Belhadji2019, Belhadji2020-thesis, Belhadji2021}; and (v) nodes selected using a certain recombination algorithm~\citep[Sec.\@~6]{Adachi2022}.
\citet{CaiScarlett2024} have proved error bounds for BQ applied to the computation of normalizing constants of the form $\int_\domain \exp(-\lambda f(\vx)) \dif \vx$ with $\lambda > 0$.
Convergence guarantees for multi-output and multilevel Bayesian quadrature are provided in \citet[Sec.\@~3]{Xi2018} and \citet[Sec.\@~4]{Li2022}, respectively.
\citet{Chatalic2025} consider subsampling kernel mean embeddings approximated with Monte Carlo.
The equivalences reviewed in \Cref{sec:integration-in-rkhs} allow one to apply the wealth of results collected in \citet{Oettershagen2017} to Bayesian quadrature.

\chapter{Empirical illustrations}
\label{sec:experiments}

As mentioned in the introduction, the aim of our experiments is not to re-produce the performance of Bayesian quadrature methods, this having been done in the literature we have cited, nor to propose and evaluate new methods.
Rather, the aim is to illustrate the taxonomy of \Cref{sec:taxonomy} and the qualitative influence of its axes, \emph{sampling}, \emph{model} and \emph{inference}. In particular, the experiments show ablation curves that disentangle the effect of the axes for several integrands $f$ that differ in their degree of smoothness or (non-)stationarity. 
Where possible, we tie the qualitative or quantitative behavior of the BQ methods to the theoretical results from \Cref{sec:guarantees}. We hope the illustrations provide additional intuition on BQ methods for practitioners.

The available open source code for BQ is unfortunately still scarce in comparison to the diversity of methods proposed in the literature. Under this constraint, we have opted to implement the experiments in the open source Python library \emukit~\citep{Paleyes19, Paleyes2023} since, at the time of writing, it provided the highest diversity of BQ algorithms. It further enabled us to code in a single language. In particular, \emukit~contains a variety of kernel embeddings for the prevalent conjugate BQ model and a variety of node selection schemes. Another contender, the Python library \probnum~\citep{Wenger21}, did not contain adequate hyper-parameter tuning capability, nor state-of-the-art active sampling loops at the time of writing.\footnote{As the library seems to be under active development, this may have changed.} Unfortunately, scalable BQ methods are currently not available in \emukit~and are thus omitted here; some are implemented in the \matlab~library \gail~\citep{GAIL2020}. All experiments use exact inference.
Section~\ref{sec:exp-setup} describes the experimental setup and integrands (called test functions) in more detail. Section~\ref{sec:exp-results} illustrates and discusses results.

\section{Experimental setup}
\label{sec:exp-setup}

This section describes the experimental setup.

\subsection{Kernels (model axis)}
\label{sec:exp-kernels}

As all the experiments use the conjugate model, the models only differ by their choice of the kernel $k$. We select 5 different kernels to denote 5 different models. Three of the kernels are Matérn kernels in~\eqref{eq:matern-kernel} with smoothness parameter $\nu \in \{\frac{1}{2}, \frac{3}{2}, \frac{5}{2}\}$. As described in Sections~\ref{sec:integration-in-rkhs} and \ref{sec:guarantees-isotropic-matern}, these kernels are linked to Sobolev spaces. Matérn kernels are stationary.
  Further, we use the non-stationary and non-smooth Brownian motion kernel as in~\eqref{eq:brownian-kernel} and the stationary and infinitely smooth square exponential kernel in~\eqref{eq:se-kernel}.
  All selected kernels have analytic kernel embeddings with respect to $\measure$ that are available in \emukit. Table~\ref{tab:exp-hypers} summarizes the kernels and their parameters.
\begin{table}[h]
  \def\arraystretch{1.5} 
    \centering
    \begin{tabular}[t]{l|lcc}
{\bf Kernel}  & {\bf Kernel params.} & {\bf Stationary} & {\bf Smooth.}\\
\hline
\rowcolor{black!05} Matérn ($\nu=\frac{1}{2}$) & $\hparams = \{\sigma^2, \ell_1, \ldots, \ell_d\}$ & \cmarkg & 0\\
 Matérn ($\nu=\frac{3}{2}$) & $\hparams = \{\sigma^2, \ell_1, \ldots, \ell_d\}$  & \cmarkg & 1\\
 \rowcolor{black!05} Matérn ($\nu=\frac{5}{2}$) & $\hparams = \{\sigma^2, \ell_1, \ldots, \ell_d\}$ & \cmarkg & 2\\
 Brownian motion & $\hparams = \{\sigma^2\}$ & \xmarkr & 0\\
 \rowcolor{black!05} Square exponential & $\hparams = \{\sigma^2, \ell_1, \ldots, \ell_d\}$ & \cmarkg & $\infty$\\
\end{tabular}
\caption{Summary of the kernels used in the experiments. Here smoothness (smooth.) refers to the classical differentiability of the sample paths of the GP prior.}
\label{tab:exp-hypers}
\end{table}

\subsection{Node selection (sampling axis)}
\label{sec:exp-nodes}
Along the sampling axis we use 5 sampling schemes: two random ones and three deterministic ones to cover the mode of sampling. The two random schemes are random sampling from $\measure$ and the space-filling Latin hypercube sampling (Section~\ref{sec:sampling-random}). Two of the three deterministic samplers produce model-independent designs: the space-filling Sobol sequence and the classically motivated Legendre roots (Section~\ref{sec:deterministic-sampling}). The final deterministic node selection scheme uses a sequential, active sampler and the IVR acquisition function (Section~\ref{sec:sampling-active}). The samplers are summarized in Table~\ref{tab:exp-sampler}. Further details can be found in Section~\ref{sec:exp-details}.

\begin{table}[h]
  \def\arraystretch{1.5} 
    \centering
    \begin{tabular}[h]{l|lll}
{\bf Sampler}  & {\bf Mode} & {\bf Granular label} & {\bf Code}\\
\hline
\rowcolor{black!05} Random from $\measure$ & random & model-indep. & \numpy\\
 Latin hypercube        & random & model-indep. & \scipy\\
 \rowcolor{black!05} Sobol sequence         & deterministic & model-indep. & \scipy\\
 Legendre roots         & deterministic & model-indep. & \scipy\\
 \rowcolor{black!05} IVR                    & deterministic & model-dep., seq., act. & \emukit
\end{tabular}
\caption{Summary of the samplers used in the experiments. The ``code'' column indicates which library was used to create the nodes (more details in Section~\ref{sec:exp-details}).}
\label{tab:exp-sampler}
\end{table}

\subsection{Implementation details}
\label{sec:exp-code}
As the experiments use the conjugate model throughout, $\fGP=\gGP$ is a Gaussian process. To represent $\fGP$ we use the \gpy~Python library \citep{gpy2014} and for the kernel means and variances $\kernel_{\measure}$ and $\kernel_{\measure\measure}$ we use the Python library \emukit~\citep{Paleyes19, Paleyes2023} and its existing \gpy~quadrature wrappers.
The kernel parameters $\hparams$ (Table~\ref{tab:exp-hypers}) are fitted with maximum marginal likelihood (Section~\ref{sec:param-estimation}), which \gpy~defaults to.

We use the default settings for the nugget of both \emukit~and \gpy. The nugget $\lambda$ (Section~\ref{sec:num-ill-cond}) is composed of three terms:
\begin{equation*}
  \lambda = \lambda_{\textup{Emu}} + \lambda_{\textup{GPy}} + \lambda_{\textup{GPyDyn}} .
\end{equation*}
The first two terms, $\lambda_{\textup{Emu}}=10^{-10}$ and $\lambda_{\textup{GPy}}=10^{-8}$, are the fixed default nuggets of \emukit~and \gpy~respectively. While $\lambda_{\textup{Emu}}$ can be adjusted via the \emukit~interface, $\lambda_{\textup{GPy}}$ is not directly accessible to the user. The dynamic nugget term $\lambda_{\textup{GPyDyn}}$ is computed as follows. Let $\mA = \mK_{\points\points} +  (\lambda_{\textup{Emu}} + \lambda_{\textup{GPy}} )\Id_N$ be the kernel Gram matrix augmented with the fixed nugget terms and let $\bar{a}=\frac{1}{N}\sum_{i=1}^NA_{ii}$ be the mean of its diagonal elements. 
The third summand, $\lambda_{\textup{GPyDyn}}$, is a dynamic nugget that \gpy~adds automatically when the Cholesky decomposition of $\mA$ (used to solve the linear system of the kernel Gram matrix) fails. $\lambda_{\textup{GPyDyn}}$ is initialized to zero and then increased up to 5 times, taking the successive values $0$, $\bar{a}10^{-6}$ , $\bar{a}10^{-5}$ , $\bar{a}10^{-4}$, $\bar{a}10^{-3}$, and $\bar{a}10^{-2}$ until the Cholesky decomposition is successful. If the Cholesky decomposition fails with the largest dynamic nugget $\bar{a}10^{-2}$, the method finally gives up and throws an error. In contrast to the fixed nugget terms, the dynamic nugget depends on $\mK_{\points\points}$ and is reset to zero after each successful Cholesky decomposition.
Hence every Gram matrix in the experiments incorporates a nugget that is at least $10^{-8}$ but \emph{may} vary from one matrix to the next.

All nodes, except the ones obtained actively via IVR, can be pre-computed as their sampling is model-independent. In the settings where we pre-compute nodes we (i) first sample the nodes, (ii) then hand them to the \emukit~BQ wrapper as a fixed dataset, (iii) fit the hyper-parameters $\hparams$, and (iv) finally compute $I_{\measure}(\fGP)\mid\dataset$ using \emukit's kernel embeddings. When the nodes are random, we use the same \numpy~seed of~0 across all experiments. This means that random, model-independent nodes are identical across models and integrands $f$ for a given $N$, which may be reflected in some curves showing similar fluctuations. 
We elaborate on this further when describing the results below.
For active node selection with IVR, we fit the hyper-parameters successively after every newly obtained node.
IVR as an active node acquisition method further requires an initial design which we choose to be $N_{\textup{init}}=\min\{N, 5\}$ random nodes from $\measure$. This ensures that hyper-parameter optimization by maximum marginal likelihood does not fail. The initial design uses the same \numpy~seed of~0.

\subsection{Integrands (test functions)}
\label{sec:exp-func}

An integration problem is composed of an integrand or, \emph{test function}, $f$ and a measure $\measure$. For the experiments we choose $\measure$ to be the standard Lebesgue (i.e., uniform) measure on the unit interval $\domain=[0, 1]$.
The test functions $f$ cover a smooth and non-smooth as well as stationary and non-stationary\footnote{When referring to functions we use the word ``stationary'' as in \citet[p.\@~4]{RasmussenWilliams2006} who write \textit{``informally, stationarity means that the functions look similar at all $x$ locations''.}} integrands.
All test functions are continuous and do not exhibit ``jumps''.

In order to create distinct test functions with similar properties, we choose parameterized families of functions and sample the parameters for an individual test function from some distribution.
That is, $f$ is deterministic for each fixed set of parameters.
In total, for every parameterized family we produce $T=50$ deterministic integrands for evaluation. The two families are:

\begin{itemize}
\item \emph{Univariate, smooth, stationary integrands:} We use smooth test functions as in \citet[Sections~2 \&~3]{Filip19}. The integrand is defined as a finite Fourier series 
\begin{equation} \label{eq:fourier-test-func}
  f(x) = \sqrt{2}\sum_{j=0}^J{[a_j\cos(2\pi j x/L) + u_j\sin(2\pi j x/L)]} 
\end{equation}
with length-scale $L=5$ and $J=25$ terms and scalar coefficients $a_j$ and $u_j$ for $j=0,\dots, J$. The chosen values for $L$ and $J$ ensure that the function is reasonably variable in the domain $\domain$ yet not too easy to learn.
  All functions are infinitely smooth.
  Each of the $T=50$ sets of coefficients $\{a_j, u_j\}_{j=1}^J$ is produced according to $a_j, u_j\sim\Normal(0, 2(J+1))$.
  We set $a_0=u_0 = 0$ to define an offset.
\item \emph{Univariate, non-smooth, continuous, non-stationary integrands}: These are approximate Brownian motion paths represented by a truncated Karhunen--Loève expansion of the Brownian motion:
\begin{equation} \label{eq:BM-test-func}
  f(x) = \sqrt{2}\sum_{j=1}^J a_j \frac{\operatorname{sin}((j-\frac{1}{2})\pi x)}{(j-\frac{1}{2})\pi} .
\end{equation}
We choose $J=500$ terms. Each of the $T=50$ sets of sets of coefficients $\{a_j\}_{j=1}^J$ is sampled according to $a_j\sim\Normal(0, 1)$. The constructed functions resemble non-smooth, continuous functions quite well.
 \end{itemize}

\begin{table}[b]
  \def\arraystretch{1.5} 
    \centering
    \begin{tabular}[t]{lccl}
{\bf Name}  & {\bf Smooth} & {\bf Stationary} & {\bf $d$}\\
\hline
\rowcolor{black!05} Fourier [Eq.\@~\eqref{eq:fourier-test-func}]& \cmarkg & \cmarkg & 1\\
 Brownian motion [Eq.\@~\eqref{eq:BM-test-func}] & \xmarkr & \xmarkr & 1\\
\end{tabular}
\caption{Summary of the test functions used in the experiments.}
\label{tab:exp-f}
\end{table}

All of the above integrands can be represented as function handle $f(\cdot)$ in code as opposed to some fixed vector of observations $\vf_{\points}$.
This is achieved by constructing $f(\cdot)$ as a finite sum of basis functions. In our experiments, an explicit function handle is only really needed by active IVR sampling; the BQ model and inference do not require a function handle per se. 

The ground truth integrals are compute with \scipy's {\tt integrate.quad} method to high precision. Table~\ref{tab:exp-f} summarizes the tests functions. Illustrations of some of the $T=50$ integrands for each family are shown in Figure~\ref{fig:example-integrands}.

 \begin{figure}
  \centering
  \includegraphics[width=0.8\linewidth]{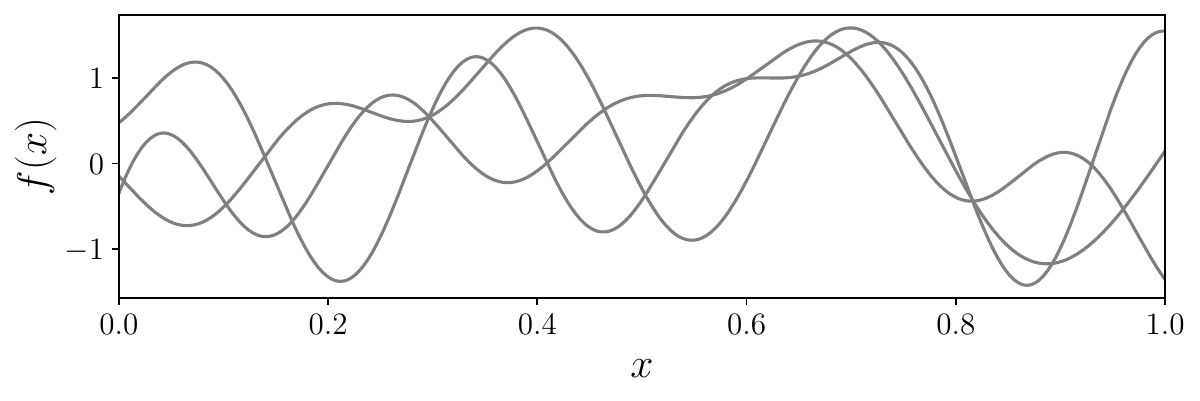} 
  \\
  \includegraphics[width=0.8\linewidth]{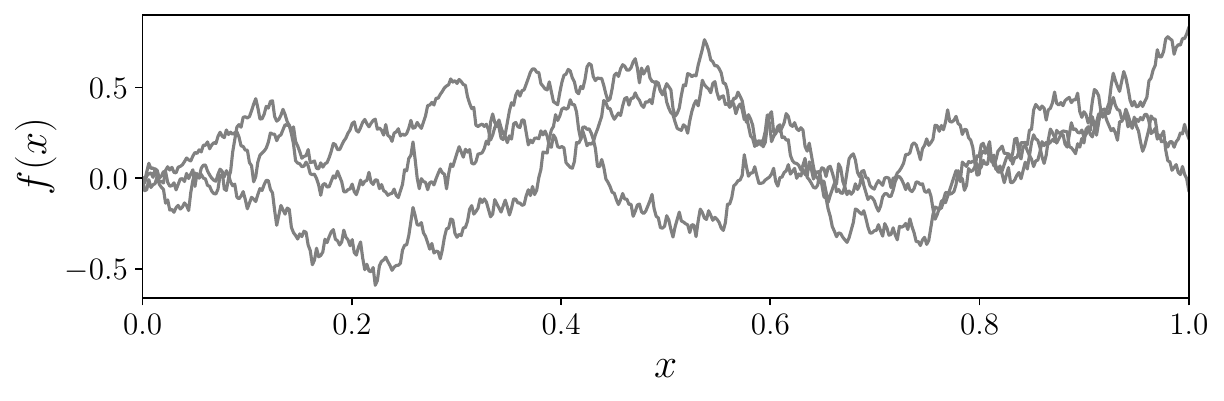}
  \caption{Examples of integrands. \emph{Top}: Three examples of smooth univariate integrands in~\eqref{eq:fourier-test-func}. \emph{Bottom}: Three example of approximate Brownian motion paths in~\eqref{eq:BM-test-func}.
  }
  \label{fig:example-integrands}
\end{figure}

\subsection{Description of individual experiments}
\label{sec:descr-indiv-exper}

We pair each of the five conjugate BQ models of Section~\ref{sec:exp-kernels} with the five node selection schemes of Section~\ref{sec:exp-nodes}.
This yields a total of  25 BQ methods, which we identify by the kernel name and sampling method. Then we evaluate each of the methods on the $T=50$ test integrands of each family of integrands of Section~\ref{sec:exp-func}, which amounts to 1250 combinations of test integrands and BQ methods per family.

Furthermore, we evaluate each of those 1250 combinations on several fixed budgets of size $N$ ranging from $N=1$ to some maximal $N_{f}$ in increments of 1. The maximal budget $N_f$ depends on the integrand family. We choose $N_f = 30$ and $N_f = 50$ for the smooth and non-smooth univariate problems, respectively.
These numbers are chosen to illustrate the performance change of the BQ methods across a relevant range of $N$.
In total, we execute $100\,000$ BQ runs, which were run sequentially on a general purpose \texttt{m1.xlarge} instance with 4 vCPUs.

\subsection{Scores}
\label{sec:scores}

Recall from Eqs.~\eqref{eq:bq-mean} and~\eqref{eq:bq-var} that $\mu_{\dataset}$ and $\Sigma_{\dataset}$ stand for the posterior mean and variance, respectively, of the conjugate BQ model.
The mean $\mu_{\dataset}$ serves as an estimator for the true integral value $I_{\measure}(f)$ and $\Sigma_{\dataset}$ describes the spread of the posterior distribution $ I_\measure(\fGP) \mid \dataset$.

In the experiments, we evaluate the behavior of the BQ methods with two empirical scores.
The first score is the \emph{error score}
\begin{equation*}
  \frac{1}{T}\sum_{t=1}^T\left|\frac{I_{\measure}(f_t) - \mu_{\dataset, t}}{I_{\measure}(f_t)}\right|.
\end{equation*}
The error score quantifies the performance of $\mu_{\dataset}$ as an estimator for the ground truth $I_{\measure}(f)$. More precisely, the error score is the average relative error between $\mu_{\dataset}$ and  $I_{\measure}(f)$ over the test functions.
The subscript $t$ indicates the $t$\textsuperscript{th} test integrand; the average is computed over the $T=50$ integrands for each family separately.\footnote{Occasional outliers, mainly from failed runs, have been discarded when computing the averages. Outliers are defined as scores that lie more than $1.5\times \Delta_{1-3}$ outside of the interquartile range $\Delta_{1-3}$ of all scores. This is the same definition that \scipy~uses for its default boxplots.}

The second score we use is the \emph{calibration score}
\begin{equation*}
  \frac{1}{T}\sum_{t=1}^T\frac{|I_{\measure}(f_t) - \mu_{\dataset, t}|}{\Sigma_{\dataset, t}^{1/2}}.
\end{equation*}
The calibration score quantifies the calibration of the BQ method, or how well the credible interval of the posterior describes the error between the estimator  $\mu_{\dataset}$ and the true integral value $I_{\measure}(f)$. 
The calibration score equals the average absolute error in units of the corresponding credible intervals.
For each test function, the 68\% credible interval centered around $I_{\measure}(\fGP) \mid \dataset = I_{\measure}(f)$ and 1 standard deviation wide to each side is \smash{$[\mu_{\dataset} - \Sigma_{\dataset}^{1/2}, \mu_{\dataset} + \Sigma_{\dataset}^{1/2}]$}, or equivalently \smash{$[-\Sigma_{\dataset}^{1/2}, \Sigma_{\dataset}^{1/2}]$} for the shifted value $I_{\measure}(\fGP) \mid \dataset -  \mu_{\dataset} = I_{\measure}(f) -  \mu_{\dataset}$. Taking the absolute value lets us compare the size of $|I_{\measure}(f) -  \mu_{\dataset}|$ (enumerator) to $\Sigma_{\dataset}^{1/2}$ (denominator).
For an individual test function $f$, the score is smaller than 1 if the true value $I_{\measure}(f)$ lies inside the credible interval and larger than 1 if it lies outside. 
Hence, a somewhat \emph{constant} calibration score across different $N$ (e.g., hovering around $1$) suggests that the BQ method is well calibrated. An increasing or decaying score suggests that the method is prone to over- or underconfidence, respectively~\citep{Karvonen2020b, Wang2021}. It is usually not as harmful for a method to be underconfident (too uncertain) than overconfident (too certain).\footnote{The calibration score is a simple proxy to get a general idea of the calibration and provides by no means a complete picture. We refer to \citet{Cockayne2022} for a formal treatment of calibration (see in particular Remark~9).}

\section{Results}
\label{sec:exp-results}

This section describes the qualitative behavior of the BQ methods on the families of test functions.

\subsection{Smooth test functions (sampling axis)}
\label{sec:smooth-test-funct-sampling}

Figure~\ref{fig:sampling-config003} depicts the performance of the 5 models along the \emph{sampling} axis for the family of univariate smooth test functions. Each row shows a single model with kernels of increasing smoothness from top to bottom. The first 4 rows show models of finite, and hence of lesser, smoothness than the test functions (Brownian motion, Matérn~$\sfrac{1}{2}$, Matérn~$\sfrac{3}{2}$, and Matérn~$\sfrac{5}{2}$), while the last row shows a model with matching (infinite) smoothness (square exponential). The left and right columns show the error score and the calibration score, respectively. From the plots we can make several qualitative observations.

\afterpage{%
\begin{figure}
  \centering
  \includegraphics[width=\textwidth]{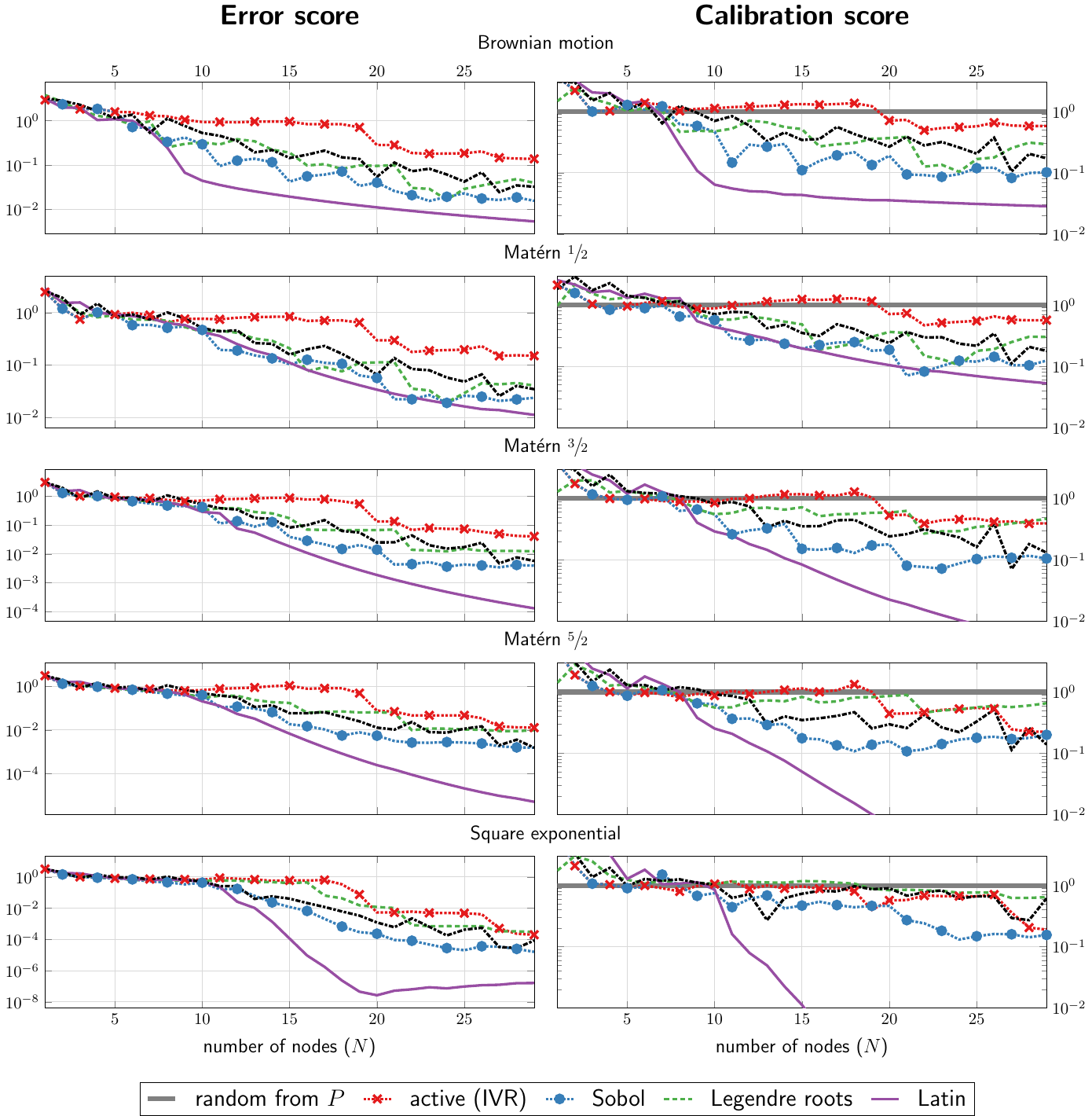}
  \caption{Performance of the models along the \emph{sampling} axis on \emph{univariate smooth functions}. 
  Left column: error scores. Right column: calibration scores.
  Each row shows a single model with kernels of increasing smoothness from top to bottom. The first 4 rows show models of finite, and hence of lesser, smoothness than the test functions (Brownian motion, Matérn~$\sfrac{1}{2}$, Matérn~$\sfrac{3}{2}$, and Matérn~$\sfrac{5}{2}$), while the last row shows a model with matching (infinite) smoothness (square exponential). The left and right columns show the error score and the calibration score, respectively. From the plots we can make several qualitative observations.
  It is apparent that the sampling choice has a major effect on the error score (left column).
  In this particular example, Legendre roots seem to outperform all other sampling methods considered, followed by IVR, while random samples from $\measure$ perform the worst across all models. Further, all models are well calibrated or under-confident for \smash{$N \gtrapprox 10$}. See main text for more details.
  }
  \label{fig:sampling-config003}
\end{figure}
\clearpage
}

  First, it is apparent that the sampling choice has a major effect on the error score (left column).
  In this particular example, Legendre roots seem to outperform all other sampling methods considered, followed by IVR. Random samples from $\measure$ perform the worst across all models. The Sobol sequence and Latin hypercube perform similarly, and generally better than random samples from $\measure$.
  While not covered by the theory in Section~\ref{sec:guarantees}, the superiority of Legendre roots is likely related to the well-known phenomenon in numerical analysis that nodes which are denser near the boundary of the domain than in its interior are good for approximating and integrating infinitely smooth functions~\citep{Trefethen2008, RiegerZwicknagl2014}. In Figure~\ref{fig:sampling-config017} we observe that the performance of Legendre nodes differs little from other sampling choices when the integrands are non-smooth.

  Second, from the right column we see that all models are well calibrated or underconfident for \smash{$N \gtrapprox 10$}, as the calibration scores hover around a constant value or tend to zero.
   There is a ``kink'' in the calibration score for random nodes (red) at $N \approx 18$ for nodes sampled from $\measure$ that vanishes at $N=20$ (coinciding with a drop of the corresponding error scores). This is an artifact of a ``gap'' in the randomly sampled nodes that gets filled at $N = 20$.
  The other, space-filling designs do not exhibit this fluctuation.
  The ``kink'' is less pronounced for more well-specified (i.e., smoother) models.
  The ``kinks'' do not average out across the $T=50$ integrands because the random nodes $\points$ are identical for a given $N$ and across all experiments as the same random seed was used (see Section~\ref{sec:exp-code}). When the model does not oversmooth a univariate stationary integrand, the sampling choice therefore appears to mostly affect the performance of the mean estimator $\mu_{\dataset}$. Calibration is decent and largely unaffected by sampling. This makes the BQ methods usable in this particular setting.

 The third observation we make is the flattening of the performance score in the bottom left plot of Figure~\ref{fig:sampling-config003} for the Legendre nodes at around $10^{-7}$. This is likely due to the positive nugget used in the experiments (see Section~\ref{sec:exp-code}). \Cref{thm:matern-generic,thm:matern-quasi-uniform,thm:se-generic,thm:se-quasi-uniform} suggest that a positive nugget may either greatly slow the rate of convergence or altogether prevent convergence after a certain accuracy has been reached.

\subsection{Smooth test functions (model axis)}
\label{sec:smooth-test-funct-model}

Figure~\ref{fig:kernel-config003} shows the same curves as Figure~\ref{fig:sampling-config003} but this time along the \emph{model} axis. That is, rows are now grouped according to the sampling choice such that the effect of the kernel can be observed in each subplot. The nodes $\points$ are identical on each row and $N$ across all test functions if sampling is model-independent (rows 1--4). For deterministic nodes this is fulfilled trivially; for random nodes this is because the same random seed was used across experiments (see Section~\ref{sec:exp-code}). Hence, solely the effect of the model can be observed in each plot. The nodes on row~5 differ across $N$, as well as models and test functions, as the IVR sampler is active and model-dependent. Again we make some qualitative observations.

\afterpage{%
\begin{figure}[t]
  \centering
  \includegraphics[width=\textwidth]{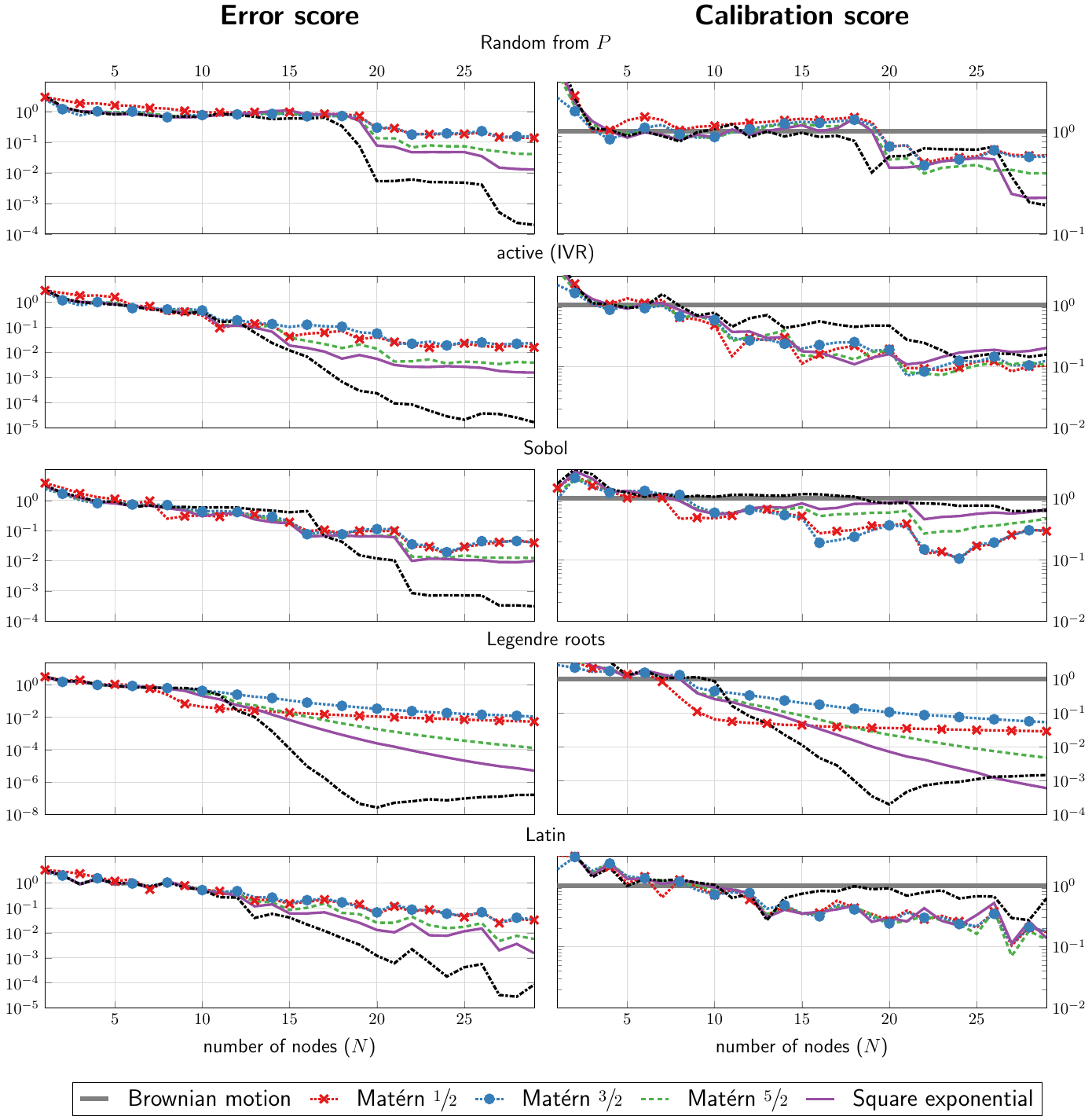}
  \caption{
    Performance of the models along the \emph{model} axis on the family of \emph{univariate smooth functions}.
    Left column: error scores. Right column: calibration scores.
    Rows are now grouped according to the sampling choice such that the effect of the kernel can be observed in each subplot. The nodes $\points$ are identical on each row and $N$ across all test functions if sampling is model-independent (rows 1--4). For deterministic nodes this is fulfilled trivially; for random nodes this is because the same random seed was used across experiments (see Section~\ref{sec:exp-code}). Hence, solely the effect of the model can be observed in each plot. The nodes on row~5 differ across $N$, as well as models and test functions, as the IVR sampler is active and model-dependent.
    The models perform better with increasing smoothness; the infinitely differentiable square exponential kernel that matches the smoothness of the test functions performs best across all sampling choices. All models are well calibrated or underconfident.
    See Section~\ref{sec:smooth-test-funct-model} for more details.
 }
  \label{fig:kernel-config003}
\end{figure}
\clearpage
}

First, the right column shows that all models are well calibrated or underconfident. The ``kink'' mentioned in Section~\ref{sec:smooth-test-funct-sampling} at $N \approx 18$ for random nodes from $\measure$ can be observed in the top right plot. There are no substantial differences between the models, except that the infinitely smooth square exponential model is generally the least underconfident.
Because the test integrands are infinitely smooth, this is not surprising.

Second, the error score (left column) decreases as the kernel becomes smoother.
In particular, the infinitely smooth square exponential kernel performs better than other kernels across all sampling choices.
This is consistent with the theoretical guarantees of Section~\ref{sec:guarantees}, which indicate that smoother kernels can attain faster rates of convergence provided that the integrand is sufficiently smooth.
In conclusion, models that undersmooth the truth are decently calibrated but have larger integration errors than models that do not undersmooth.

\subsection{Non-smooth test functions (sampling axis)}
\label{sec:non-smooth-test-sampling}

Figure~\ref{fig:sampling-config017} depicts the performance of the 5 models along the \emph{sampling} axis for the family of approximate Brownian motion paths. Again, each row shows a single model with kernels of increasing smoothness from top to bottom. The first two rows show models that are continuous but non-differentiable and match the smoothness (or roughness) of the test functions. The models on rows 3--5 are increasingly smooth and thus oversmooth the truth. In particular, the square exponential kernel on row~5 is ``infinitely smoother'' than the test functions. Again, the left and right columns show the error score and the calibration score, respectively. From the plots we can make the following qualitative observations.

First, the sampling choice has no apparent effect on the error score for the models with matching smoothness (rows~1--2); solely random sampling from $\measure$ performs slightly poorer. 
This is quite surprising and unlike the behavior we observed in Section~\ref{sec:smooth-test-funct-sampling} for the family of smooth test functions.
As expected, calibration is good for all sampling schemes.

\afterpage{%
\begin{figure}[t]
  \centering
  \includegraphics[width=\linewidth]{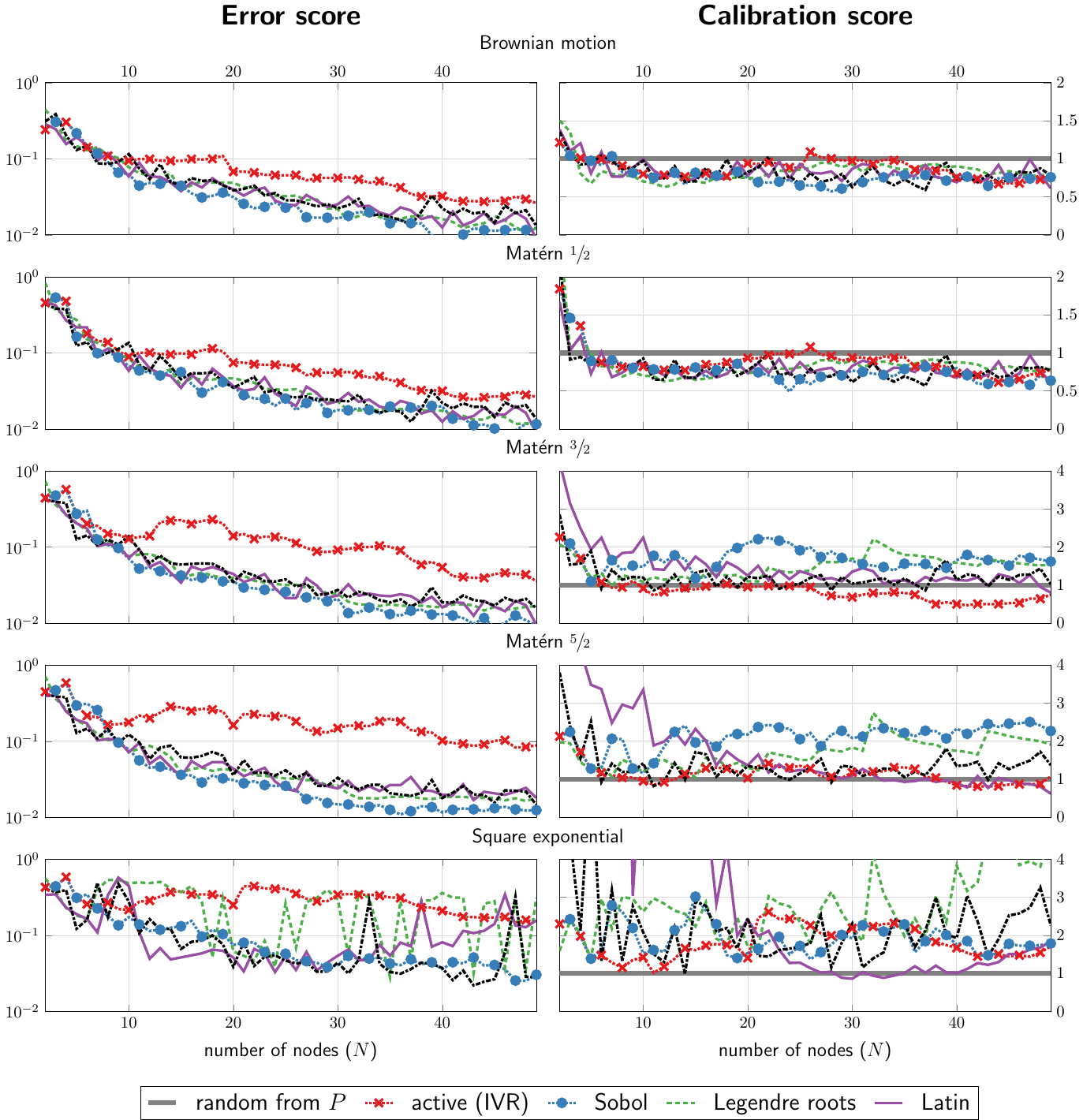}
  \caption{
    Performance of the models along the \emph{sampling} axis on the family of \emph{approximate Brownian motion paths}. 
    Left column: error scores. Right column: calibration scores.
    Each row shows a single model with kernels of increasing smoothness from top to bottom. The first two rows show models that are continuous but non-differentiable and match the smoothness (or roughness) of the test functions. The models on rows 3--5 are increasingly smooth and thus oversmooth the truth. In particular, the square exponential kernel on row~5 is ``infinitely smoother'' than the test functions.
    The sampling scheme has no visible effect on the error score for models with matching smoothness. 
    Random sampling from $\measure$ is slightly worse than other schemes.
    All sampling schemes yield well calibrated models.
    The smoother then model, the worse it performs in both scores when nodes are sampled from $\measure$.    
    See Section~\ref{sec:non-smooth-test-sampling} for more details.
    }
  \label{fig:sampling-config017}
\end{figure}
\clearpage
}

Second, models that oversmooth (rows~3--5) perform successively worse in error score. Interestingly, random sampling from $\measure$ reinforces the adverse effect of oversmoothing.
This may be explained by the presence of the mesh ratio term $\rho_\points^{\alpha - \beta}$ [mesh ratio is define in Eq.\@~\eqref{eq:mesh-ratio}] in \Cref{thm:matern-generic}: when there is significant oversmoothing ($\alpha - \beta$ is large), this term is large.
Moreover, the mesh ratio is large if the nodes exhibit clustering, which explains the behavior of random sampling.
The calibration score is less stable and indicates increasing overconfidence across sampling schemes for models that oversmooth.

Unlike undersmoothing, oversmoothing the truth can have a detrimental effect on error and calibration scores across all sampling choices. Oversmoothing the truth slightly can have an undesirable effect on the calibration score while the error score is mostly unaffected unless random sampling from $\measure$ is used. At the same time, the choice of sampler is less important for models with matching smoothness. Random sampling seems to be the poorest choice for all models, a fact that may be explained by the observation above on the magnitude of the mesh ratio for random nodes.

\subsection{Non-smooth test functions (model axis)}
\label{sec:non-smooth-test-model}

Figure~\ref{fig:kernel-config017} shows the same curves as Figure~\ref{fig:sampling-config017} but this time along the \emph{model} axis. That is, rows are now grouped according to the sampling choice such that the effect of the kernel can be observed in each subplot. Again, the nodes $\points$ are identical on each row and $N$ across all test functions if sampling is model-independent (rows 1--4).
Hence, solely the effect of the model can be observed in each plot. Again, the nodes in the bottom row~5 differ across $N$, as well as models and test functions, as the IVR sampler is active and model-dependent. We observe the following qualitative behavior.

Unsurprisingly, models that match the roughness of the integrand (Brownian and Matérn~$\sfrac{1}{2}$) perform best on both error and calibration scores for all sampling schemes. 
Models that oversmooth slightly (Matérn~$\sfrac{3}{2}$ and Matérn~$\sfrac{5}{2}$) have equally good error scores (exception described below) but appear to be somewhat overconfident.
The square exponential kernel that oversmooths heavily performs poorly across all sampling schemes.
Random sampling from $\measure$ (row~1) is the only sampling scheme where even slight oversmoothing has an adverse effect on the the error score.
This effect was given theoretical explanation in \Cref{sec:non-smooth-test-sampling}.
Random sampling should therefore be avoided whenever there is a risk of oversmoothing.
However, random sampling can be expected to perform well and be well-calibrated if smoothness is correctly specified.

\afterpage{%
\begin{figure}[t]
  \centering
  \includegraphics[width=\linewidth]{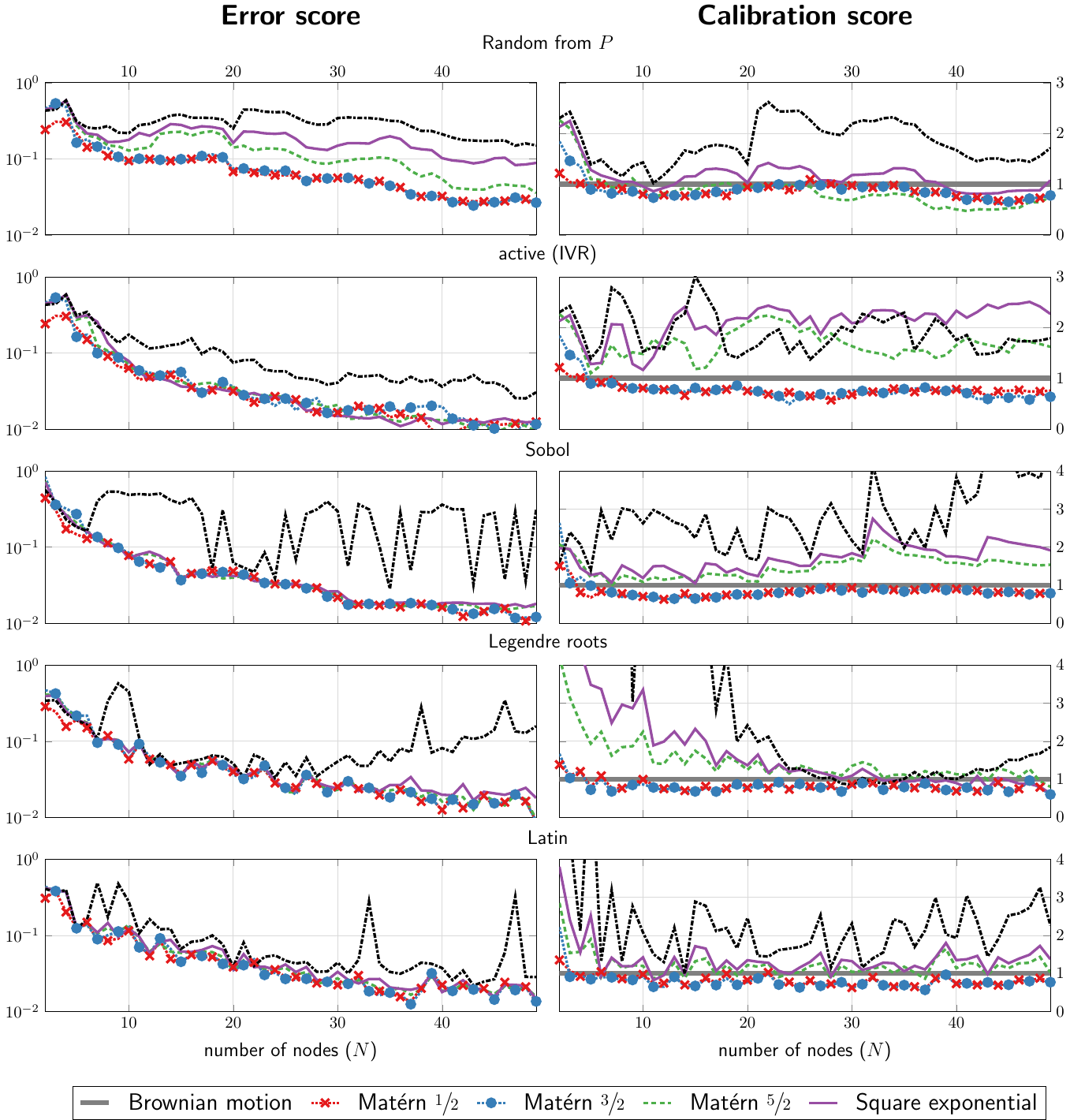}
  \caption{
    Performance of the models along the \emph{model} axis on the family of \emph{approximate Brownian motion paths}.
    Left column: error scores. Right column: calibration scores.
    Each row shows a single sampling choice for models with increasing smoothness. 
    The nodes $\points$ are identical on each row and $N$ across all test functions if sampling is model-independent (rows 1--4).
  Hence, solely the effect of the model can be observed in each plot.   
  The nodes in the bottom row~5 differ across $N$, as well as models and test functions, as the IVR sampler is active and model-dependent. 
  Models that match the smoothness of the integrand perform best on both scores for all sampling schemes. 
  Models perform poorer in error score and can be increasingly overconfident the more they oversmooth the truth; unsurprisingly, the square exponential kernel that oversmooths the truth infinitely performs much poorer than all other models.    
    See Section~\ref{sec:non-smooth-test-model} for more details.
    }
  \label{fig:kernel-config017}
\end{figure}
\clearpage
}

\section{Further experimental details}
\label{sec:exp-details}

This section contains further details on the implementation of the experiments.
The samplers of Table~\ref{tab:exp-sampler} were implemented as follows:

\begin{itemize}
  \item \emph{Random from $\measure$} (\emph{random}): Nodes are random samples from the integration measure $\measure$. We use the method
  \begin{equation*}
    \text{\texttt{uniform} \: of \: \numpy's \texttt{random.default\_rng}}
  \end{equation*}
  when $\measure$ is the uniform measure on the unit interval $[0, 1]$.

  \item \emph{Latin hypercube} (\emph{random}): Nodes are samples of a Latin hypercube \citep{MacKay1979}, a distribution of space-filling pseudo hyper-grids. We use the method
  \begin{equation*}
    \text{ \texttt{sample} \: of \: \scipy's  \texttt{stats.qmc.LatinHypercube} }
  \end{equation*} 
  with default arguments.

  \item \emph{Sobol sequence} (\emph{deterministic}): Nodes are the Sobol sequence in dimension $d$. We use the method
  \begin{equation*}
    \text{\texttt{random\_base2} \: of \: \scipy's \texttt{stats.qmc.Sobol} }
  \end{equation*}
  with \texttt{scramble=False}.
  A Sobol sequence produces points for $N$, an integer power of 2. For $N$ which are not powers of 2, we use the next larger sequence and truncate it.

  \item \emph{Legendre roots} (\emph{deterministic}): Nodes are the roots of some Legendre polynomial. We use the method 
  \begin{equation*}
    \text{\texttt{roots\_legendre} \: of \: \scipy's \texttt{special}}
  \end{equation*}
  module.
  
  \item \emph{Integral variance reduction} (\emph{sequential, active, model-dependent}): Nodes are collected sequentially and myopically with batch size $n=1$. The acquisition function  $a(\tilde{\vx}\mid \measure)$ is the integral variance reduction (IVR) acquisition function from Section~\ref{sec:sampling-active}. We use
  \begin{equation*}
    \begin{split}
    &{\tt IntegralVarianceReduction} \\
     &\hspace{2cm}\text{of \: \emukit's {\tt quadrature.acquisitions}}
    \end{split}
  \end{equation*}
  module as well as its sequential acquisition loop.
\end{itemize}

\chapter{Proofs for Chapter~\ref{sec:guarantees}} \label{sec:proofs}

This chapter contains proofs for the results in \Cref{sec:guarantees}.
We begin with some preliminaries.

\section{Reproducing kernel Hilbert spaces}
\label{sec:rkhs-appendix}

Some basic properties of RKHSs have appeared already in \Cref{sec:conn-appr-theory}.
A kernel $\kernel \colon \domain \times \domain \to \R$ is positive-semidefinite if the kernel Gram matrix $\mK_{\points\points} = (\kernel(\vx_i, \vx_j))_{i,j=1}^N$ is positive-semidefinite for any $N \in \N$ and any set $\points = \{\vx_1, \ldots, \vx_N \} \subseteq \domain$ of points.
The classical Moore--Aronszajn theorem~\citep{Aronszajn1950} states that for every positive-semidefinite kernel there is a unique Hilbert space $\rkhs(\kernel)$, the RKHS of $\kernel$, which consists of real-valued functions defined on $\domain$ and is equipped with an inner product $\langle \cdot, \cdot \rangle_{\rkhs(\kernel)}$ and the corresponding norm $\norm{\cdot}_{\rkhs(\kernel)}$ such that the kernel has the following reproducing property:
\begin{equation} \label{eq:reproducing-property}
  \langle g, \kernel(\cdot , \vx) \rangle_{\rkhs(\kernel)} = g(\vx) \quad \text{ for all } \quad g \in \rkhs(\kernel) \text{ and } \vx \in \domain.
\end{equation}
Whenever it is necessary to be precise about the domain of the functions in the RKHS, we say that $\rkhs(\domain)$ is an RKHS \emph{on} $\domain$.
From the Riesz representation theorem it follows that the reproducing property extends to any continuous linear functional on $\rkhs(\kernel)$.
Namely, if $L \colon \rkhs(\kernel) \to \R$ is a continuous linear functional, then
\begin{equation*}
  L(g) = \langle g, \kernel_L \rangle_{\rkhs(\kernel)} \quad \text { for all } \quad g \in \rkhs(\kernel),
\end{equation*}
where $\kernel_L$, the representer of $L$, is defined as $\kernel_L(\vx) = L(k(\cdot, x))$.
Selecting $L = I_\measure$ yields
\begin{equation*}
  I_\measure(g) = \langle g, \kernel_\measure \rangle_{\rkhs(\kernel)},
\end{equation*}
where $\kernel_\measure(\vx) = \int_\domain \kernel(\vx', \vx) \dif \measure(\vx')$ is the kernel mean embedding in~\eqref{eq:kernel-mean}.

It is usually not straightforward to decide whether or not a given function is an element of the RKHS of a given kernel.
There are two characterisations of the RKHS that are occasionally helpful.
If there are functions $\{\varphi_i\}_{i=1}^\infty$ such that
\begin{equation*}
  \kernel(\vx, \vx') = \sum_{i=1}^\infty \varphi_i(\vx) \varphi_i(\vx') \quad \text{ for all } \quad \vx, \vx' \in \domain,
\end{equation*}
then every $g \in \rkhs(\kernel)$ may be written as $g = \sum_{i=1}^\infty c_i \varphi_i$ (the convergence is both in $\norm{\cdot}_{\rkhs(\kernel)}$ and pointwise) for some $c_i \in \R$ such that $\sum_{i=1}^\infty c_i^2 < \infty$~\citep[Sec.\@~2.1]{Paulsen2016}.
Conversely, if no such expansion for square-summable $c_i$ exists, then $g \notin \rkhs(\kernel)$.

If the kernel is stationary on $\R^d$, in that there is a function $\Phi_\kernel \colon \R^d \to \R$ such that $\kernel(\vx, \vx') = \Phi_k(\vx - \vx')$ for all $\vx, \vx' \in \R^d$, then the RKHS of $\kernel$ on $\R^d$ may be expresssed in terms of Fourier transforms~(\citealp{KimeldorfWahba1970}; \citealp[Thm.\@~10.12]{Wendland2005}).
Namely, if $\Phi_\kernel$ is continuous and integrable, then $\rkhs(\kernel)$ is 
\begin{equation} \label{eq:RKHS-Fourier-Rd}
  \Set[\bigg]{ g \in C(\R^d) \cap L^2(\R^d) }{ \norm{g}_{\rkhs(k)}^2 = \frac{1}{(2\pi)^{d/2}} \int_{\R^d} \frac{\abs[0]{\widehat{g}(\vxi)}^2}{\widehat{\Phi}_k(\vxi)} \dif \vxi < \infty},
\end{equation}
where $\widehat{g}(\vxi) = (2\pi)^{-d/2} \int_{\R^d} g(\vx) e^{-\mathrm{i} \vxi\Trans \vx} \dif \vx$ denotes the Fourier transform.
To use the Fourier characterisation on a subset of $\R^d$ we use a standard result on restrictions of RKHSs~\citep[Sec.\@~1.4.2]{Berlinet2004}.
Let $\rkhs(\kernel, \R^d)$ denote the RKHS in~\eqref{eq:RKHS-Fourier-Rd} and let $\rkhs(\kernel)$ be the RKHS on $\domain \subset \R^d$.
Then $\rkhs(\kernel)$ consists of functions $g \colon \domain \to \R$ for which there exists an \emph{extension} $g_e \in \rkhs(\kernel, \R^d)$ such that $g_e(\vx) = g(\vx)$ for all $\vx \in \domain$ (i.e., $g_e|_\domain = g$).
Moreover,
\begin{equation*}
  \norm[0]{g}_{\rkhs(\kernel)} = \min\Set{ \norm[0]{g_e}_{\rkhs(\kernel, \R^d)} }{ g_e \in \rkhs(\kernel, \R^d) \text{ is s.t. } g_e|_\domain = g}.
\end{equation*}

\section{Sobolev spaces} \label{sec:sobolev-spaces}

In Section~\ref{sec:guarantees-isotropic-matern}, we defined the Sobolev space $H^\alpha([0, 1]^d)$ of integer smoothness $\alpha$ as the space of functions whose weak derivatives up to order $\alpha$ exist and are in $L^2([0, 1]^d)$.
To allow fractional smoothness we may turn to Fourier transforms.
For any $\alpha > 0$, we let $H^{\alpha}(\R^d)$ be the Hilbert space of functions $g \in L^2(\R^d)$ such that
\begin{equation} \label{eq:sobolev-norm-fourier}
  \norm[0]{ g }_{H^{\alpha}(\R^d)}^2 = \int_{\R^d} (1 + \norm[0]{\vxi}^2)^\alpha \abs[0]{ \widehat{g}(\vxi) }^2 \dif \vxi < \infty.
\end{equation}
To illustrate the relationship of this norm to condition in~\eqref{eq:sobolev-norm-weak-derivative} expressed using weak derivatives, consider $\alpha \in \N$ and $d = 1$.
We may then use the binomial theorem, the properties of the Fourier transform, and Parseval's identity to compute
\begin{equation*}
  \begin{split}
    \norm[0]{ g }_{H^{\alpha}(\R)}^2 = \int_{\R} (1 + \xi^2)^\alpha \abs[0]{ \widehat{g}(\xi) }^2 \dif \xi &= \sum_{n=0}^\alpha \binom{\alpha}{n} \int_{\R} \abs[0]{ \xi^{n} \widehat{g}(\xi) }^2 \dif \xi \\
    &= \sum_{n=0}^\alpha \binom{\alpha}{n} \int_{\R} [ \mathrm{D}^n g(x) ]^2 \dif x,
  \end{split}
\end{equation*}
which is, up to binomial coefficient factors for each integral, equal to the left-hand side of~\eqref{eq:sobolev-norm-weak-derivative} defined on $\R$ instead of on $[0, 1]$.
A fractional Sobolev space and its norm on $[0, 1]^d$ can be defined via restrictions in a way analogous to how we related $\rkhs(\kernel)$ to $\rkhs(\kernel, \R^d)$ in Section~\ref{sec:rkhs-appendix}.
One can then show that $g$ is in the resulting space if and only if it satisfies~\eqref{eq:sobolev-norm-weak-derivative}.

\section{Proofs for Section~\ref{sec:conv-rates}} \label{sec:proofs-conv-rates}

This section contains proofs for the results in Section~\ref{sec:conv-rates}.

\subsection{Proofs for generic nodes}

It is a basic result~\citep[e.g.,][Sec.\@~3.2]{Kanagawa2018} that the regularised posterior mean
\begin{equation*}
  m_{\dataset, \lambda}(\vx) = \vk_\points(\vx)\Trans (\mK_{\points\points} + \lambda \Id_N  )^{-1} \vy
\end{equation*}
is the unique solution to a regularised least-squares problem:
\begin{equation} \label{eq:mean-regularised-least-squeres}
  m_{\dataset, \lambda} = \argmin_{s \in \rkhs(\kernel)} \Big\{ \sum_{i=1}^N [\f(\vx_i) - s(\vx_i)]^2 + \lambda \norm[0]{s}_{\rkhs(\kernel)}^2 \Big \}.
\end{equation}
Because $m_\dataset$ is an interpolant to $\f$, which is to say that $m_\dataset(\vx_i) = \f(\vx_i)$ for every $i = 1,\ldots,N$, and an element of $\rkhs(\kernel)$, we have from~\eqref{eq:mean-regularised-least-squeres} that 
\begin{equation} \label{eq:regularisation-norm-bound}
  \sum_{i=1}^N [\f(\vx_i) - m_{\dataset,\lambda}(\vx_i)]^2 + \lambda \norm[0]{m_{\dataset,\lambda}}_{\rkhs(\kernel)}^2 \leq \lambda \norm[0]{m_{\dataset}}_{\rkhs(\kernel)}^2.
\end{equation}
Consequently,
\begin{equation} \label{eq:regularised-nodeset-error}
  \Big(\sum_{i=1}^N [\f(\vx_i) - m_{\dataset,\lambda}(\vx_i)]^2 \Big)^{1/2} \leq \sqrt{\lambda} \norm[0]{m_{\dataset}}_{\rkhs(\kernel)}.
\end{equation}
We also need to recall the minimum-norm interpolation property~\eqref{eq:minimum-norm-interpolation} of the conditional mean, which states that
\begin{equation} \label{eq:minimum-norm-interpolation-proofs}
  m_\dataset = \argmin_{s \in \rkhs(\kernel)} \Set[]{ \norm[0]{s}_{\rkhs(\kernel)} }{ s(\vx_i) = \f(\vx_i) \text{ for every } i = 1, \ldots, N},
\end{equation}
and its particular consequence $\norm[0]{m_\dataset}_{\rkhs(\kernel)} \leq \norm[0]{\f}_{\rkhs(\kernel)}$ if $\f \in \rkhs(\kernel)$.

Observe that~\eqref{eq:bq-mean} and its regularised variant, followed by Jensen's inequality, yield
\begin{equation} \label{eq:bq-integral-bound-interpolant}
  \begin{split}
  \abs[0]{ I_P(\f) - \mu_{\dataset, \lambda} } &= \Big\lvert \int_\domain [ \f(\vx) - m_{\dataset,\lambda}(\vx) ] \dif \measure(\vx) \Big\rvert \\
  &\leq \int_\domain \abs[0]{ \f(\vx) - m_{\dataset,\lambda}(\vx) } \dif \measure(\vx).
  \end{split}
\end{equation}
Assumption~\ref{assumption:theory} and a second application of Jensen's inequality then produce
\begin{equation} \label{eq:error-L1L2}
  \begin{split}
    \abs[0]{ I_P(\f) - \mu_{N, \lambda} } &\leq p_\textup{max} \int_{[0, 1]^d} \abs[0]{ \f(\vx) - m_{\dataset,\lambda}(\vx) } \dif \vx \\
    &\leq p_\textup{max} \norm[0]{\f - m_{\dataset,\lambda}}_{L^2([0,1]^d)}.    
    \end{split}
\end{equation}
To bound the integration error it is therefore sufficient to bound the error of approximating $f$ with $m_{\dataset,\lambda}$ in $L^2([0,1]^d)$ norm.
This is what is done in most of the proofs below.

\begin{proof}[Proof of Proposition~\ref{prop:generic-guarantee-1}]
Because the node sets are assumed nested, $m_{N+1}$ interpolates $f$ at $\vx_1, \ldots, \vx_N$.
It thus follows from~\eqref{eq:minimum-norm-interpolation-proofs} that for any sequence of nodes the norm $\norm[0]{m_N}_{\rkhs(\kernel)}$ is non-decreasing in $N$.
Moreover, $\norm[0]{m_{N}}_{\rkhs(\kernel)} \leq \norm[0]{\f}_{\rkhs(\kernel)}$ since $\f \in \rkhs(\kernel)$.
From~\eqref{eq:gp-posterior-mean} and the reproducing property~\eqref{eq:reproducing-property} we get
\begin{equation*}
  \langle m_{N'}, m_N \rangle_{\rkhs(\kernel)} = \langle m_N, m_N \rangle_{\rkhs(\kernel)} = \vf_\points\Trans \mK_{\points\points}^{-1} \vf_\points
\end{equation*}
for any $N' \geq N$.
Therefore 
\begin{equation*}
  \begin{split}
  \lVert m_{N'} - m_N \rVert_{\rkhs(\kernel)}^2 &= \lVert m_{N'} \rVert_{\rkhs(\kernel)}^2 - 2 \langle m_{N'}, m_N \rangle_{\rkhs(\kernel)} + \lVert m_N \rVert_{\rkhs(\kernel)}^2 \\
  &= \lVert m_{N'} \rVert_{\rkhs(\kernel)}^2 - \lVert m_N \rVert_{\rkhs(\kernel)}^2,
  \end{split}
\end{equation*}
which tends to zero as $N \to \infty$ because $\norm[0]{m_N}_{\rkhs(\kernel)}$ is non-decreasing and bounded.
Hence $(m_N)_{N=1}^\infty$ is Cauchy in $\rkhs(\kernel)$ and consequently has a limit $m \in \rkhs(\kernel)$.
We show that the assumptions imply $m = f$.
It is a standard result that the continuity of $\kernel(\cdot, \vx)$ for all $\vx \in \domain$ and the boundedness of $k(\vx, \vx)$ imply that every element of $\rkhs(\kernel)$ is continuous~\citep[e.g.,][Lem.\@~4.28]{Steinwart2008}.
Hence $m$ and $f$ are continuous.
The assumption $h_N \to 0$ is equivalent to $\{\vx_i\}_{i=1}^\infty$ being dense in $\domain$.
A continuous function is fully determined by its values at at a dense set.
Because $m(\vx_i) = f(\vx_i)$ for every $i \in \N$, it thus follows that $m = f$.
Therefore $\lVert f - m_N \rVert_{\rkhs(\kernel)} \to 0$ as $N \to \infty$.
From the reproducing property and the Cauchy--Schwarz inequality it then follows that $m_N$ tends to $f$ uniformly on $\domain$:
\begin{equation*} 
\begin{split}
  \sup_{\vx \in \domain} \abs[0]{ m_{N}(\vx) - f(\vx) } &= \sup_{\vx \in \domain} \abs[0]{\langle m_{N} - f, \kernel(\cdot, \vx) \rangle_{\rkhs(\kernel)} } \\
  &\leq \norm[0]{m_{N} - f}_{\rkhs(\kernel)} \sup_{\vx \in \domain} \norm[0]{\kernel(\cdot, \vx)}_{\rkhs(\kernel)} \\
&= \norm[0]{m_{N} - f}_{\rkhs(\kernel)} \sup_{\vx \in \domain} \sqrt{\smash[b]{k(\vx, \vx)}} \\
&= C_k \norm[0]{m_{N} - f}_{\rkhs(\kernel)},
\end{split}
\end{equation*}
where $C_k = \sup_{\vx \in \domain} \sqrt{\smash[b]{k(\vx, \vx)}} < \infty$.
Equation~\eqref{eq:bq-integral-bound-interpolant} and Assumption~\ref{assumption:theory} yield
\begin{align*}
  \abs[0]{ I_\measure(\f) - \mu_{N} } \leq p_\textup{max} \int_{[0,1]^d} \abs[0]{ \f(\vx) - m_{N}(\vx) } \dif \vx &\leq C_k \, p_\textup{max} \norm[0]{f - m_{N}}_{\rkhs(\kernel)} \\
  &\to 0.
\end{align*}
This concludes the proof.
\end{proof}

Proposition~\ref{prop:generic-guarantee-2} is likely an old result. 
We follow the proof of Corollary~2.8 in \citet{Graf2013}.

\begin{proof}[Proof of Proposition~\ref{prop:generic-guarantee-2}]
Let $\vu = (\frac{1}{N}, \ldots, \frac{1}{N}) \in \R^N$. By Proposition~\ref{prop:wce} and~\eqref{eq:wce-explicit},
\begin{equation*}
  \begin{split}
    \Sigma_N = \min_{\vv \in \R^N} e_k(\vv, X)^2 &\leq e_k(\vu, X)^2 \\
    &= \kernel_{\measure\measure} - 2\vu\Trans\vk_{\measure\points} + \vu\Trans \mK_{\points\points} \vu \\
&= \kernel_{\measure\measure} - \frac{2}{N} \sum_{i=1}^N \kernel_\measure(\vx_i) + \frac{1}{N^2} \sum_{i=1}^N \sum_{j=1}^N \kernel(\vx_i, \vx_j).
  \end{split}
\end{equation*}
Since $\vx_i \sim \measure$ are i.i.d., taking expectation and using the definitions of $\kernel_\measure$ and $\kernel_{\measure\measure}$ yields
\begin{equation*}
\begin{split}
  \mathbb{E}[ \Sigma_N ] &\leq \mathbb{E}[\kernel_{\measure\measure}] - \frac{2}{N} \sum_{i=1}^N \mathbb{E}[\kernel_\measure(\vx_i)] + \frac{1}{N^2} \sum_{i=1}^N \sum_{j=1}^N \mathbb{E}[\kernel(\vx_i, \vx_j)] \\
&= \kernel_{\measure\measure} - \frac{2}{N} \sum_{i=1}^N \kernel_{\measure\measure} + \frac{1}{N^2} \sum_{i \neq j} \kernel_{\measure\measure}  + \frac{1}{N^2} \sum_{i=1}^N \int_{[0,1]^d} \kernel(\vx, \vx) \dif \measure(\vx) \\
&= \Big( 1 - 2 + \frac{N^2 - N}{N^2} \Big) \kernel_{\measure\measure} + \frac{1}{N} \int_{[0,1]^d} \kernel(\vx, \vx) \dif \measure(\vx) \\
&= \frac{1}{N} \Big( \int_{[0,1]^d} \kernel(\vx, \vx) \dif \measure(\vx) - \kernel_{\measure\measure} \Big).
\end{split}
\end{equation*}
The claim follows from~\eqref{eq:rkhs-error-bound}.
\end{proof}

\subsection{Proofs for Section~\ref{sec:guarantees-isotropic-matern}} \label{sec:proofs-isotropic-matern}

In this section we use $\lesssim$ to denote an inequality which holds up to a constant coefficient independent of $\f$, $X$, the nugget $\lambda$, the length-scale $\ell$ and the scale $\sigma$.
Let $\alpha > d/2$ and $\nu = \alpha - d/2 > 0$.
Here we assume that $\kernel$ is the isotropic Matérn kernel of order $\nu$ given by
\begin{equation*}
  \begin{split}
    \kernel(\vx, \vx') &= \Phi_\kernel(\norm[0]{\vx - \vx'}) \\
    &= \sigma^2 \frac{2^{1-\nu}}{\Gamma(\nu)} \Big( \frac{\sqrt{2\nu} \norm[0]{\vx - \vx'}}{\ell} \bigg)^\nu K_\nu \Big( \frac{\sqrt{2\nu} \norm[0]{\vx - \vx'}}{\ell} \Big),
  \end{split}
\end{equation*}
where $\Gamma$ denotes the gamma function and $K_\nu$ the modified Bessel function of the second kind.
The length-scale $\ell > 0$ controls how ``wiggly'' the processes are, while the regularity parameter $\nu > 0$ determines the differentiability of a process with Matérn covariance.
Because the Fourier transform of a Matérn kernel is
\begin{equation*}
  \widehat{\Phi}_\kernel(\vxi) = \sigma^2 \frac{2^{\d/2} \Gamma(\nu + d/2)}{\Gamma(\nu)} \Big( \frac{2\nu}{\ell^2} \Big)^\nu \Big( \frac{2\nu}{\ell^2} + \norm[0]{\vxi}^2 \Big)^{-(\nu + d/2)},
\end{equation*}
one may compute from the Fourier definition of the Sobolev norm that the RKHS of a Matérn kernel is \emph{norm-equivalent} to a Sobolev space of order $\nu$, which in this case means that $\rkhs(\kernel) = H^\alpha([0, 1]^d)$ as sets of functions and
\begin{equation} \label{eq:norm-equivalence}
  C_k \norm[0]{g}_{H^\alpha([0,1]^d)} \leq \norm[0]{g}_{\rkhs(\kernel)} \leq C_k' \norm[0]{g}_{H^\alpha([0,1]^d)}.
\end{equation}
for the positive constants
\begin{equation} \label{eq:matern-norm-constants}
  C_k = c \frac{1}{\sigma \ell^{d/2}} \max\Big\{1, \frac{\sqrt{2\nu}}{\ell} \Big\}^{-1}, \quad C_k' = c' \frac{1}{\sigma \ell^{d/2}} \max\Big\{1, \frac{\ell}{\sqrt{2\nu}} \Big\},
\end{equation}
where $c$ and $c'$ are positive constants that depend only on $d$ and $\nu$.
See, for example, Lemma~3.4 and its proof in \citet{Teckentrup2020} for the details of this derivation.
Consequently,
\begin{equation} \label{eq:matern-norm-constants-ratio}
  \hspace{-0.1cm} C_k^{-1} C_k' \lesssim \max\Big\{1, \frac{\ell}{\sqrt{2\nu}} \Big\} \max\Big\{1, \frac{\sqrt{2\nu}}{\ell} \Big\} = \max\Big\{ \frac{\ell}{\sqrt{2\nu}}, \frac{\sqrt{2\nu}}{\ell} \Big\}.
\end{equation}

The following result is Theorem~3.2 in \citet{Arcangeli2012} with $n = d$, $p = q = 2$, $s = \tau$, $r = \kappa$, $\varkappa = 2$ and $\Omega = [0, 1]^d$.
The requirement in \citet{Arcangeli2012} that $h_X$ be sufficiently small can be eliminated by inflating the constant $C$.

\begin{theorem}[Sobolev sampling inequality] \label{thm:sobolev-sampling-inequality}
  Let $\kappa > d/2$ and $\tau \in [0, \kappa]$.
  Then there is a positive constant $C$, which does not depend on $u$, $X$, $\lambda$, $\ell$ or $\sigma$, such that
  \begin{equation*}
    \norm[0]{u}_{H^\tau([0, 1]^d)} \leq C \Big( h_X^{\kappa - \tau} \norm[0]{u}_{H^\kappa([0,1]^d)} + h_X^{d/2 - \tau} \Big[ \sum_{i=1}^N u(\vx_i)^2 \Big]^{1/2} \Big)
  \end{equation*}
  for any $u \in H^\kappa([0,1]^d)$ and any nodes $X = \{\vx_1, \ldots, \vx_N\} \subset [0, 1]^d$.
\end{theorem}

The following result is a version of Theorem~4.2 in \citet{Narcowich2006} that makes explicit the dependency on the norm-equivalence constants.
See also Theorem~16 in \citet{Wynne2021}.
A proof is contained in the proof of Theorem~4.2 in \citet{Karvonen2020b}.

\begin{proposition} \label{prop:sobolev-prop-1}
  Let $\alpha \geq \beta > d/2$ and suppose that $\f \in H^\beta([0, 1]^d)$.  
  Then
  \begin{equation*}
    \norm[0]{\f - m_{\dataset}}_{H^\beta([0,1]^d)} \lesssim C_k^{-1} C_k' \rho_X^{\alpha-\beta} \norm[0]{\f}_{H^\beta([0,1]^d)}
  \end{equation*}
  for any pairwise distinct nodes $X = \{\vx_1, \ldots, \vx_N\} \subset [0, 1]^d$.
\end{proposition}

\begin{proposition} \label{prop:sobolev-prop-2}
  Let $\alpha \geq \beta > d/2$ and suppose that $f \in H^\beta([0,1]^d)$.
  Then
  \begin{equation*}
    \norm[0]{m_\dataset}_{\rkhs(\kernel)} \lesssim C_k' q_X^{-(\alpha-\beta)} \norm[0]{f}_{H^\beta([0,1]^d)}
  \end{equation*}
  and
  \begin{equation*}
    \norm[0]{m_\dataset}_{H^\alpha([0,1]^d)} \lesssim C_k^{-1} C_k' q_X^{-(\alpha-\beta)} \norm[0]{f}_{H^\beta([0,1]^d)}
  \end{equation*}
  for any pairwise distinct nodes $X = \{\vx_1, \ldots, \vx_N\} \subset [0, 1]^d$.
\end{proposition}
\begin{proof}
  The proof is a combination of the a standard bound for band-limited functions and the norm-equivalence~\eqref{eq:norm-equivalence}.
  By Lemma~A.1 in \citet{Karvonen2020b} there is a function $\f_\beta \in H^\alpha([0, 1]^d)$ such that $\f_\beta(\vx_i) = f(\vx_i)$ for every $i = 1, \ldots, N$ and
  \begin{equation} \label{eq:f-beta-norm-bound}
    \norm[0]{\f_\beta}_{H^\alpha([0,1]^d)} \lesssim q_X^{-(\alpha-\beta)} \norm[0]{\f}_{H^\beta([0,1]^d)}.
  \end{equation}
  Since $\f_\beta$ interpolates $\f$ and, by the norm-equivalence, is an element of $\rkhs(\kernel)$, it is one of the functions over which the minimum in~\eqref{eq:minimum-norm-interpolation-proofs} is taken.
  Therefore $\norm[0]{m_\dataset}_{\rkhs(\kernel)} \leq \norm[0]{\f_\beta}_{\rkhs(\kernel)}$.
  Consequently, norm-equivalence and~\eqref{eq:f-beta-norm-bound} give
  \begin{equation*}
    \begin{split}
      \norm[0]{m_\dataset}_{H^\alpha([0,1]^d)} \leq C_k^{-1} \norm[0]{m_\dataset}_{\rkhs(\kernel)} &\leq C_k^{-1} \norm[0]{\f_\beta}_{\rkhs(\kernel)} \\
      &\leq C_k^{-1} C_k' \norm[0]{\f_\beta}_{H^\alpha([0,1]^d)} \\
      &\lesssim C_k^{-1} C_k' q_X^{-(\alpha-\beta)} \norm[0]{\f}_{H^\beta([0,1]^d)},
      \end{split}
  \end{equation*}
  which is the claim.
\end{proof}

\begin{proof}[Proof of \Cref{thm:matern-generic}]
The triangle inequality gives
\begin{equation*}
  \norm[0]{\f - m_{\dataset,\lambda}}_{L^2([0,1]^d)} \leq \norm[0]{\f - m_{\dataset}}_{L^2([0,1]^d)} + \norm[0]{m_{\dataset} - m_{\dataset,\lambda}}_{L^2([0,1]^d)}.
\end{equation*}
Since $\f - m_{\dataset} \in H^\beta([0,1]^d)$ and $m_{\dataset} - m_{\dataset,\lambda} \in H^\alpha([0,1]^d)$ by the norm-equivalance, we may apply \Cref{thm:sobolev-sampling-inequality} with (a) $\tau = 0$, $\kappa = \beta$ and $u = \f - m_{\dataset, \lambda}$ and (b) $\tau = 0$, $\kappa = \alpha$ and $u = m_\dataset - m_{\dataset, \lambda}$.
This yields
\begin{equation*}
  \begin{split}
    \lVert \f - m_{\dataset,\lambda} &\rVert_{L^2([0,1]^d)} \\
    \lesssim{}& h_X^{\beta} \norm[0]{\f - m_{\dataset}}_{H^\beta([0,1]^d)} + h_X^\alpha \norm[0]{m_{\dataset} - m_{\dataset,\lambda}}_{H^\alpha([0,1]^d)} \\
    &+ h_X^{d/2} \Big( \sum_{i=1}^N [ m_\dataset(\vx_i) - m_{\dataset,\lambda}(\vx_i)]^2 \Big)^{1/2}.
    \end{split}
\end{equation*}
From~\eqref{eq:regularised-nodeset-error} we get
\begin{equation*}
  h_X^{d/2} \Big( \sum_{i=1}^N [ m_\dataset(\vx_i) - m_{\dataset,\lambda}(\vx_i)]^2 \Big)^{1/2} \leq h_X^{d/2} \sqrt{\lambda} \norm[0]{m_\dataset}_{\rkhs(\kernel)},
\end{equation*}
while~\eqref{eq:regularisation-norm-bound} gives
\begin{equation*}
  \begin{split}
  \norm[0]{ m_{\dataset} - m_{\dataset,\lambda} }_{H^\alpha([0, 1]^d)} &\leq \norm[0]{ m_{\dataset} }_{H^\alpha([0, 1]^d)} + \norm[0]{ m_{\dataset,\lambda} }_{H^\alpha([0, 1]^d)} \\
  &\leq 2 \norm[0]{ m_{\dataset} }_{H^\alpha([0, 1]^d)}.
  \end{split}
\end{equation*}
Therefore
\begin{equation*}
  \begin{split}
    \norm[0]{\f - m_{\dataset,\lambda}}_{L^2([0,1]^d)} \lesssim{}& h_X^{\beta} \norm[0]{\f - m_{\dataset}}_{H^\beta([0,1]^d)} + h_X^\alpha \norm[0]{ m_{\dataset} }_{H^\alpha([0, 1]^d)} \\
    &+ h_X^{d/2} \sqrt{\lambda} \norm[0]{m_\dataset}_{\rkhs(\kernel)}.
  \end{split}
\end{equation*}
Applying Propositions~\ref{prop:sobolev-prop-1} and~\ref{prop:sobolev-prop-2} and the definition $\rho_X = h_X / q_X$ then yields
\begin{equation*}
  \begin{split}
    \norm[0]{\f - m_{\dataset,\lambda}}_{L^2([0,1]^d)} \lesssim{}& C_k^{-1} C_k' h_X^{\beta} \rho_X^{\alpha-\beta} \norm[0]{\f}_{H^\beta([0,1]^d)} \\
    &+ C_k' h_X^{d/2} q_X^{-(\alpha-\beta)} \sqrt{\lambda} \norm[0]{\f}_{H^\beta([0,1]^d)}.
  \end{split}
\end{equation*}
Using~\eqref{eq:matern-norm-constants} and~\eqref{eq:matern-norm-constants-ratio} yields
\begin{equation*}
  \begin{split}
    \lVert \f - m_{\dataset,\lambda} &\rVert_{L^2([0,1]^d)} \\
    \lesssim{}& \max\Big\{ \frac{\ell}{\sqrt{2\nu}}, \frac{\sqrt{2\nu}}{\ell} \Big\} h_X^{\beta} \rho_X^{\alpha-\beta} \norm[0]{\f}_{H^\beta([0,1]^d)} \\
    &+ \frac{1}{\sigma \ell^{d/2}} \max\Big\{1, \frac{\ell}{\sqrt{2\nu}} \Big\} h_X^{d/2} q_X^{-(\alpha-\beta)} \sqrt{\lambda} \norm[0]{\f}_{H^\beta([0,1]^d)},
    \end{split}
\end{equation*}
which, when inserted in~\eqref{eq:error-L1L2}, concludes the proof.
\end{proof}

\begin{proof}[Proof of Theorem~\ref{thm:matern-random}]
Recall the definitions and results from Section~\ref{sec:integration-in-rkhs}. Because $\rkhs(\kernel)$ is norm-equivalent to the Sobolev space $H^\alpha([0, 1]^d)$, Theorems~1 and~2 and Corollary~2 in \citet{KriegSonnleitner2022} imply that
\begin{equation*}
  \begin{split}
    \mathbb{E}[ \Sigma_N^{1/2} ] &= \mathbb{E}\Big[ \min_{\vv \in \R^N} \sup_{ \norm[0]{g}_{\rkhs(\kernel)} \leq 1} \, \abs[2]{ \int_{[0,1]^d} g(\vx) \dif \vx - \sum_{i=1}^N v_i g(\vx_i) } \Big] \\
    &\leq C_k^{-1} \mathbb{E}\Big[ \min_{\vv \in \R^N} \, \sup_{ \norm[0]{g}_{H^\alpha([0,1]^d)} \leq 1} \, \abs[2]{ \int_{[0,1]^d} g(\vx) \dif \vx - \sum_{i=1}^N v_i g(\vx_i) } \Big]
    \end{split}
\end{equation*}
decays as $N^{-\alpha / d}$, where the second inequality uses~\eqref{eq:norm-equivalence}.
From~\eqref{eq:rkhs-error-bound} we then obtain
\begin{equation*}
  \mathbb{E}[ \abs[0]{ I_\measure(\f) - \mu_N } ] \leq \norm[0]{\f}_{\rkhs(k)} \mathbb{E}[ \Sigma_N^{1/2} ],
\end{equation*}
so that the claim follows from~\eqref{eq:norm-equivalence} and~\eqref{eq:matern-norm-constants-ratio} and the assumption on the boundedness of the length-scale.
\end{proof}

\subsection{Proofs for Section~\ref{sec:guarantees-se-kernel}} \label{sec:proofs-guarantees-se-kernel}

\begin{proof}[Proof of Theorem~\ref{thm:se-generic}]
  The theorem is an immediate consequence of~\eqref{eq:error-L1L2} and Theorem~6.1 in \citet{RiegerZwicknagl2010} with $q=2$.
\end{proof}

Theorems~\ref{thm:se-generic} and~\ref{thm:se-quasi-uniform} have previously appeared as Theorem~2.20 in \citet{Karvonen2019a}.
Theoretical guarantees for explicitly active Bayesian quadrature based on the square exponential kernel may be found in \citet[Sec.\@~4.1]{Kanagawa2019}.
Further guarantees, including on the unbounded domain $\domain = \R^d$, could be derived from the results in~\citet{KuoWozniakowski2012, KuoSloanWozniakowski2017, Karvonen2021a, Karvonen2022-power-series}.

\section{Proofs for Section~\ref{sec:active-bq-guarantees}} \label{sec:proofs-active-bq-guarantees}

\begin{proof}[Proof of Theorem~\ref{thm:active-bq-guarantee}]
Because the transformation $\psi$ is strictly increasing, maximization of~\eqref{eq:sequential-selection-for-guarantees} is equivalent to maximization of 
\begin{equation*}
  \vk_{\measure\points(\tilde{\vx})}\Trans \mK_{\points(\tilde{\vx})\points(\tilde{\vx})}^{-1}\vk_{\measure\points(\tilde{\vx})} ,
\end{equation*}
which in turn is equivalent to minimization of the integral variance
\begin{equation*}
  \Sigma_{\dataset_n(\tilde{\vx})} = \kernel_{\measure\measure} - \vk_{\measure\points(\tilde{\vx})}\Trans \mK_{\points(\tilde{\vx})\points(\tilde{\vx})}^{-1}\vk_{\measure\points(\tilde{\vx})}
\end{equation*}
given the augmented data $\dataset_n(\tilde{\vx}) = \dataset_n \cup \{(\tilde{\vx}, \f(\tilde{\vx}))\} = \{(\vx_i, \f(\vx_i))\}_{i=1}^n \cup \{(\tilde{\vx}, \f(\tilde{\vx}))\}$.
From the interpretation (see \Cref{sec:integration-in-rkhs,sec:sampling-active}) of the integral variance as a worst-case error in the RKHS $\rkhs(\kernel)$, we see that the sequential nodes that we consider are equal to the nodes selected by the \emph{greedy algorithm} in Section~4 of \citet{Santin2021}.
We may therefore invoke Theorem~5.1 in \citet{Santin2021} and Proposition~\ref{prop:wce}, which together give
\begin{equation*}
  \Sigma_{\dataset_n}^{1/2} = \min_{\vv \in \R^N} e_k(\vv, \points_N) \leq C N^{-1/2}
\end{equation*}
for a positive $C$ that depends only on $\domain$, $\measure$ and $\kernel$.
The claim then follows from~\eqref{eq:rkhs-error-bound}.
\end{proof}

\FloatBarrier

\section*{Acknowledgements}

MM's work was mainly performed while at the University of Tübingen. MM gratefully acknowledges financial support by the European Research Council through ERC StG Action 757275 / PANAMA; the DFG Cluster of Excellence “Machine Learning - New Perspectives for Science”, EXC 2064/1, project number 390727645; the German Federal Ministry of Education and Research (BMBF) through the T\"{u}bingen AI Center (FKZ: 01IS18039A); and funds from the Ministry of Science, Research and Arts of the State of Baden-W\"{u}rttemberg.
TK was generously supported by the Research Council of Finland projects 338567 (``Scalable, adaptive and reliable probabilistic integration''), 359183 (``Flagship of Advanced Mathematics for Sensing, Imaging and Modelling''), and 368086 (``Inference and approximation under misspecification'').
TK acknowledges the research environment provided by ELLIS Institute Finland.
We thank Vesa Kaarnioja for correcting some of our misapprehensions about Bayesian quadrature with periodic covariance and two anonymous reviewers for comments and suggestions that helped to improve the survey.

\backmatter  

\printbibliography

\end{document}